\documentclass[lettersize,journal]{IEEEtran}
\ifCLASSINFOpdf
\else
\fi

\usepackage[caption=false,font=footnotesize]{subfig}
\usepackage{amsthm}
\usepackage{amsmath}
\usepackage{amssymb}
\usepackage{empheq}
\usepackage{booktabs}
\usepackage{multirow}
\usepackage{makecell}
\usepackage{cite}
\usepackage{color}
\usepackage{algorithm}
\usepackage{algorithmic}
\usepackage{graphicx}
\usepackage{threeparttable}
\usepackage{pifont}

\newcommand{\col}{\operatorname{col}}

\newtheorem{theorem}{\indent Theorem}
\newtheorem{lemma}{\indent Lemma}
\newtheorem{corollary}{\indent Corollary}
\newtheorem{definition}{Definition}
\newtheorem{assumption}{\indent Assumption}
\newtheorem{remark}{\indent Remark}
\theoremstyle{definition}

\newfloat{mechanism}{tbp}{lom}
\floatname{mechanism}{Mechanism}

\begin{document}
\title{Gradient Manipulation in Distributed Stochastic Gradient Descent with Strategic Agents: Truthful Incentives with Convergence Guarantees}

\author{Ziqin Chen and Yongqiang Wang, \textit{Senior Member, IEEE}
\thanks{This work was supported by the National Science Foundation under Grant CCF-2106293, Grant CCF-2215088, Grant CNS-2219487, Grant CCF-2334449, and Grant CNS-2422312. (Corresponding author: Yongqiang Wang, email:yongqiw@clemson.edu).}
\thanks{Ziqin Chen and Yongqiang Wang are with the Department of Electrical and Computer Engineering, Clemson University, Clemson, SC 29634 USA.}}

\maketitle
\begin{abstract}
Distributed learning has gained significant attention due to its advantages in scalability, privacy, and fault tolerance. In this paradigm, multiple agents collaboratively train a global model by exchanging parameters only with their neighbors. However, a key vulnerability of existing distributed learning approaches is their implicit assumption that all agents behave honestly during gradient updates. In real-world scenarios, this assumption often breaks down, as selfish or strategic agents may be incentivized to manipulate gradients for personal gain, ultimately compromising the final learning outcome. In this work, we propose a fully distributed payment mechanism that, for the first time, guarantees both truthful behaviors and accurate convergence in distributed stochastic gradient descent. This represents a significant advancement, as it overcomes two major limitations of existing truthfulness mechanisms for collaborative learning: (1) reliance on a centralized server for payment collection, and (2) sacrificing convergence accuracy to guarantee truthfulness. In addition to characterizing the convergence
rate under general convex and strongly convex
conditions, we also prove that our approach guarantees the cumulative gain that an agent can obtain through strategic behavior remains finite, even as the number of iterations approaches infinity---a property unattainable by most existing truthfulness mechanisms. Our experimental results on standard machine learning tasks, evaluated on benchmark datasets, confirm the effectiveness of the proposed approach.
\end{abstract}

\IEEEpeerreviewmaketitle
\section{Introduction}\label{SectionI}
Recent years have witnessed significant advances in distributed methods for collaborative optimization and learning~\cite{DSGD,machinelearning,survey,control}. By distributing both data and computational resources across multiple agents, distributed methods leverage the combined computing power of multiple devices to collaboratively train a global model without the need for a centralized server. Compared with server-assisted collaborative learning\footnote{We use ``server-assisted collaborative learning" to refer to collaborative learning involving a centralized server or aggregator, with federated learning as a representative example.}, distributed learning avoids monopolistic control and single points of failure~\cite{DSGDlack}, and hence, is widely applied in areas such as distributed machine learning~\cite{distributedmachine}, multi-robot coordination~\cite{robot}, and wireless networks~\cite{wireless}.

However, almost all existing distributed learning approaches implicitly assume that all participating agents act truthfully (see, e.g.,~\cite{Ranxin1GT1,Pushi2GT2,pushi1DSGD4,Leijinlong1GT3,yuankun1GT4,yuandeming1zero2,Lukaihong1DSGD5,wenguanghui1ADMM,zhangsheng2025,sunchao1DSGD6,Liushuai1Zero3}), which is essential for their successful execution. This premise becomes untenable in practical scenarios where participating agents are strategic and self-interested. In such cases, participating agents may manipulate gradient updates to maximize their own utilities, ultimately undermining the performance of collaborative learning. For example, in distributed learning with heterogeneous data distributions, an agent may inflate its local gradient updates to skew the final model in favor of its own data distribution~\cite{truthfulmain,truthfulmain2}. Similarly, in a shared market, a firm may inject noise into its shared data  to degrade the quality of other firms' predictive model training, thereby maintaining competitive advantage~\cite{honestmain}. More motivating examples are provided in Section~\ref{SectionIIB}. Such untruthful behaviors pose a significant threat to the performance of existing distributed learning and optimization algorithms (as demonstrated by our experimental results in Fig.~\ref{Fig2}).

To mitigate strategic behaviors of participating agents in collaborative learning, several approaches have been proposed, which can be broadly categorized into incentive-based approaches~\cite{oneshot2,oneshot3,oneshot4,oneshot5,oneshot6,oneshot7,oneshot8,iterative1,IcentiveFL0,IcentiveFL1,IcentiveFL2,VCGauto,iterative2,iterative3,iterative4,VCGdefinition} and joint-differential-privacy (JDP)-based approaches~\cite{JDP1,JDP2,JDP3,JDP4}. However, all existing approaches rely on a centralized server to collect information from all agents and then execute a truthfulness mechanism~\cite{VCGauto}. For example, the Vickrey–Clarke–Groves (VCG) mechanism~\cite{VCG1,VCG2,VCG3}, a well-known incentive-based approach, requires a centralized server to aggregate true gradients/functions from all agents to compute the corresponding monetary transfers. Similarly, JDP-based approaches require a centralized server to collect iteration variables from all agents in order to compute the necessary noise~\cite{JDP1,JDP2,JDP3,JDP4}\footnote{Although the recent work by~\cite{JDPziji} proposes a JDP-based truthfulness approach for distributed aggregative optimization, it is limited to scenarios where each agent's objective function depends on an aggregative term of others' optimization variables. Furthermore, due to differential-privacy noise injection, the approach achieves truthfulness at the cost of compromised optimization accuracy (see Theorem 3 in~\cite{JDPziji} for details).}. To the best of our knowledge, no existing approaches can effectively incentivize truthful behaviors in fully distributed optimization and learning.

\subsection{Related literature} \noindent\textbf{Incentive-based truthfulness approaches.} Truthfulness in statistical (mean) estimation has been addressed using incentive/payment mechanisms~\cite{oneshot3,oneshot6,oneshot8}. These mechanisms are typically one-shot, where agents choose their strategies once and a centralized server broadcasts payments accordingly, which renders them inapplicable to multi-round, gradient-based distributed learning algorithms. Based on the well-known VCG mechanism~\cite{VCG1,VCG2,VCG3}, truthfulness results have also been reported for federated learning~(see, e.g., \cite{truthfulmain,honestmain, truthfulmain2,IcentiveFL1,IcentiveFL2,VCGauto,iterative2}). However, the VCG mechanism relies on a server to calculate and collect monetary payments, which makes it inapplicable in a fully distributed setting. {\color{blue}Although some recent works, e.g.,~\cite{newlyadd1,newlyadd2,newlyadd3}, have studied distributed implementations of VCG mechanisms or dynamic mechanism design, they still require a centralized coordinator or trusted third party to assist with mechanism operations, such as outcome selection in~\cite{newlyadd1}, allocation determination in~\cite{newlyadd2}, and tax computation in~\cite{newlyadd3}.} Moreover, VCG-based approaches are not budget-balanced and often involve surplus payments~\cite{iterative1,IcentiveFL0,IcentiveFL1,IcentiveFL2,VCGauto}, which further limits their practicality. It is worth noting that many results have discussed incentive mechanisms for encouraging agents' contributions of data/resources in collaborative learning~\cite{oneshot3,IcentiveFL1,IcentiveFL2}. However, those results do not consider agents' strategic manipulation on iterative updates for personal gains.

\noindent\textbf{JDP-based truthfulness approaches.} JDP-based approaches incentivize truthful behavior by injecting noise into algorithmic outputs, thereby masking the impact of any single agent's misreporting on the final model and promoting truthfulness~\cite{JDP1,JDP2,JDP3,JDP4}. However, these approaches require a centralized server to collect local optimization variables from all agents to determine the needed noise amplitude, which makes it infeasible in a fully distributed setting. Moreover, JDP-based approaches have to compromise convergence accuracy to ensure truthfulness~\cite{JDPziji}, which is undesirable in accuracy-sensitive applications.

\begin{table}[H]
\caption{Comparison of our approach with existing truthfulness results for distributed optimization/learning.}
\label{table1}
\footnotesize
\setlength{\tabcolsep}{4pt}
\begin{threeparttable}
\begin{tabular}{lcccc}
\hline
\multirow{2}{*}{Approach} &
\multicolumn{1}{c}{Fully} &
\multicolumn{1}{c}{$\epsilon$-Incentive} &
\multicolumn{1}{c}{Budget} &
\multicolumn{1}{c}{Accurate} \\
& \multicolumn{1}{c}{distributed?} & \multicolumn{1}{c}{{compatible?}\tnote{a}}& \multicolumn{1}{c}{{ balanced?}\tnote{b}} &
\multicolumn{1}{c}{convergence?} \\
\hline
\cite{VCGauto} & \ding{55} & \ding{51} & \ding{55} & \ding{51}\\
\cite{honestmain} & \ding{55} & {\ding{55}} & {\ding{51}} & \ding{55} \\
\cite{truthfulmain,truthfulmain2}{\color{blue}} & \ding{55}  & \ding{55} & \ding{51} & \ding{55} \\
\cite{JDPziji} & \ding{51} & \ding{51} & \ding{51} & \ding{55}  \\
Our approach & \ding{51}  & \ding{51}  & \ding{51}  & \ding{51}   \\
\hline
\end{tabular}
\begin{tablenotes}
\item[a] We use ``$\epsilon$-Incentive compatible" to describe whether an approach can guarantee that the cumulative gain that an agent obtains from persistent strategic behaviors remains finite (bounded by some finite value $\epsilon$), even as the number of iterations approaches infinity.
\item[b] We use ``Budget balanced" to mean total payments collected equal total payments made, requiring no external subsidies or surplus. This ensures that the mechanism is financially sustainable and scalable in practice. {\color{blue}Both the sender and the recipient of each payment are agents in the network.}
\end{tablenotes}
\end{threeparttable}
\end{table}

\subsection{Contributions}
{In this article, we propose a distributed payment mechanism that 
guarantees both truthful behavior of agents and accurate convergence in distributed stochastic gradient methods. The main contributions are summarized as follows (Table~\ref{table1} highlights the contributions and their comparison with existing truthfulness results):}

\begin{itemize}
\item We propose a fully distributed payment mechanism that incentivizes truthful behaviors among interacting strategic agents in distributed stochastic gradient methods. This represents a substantial breakthrough, as existing truthfulness mechanisms (in, e.g.,~\cite{truthfulmain,honestmain,truthfulmain2,oneshot2,oneshot3,oneshot4,oneshot5,oneshot6,oneshot7,oneshot8,iterative1,IcentiveFL0,IcentiveFL1,IcentiveFL2,VCGauto,iterative2,iterative3,iterative4,VCGdefinition,JDP1,JDP2,JDP3,JDP4}) all rely on a centralized server to aggregate local information from agents. To the best of our knowledge, this is the {\bf first} payment mechanism for {\bf distributed gradient descent}  without the assistance of any centralized server.
\item Our payment mechanism guarantees that the incentive for a strategic agent to deviate from truthful behaviors diminishes to zero over time (see Lemma~\ref{ML1}). Building on this, we further prove that the cumulative gain that an agent can obtain from its strategic behaviors remains finite, even when the number of iterations tends to infinity (see Theorem~\ref{MT2}). This stands in sharp contrast to existing incentive-based approaches for federated learning in~\cite{truthfulmain,honestmain,truthfulmain2}, which cannot eliminate agents' incentives to behave untruthfully--resulting in a cumulative gain that grows unbounded when the number of iterations tends to infinity.
\item In addition to ensuring diminishing incentives for untruthful behavior in distributed gradient descent, our payment mechanism also guarantees accurate convergence, even in the presence of persistent gradient manipulation by agents (see Theorem~\ref{MT1}). This is in stark contrast to existing JDP-based truthfulness results in~\cite{JDP1,JDP2,JDP3,JDP4,JDPziji} and incentive-based truthfulness results in~\cite{honestmain,truthfulmain,truthfulmain2}, all of which are subject to an optimization error. We analyze the convergence rates of distributed gradient descent under our payment mechanism for general convex and strongly convex objective functions. This is more comprehensive than existing truthfulness results in~\cite{honestmain,truthfulmain,truthfulmain2,VCGauto} that focus solely on the strongly convex case. 
\item Different from most existing VCG-based approaches (in, e.g., \cite{iterative1,IcentiveFL0,IcentiveFL1,IcentiveFL2,VCGauto}), which cannot ensure budget-balance (the total payments from all agents sum to zero, a property essential for the financial sustainability and practical scalability of the mechanism), our payment mechanism is budget-balanced. This is significant in a fully distributed setting since no centralized server is used to manage subsidies or surplus.
\item We evaluate the performance of our truthful mechanism using representative distributed   learning tasks, including image classification on the FeMNIST dataset and next-character prediction on the Shakespeare dataset. The experimental results confirm the effectiveness of our approach.
\end{itemize}

{The organization of the paper is as follows. Sec.~\ref{SectionII} introduces the problem formulation, presents motivating examples, and formalizes the used game-theoretic framework. Sec.~\ref{SectionIII} proposes the distributed payment mechanism. Sec. \ref{SectionIV} analyzes convergence rate and establishes incentive compatibility guarantees. Sec.~\ref{SectionV}
presents experimental results. Finally, Sec.~\ref{SectionVI} concludes the paper.}

\textit{Notations:} We use $\mathbb{R}^{n}$ to denote the set of $n$-dimensional real Euclidean space. We denote $\nabla F(\theta)$ as the gradient of $F(\theta)$ and $\mathbb{E}[\theta]$ as the expected value of a random variable $\theta$. We denote the set of $N$ agents by $[N]$ and the neighboring set of agent $i$ as $\mathcal{N}_{i}$. The cardinality of $\mathcal{N}_{i}$ is denoted as $\deg(i)$. We denote the coupling matrix by $W\!=\!\{w_{ij}\}\!\in\!{\mathbb{R}^{N\times N}}$, where $w_{ij}\!>\!0$ if agent $j$ interacts with agent $i$, and $w_{ij}=0$ otherwise. We define $w_{ii}\!=\!1-\sum_{j\in\mathcal{N}_{i}}w_{ij}$. We abbreviate ``with respect to" as \textit{w.r.t.} Furthermore, we use an overbar to denote the average of all agents, e.g., $\bar{\theta}_{t}=\frac{1}{N}\sum_{i=1}^{N}\theta_{i,t}$, and use bold font with iteration subscripts to denote the stacked
vector of all $N$ agents, e.g., $\boldsymbol{\theta}_{t}=\col(\theta_{1,t},\cdots,\theta_{N,t})$.

\section{Problem Formulation and Preliminaries}\label{SectionII}
\subsection{Distributed optimization and learning}\label{SectionIIA}
We consider $N\geq2$ agents participating in distributed optimization and learning, each possessing a private dataset whose distribution can be heterogeneous across the agents. The goal is for all agents to cooperatively find a solution $\theta^{*}$ to the following stochastic optimization problem:
\begin{equation}
\min_{\theta\in\mathbb{R}^{n}}~F(\theta)=\frac{1}{N}\sum_{i=1}^{N}f_{i}(\theta),\quad f_{i}(\theta)=\mathbb{E}_{\zeta_{i}\sim \mathcal{P}_{i}}[l(\theta;\zeta_{i})],\label{primal}
\end{equation}
where $\theta\in\mathbb{R}^{n}$ denotes a global model parameter and $\zeta_{i}$ is a random data sample of agent $i$ drawn from its local data distribution $\mathcal{P}_{i}$. The loss function $l(\theta;\zeta_{i}):\mathbb{R}^{n}\!\times\!\mathbb{R}^{n}\!\mapsto\!\mathbb{R}$ is assumed to be differentiable in $\theta$ for every $\zeta_{i}$ and the local objective function $f_{i}(\theta)$ of agent $i$ can be nonconvex.

In real-world applications, the data distribution $\mathcal{P}_{i}$ is typically unknown. Hence, each agent $i$ can only access a noisy estimate of the gradient $\nabla f_{i}(\theta_{i,t})$, computed at its current local model parameter $\theta_{i,t}$ using the available local data. For example, 
at each iteration $t$, agent $i$ samples a batch of $B\geq 1$ data points and computes a gradient estimate as $g_{i}(\theta_{i,t})=\frac{1}{B}\sum_{j=1}^{B}\nabla l(\theta_{i,t};\zeta_{ij})$. Using this gradient estimate $g_i(\theta_{i,t})$, along with the model parameters $\{\theta_{j,t}\}_{j\in\mathcal{N}_{i}}$ received from its neighbors, agent $i$ updates its local parameter according to a distributed optimization/learning algorithm.

{\color{blue}We denote the coupling matrix by $W=[w_{ij}]\in\mathbb{R}^{N\times N}$, where $w_{ij}>0$ if agent~$j$ shares information with agent $i$, and $w_{ij}=0$ otherwise.} We denote the set of $N$ agents by $[N]$ and the neighboring set of agent $i$ as $\mathcal{N}_{i}$. The cardinality of $\mathcal{N}_{i}$ is denoted as $\deg(i)$. We define $w_{ii}=1-\sum_{j\in\mathcal{N}_{i}}w_{ij}$ and give the following assumption to ensure that the communication graph is connected~\cite{Pushi2GT2}.

\begin{assumption}\label{A2}
We assume that the matrix $W$ is symmetric and $\rho=\max\{|\pi_{2}|,|\pi_{N}|\}<1$, where $\pi_{N}\leq\cdots\leq \pi_{2}<\pi_{1}=1$ denote 
the eigenvalues of $W$.
\end{assumption}

Existing distributed optimization/learning  algorithms (see, e.g.,~\cite{Ranxin1GT1,Pushi2GT2,pushi1DSGD4,Leijinlong1GT3,yuankun1GT4,yuandeming1zero2,Lukaihong1DSGD5,wenguanghui1ADMM,zhangsheng2025,sunchao1DSGD6,Liushuai1Zero3}) universally assume that participating agents are honest and behave truthfully. However, this assumption may be unrealistic in real-world scenarios, where agents can behave selfishly or strategically. For example, a strategic agent may amplify its gradient estimates to bias the final model parameter in favor of its own data distribution, or inject noise into its gradient information to degrade the performance of other agents' models for competitive advantage. (We provide additional motivating examples to illustrate how agents can benefit from gradient manipulation in distributed least squares and distributed mean estimation in Section~\ref{SectionIIB}.) Such strategic gradient manipulation can significantly degrade the learning performance of existing distributed learning algorithms (as evidenced by our experimental results in Fig.~\ref{Fig2}).

Next, we discuss the classical Distributed Stochastic Gradient Descent (Distributed SGD) in the presence of gradient manipulation by a strategic agent $i\in[N]$.

\noindent\textbf{Distributed SGD in the presence of strategic behavior.} 
At each iteration $t$, each agent $i$ strategically chooses a manipulated gradient $m_{i,t}$, which, in general, is a function of agent $i$'s true gradient $g_{i}(\theta_{i,t})$. Using the manipulated gradient $m_{i,t}$ and the model parameters $\{\theta_{j,t}\}_{j\in\mathcal{N}_i}$ received from its neighbors, each agent $i$ updates its local model parameter according to Algorithm~\ref{algorithm}.

\begin{algorithm}[H]
\caption{Distributed SGD in the presence of strategic behavior (from agent $i$'s perspective)}
\label{algorithm} 
\begin{algorithmic}[1]
\STATE {\bfseries Initialization:} $\theta_{i,0}\in \mathbb{R}^{n}$;
stepsize $\lambda_{t}>0$.
\STATE Send $\theta_{i,0}$ to neighbors $j\in\mathcal{N}_{i}$ and receive $\theta_{j,0}$ from neighbors $j\in\mathcal{N}_{i}$.
\FOR{$t=0,\ldots,T$}
\STATE $\theta_{i,t+1}=\sum_{j\in\mathcal{N}_{i}\cup\{i\}}w_{ij}\theta_{j,t}-\lambda_{t}m_{i,t}$;
\STATE Send $\theta_{i,t+1}$ to neighbors $j\in\mathcal{N}_{i}$ and receive $\theta_{j,t+1}$ from neighbors $j\in\mathcal{N}_{i}$.
\ENDFOR
\end{algorithmic}
\end{algorithm}

It is worth noting that we focus on gradient manipulation rather than model-parameter manipulation for two reasons. First, any manipulation of model parameters shared among agents effectively corresponds to some form of alteration in the gradient estimates, as proven in Corollary~\ref{C1} in Appendix G. Second, gradient manipulation is 
the most direct and practically effective strategy for a strategic agent to increase its personal gain. Specifically, by upscaling its own gradient estimates, an agent can increase the influence of its local data on the cooperative learning process, thereby
pulling the final model parameter closer to the minimizer of its local objective function and reducing its own cost. On the other hand, by injecting noise into its gradient estimates, an agent can reduce the usefulness of its data to its neighbors, degrading their model performance and gaining competitive advantage. In comparison, manipulating model parameters does not provide a clear strategic benefit. For these reasons, we focus on gradient manipulation as the primary form of untruthful or strategic behavior of participating agents.

\subsection{Motivating examples}\label{SectionIIB}
In this subsection, we consider a distributed mean estimation problem with stochastic gradients to show that agents can improve their payoffs by manipulating gradient updates in Algorithm~\ref{algorithm}, at the expense of network performance.

Specifically, we consider a problem where $N$ agents cooperatively estimate a global mean $\mu= \frac{1}{N}\sum_{i=1}^{N} \mu_{i}$, where $\mu_{i}\in\mathbb{R}^{n}$ denotes the mean of agent $i$'s data distribution. Formally, the distributed mean estimation problem can be formulated as the following stochastic optimization problem:
\begin{equation}
\!\!\!\textstyle\min_{\theta\in\mathbb{R}^{n}}~F(\theta)=\frac{1}{N}\sum_{i=1}^{N}f_{i}(\theta),~ f_{i}(\theta)=\mathbb{E}[\|\theta-\mu_{i}\|^2].\label{mean}
\end{equation}

We assume $\mu\neq \mu_{i}$ and that the agents cooperatively solve problem~\eqref{mean} using Algorithm~\ref{mechanism} with a diminishing stepsize  $\lambda_{t}=\frac{\lambda_{0}}{(t+1)^{v}}$, where $\lambda_{0}<\frac{1}{2}$  and $0<v<1$. Since the mean $\mu_i$ is typically unknown to agent $i$, the agent only has access to a noisy gradient estimate $g_i(\theta_{i,t}) = 2(\theta_{i,t} - \zeta_{i,t})$ at iteration $t$, based on a sample $\zeta_{i,t}$ from its local distribution. Specifically, each coordinate $\zeta_{i,t}^{p},~p\in[n]$ of $\zeta_{i,t}$ is independently sampled from $\mathcal{N}(\mu_{i}^{p},\sigma^2/n)$. We denote $\mu_{i}=\col(\mu_{i}^{1},\cdots,\mu_{i}^{n})$. {\color{blue}We also denote the mean squared error of agent $i$ and the global mean squared error after $T$ iterations by $f_{i}(\bar{\theta}_{T})=\mathbb{E}[\|\bar{\theta}_{T}-\mu_{i}\|^2]$ and $F(\bar{\theta}_{T})=\frac{1}{N}\sum_{j=1}^{N}\mathbb{E}[\|\bar{\theta}_{T}-\mu_{j}\|^2]$, respectively.

We present the following lemma, showing that a strategic agent can reduce its mean squared error while simultaneously increasing the global mean squared error by amplifying its gradient estimates at every iteration.

\begin{lemma}\label{examplelemma}
For the distributed mean estimation problem defined in~\eqref{mean}, under Assumption~\ref{A2}, whenever $T+1>(\frac{\lambda_{0}(a_{i}+N-1)^2(\sqrt{c_{1}}+\sqrt{c_{i,2}})^2}{(a_{i}-1)^2\|\mu-\mu_{i}\|^2})^{1/v}$, agent $i$ can reduce its mean squared error while increasing the global mean squared error by amplifying its gradient estimates by a factor of $a_{i}>1$ at every iteration. Here, $c_{1}$ and $c_{i,2}$ are given in~\eqref{addeq1} and~\eqref{l218}, respectively. Specifically, the following results hold:
\begin{subequations}\label{eq:3}
\begin{align}
f_{i}(\bar{\theta}_{T}^{\prime})
&\textstyle\leq \left(\frac{N}{a_{i}+N-1}\|\mu-\mu_{i}\|+\sqrt{c_{i,2}\lambda_{T}}\right)^2 \notag\\
&\textstyle<\left(\|\mu-\mu_{i}\|-\sqrt{c_{1}\lambda_{T}}\right)^2
\leq f_{i}(\bar{\theta}_{T}); \label{eq:3a}\\
0
&\textstyle< \left(\frac{a_{i}-1}{a_{i}+N-1}\|\mu-\mu_{i}\|-\sqrt{c_{i,2}\lambda_{T}}\right)^2-c_{1}\lambda_{T} \notag\\
&\textstyle\leq F(\bar{\theta}_{T}^{\prime})-F(\bar{\theta}_{T}) \notag\\
&\textstyle\leq \left(\frac{a_{i}-1}{a_{i}+N-1}\|\mu-\mu_{i}\|+\sqrt{c_{i,2}\lambda_{T}}\right)^2,
\label{eq:3b}
\end{align}
\end{subequations}
where $\bar{\theta}_{T}^{\prime}$ and $\bar{\theta}_{T}$ denote the average models with and without gradient manipulation, respectively. Furthermore, when $T\rightarrow \infty$, strategic agent $i$ can reduce its mean squared error by a factor of $(\frac{N}{a_{i}+N-1})^2$, while increasing the global mean squared error by $\left(\frac{a_{i}-1}{a_{i}+N-1}\right)^2\|\mu-\mu_{i}\|^2$.
\end{lemma} 
\begin{proof}
See Appendix A.
\end{proof}}

Lemma~\ref{examplelemma} demonstrates that a self-interested agent can benefit from its strategic gradient manipulation while degrading the network-level performance.

{\color{blue}\begin{remark}
Lemma~\ref{examplelemma} analyzes the convergence of Algorithm~\ref{algorithm} for distributed mean estimation when agent $i$ manipulates its gradients at every iteration. Algorithm~\ref{algorithm} is fully distributed, with each agent computing its updates locally and communicating only with its neighbors over a network (as evidenced by our consensus analysis in~\eqref{l213}-\eqref{l218}). We clarify that the average model $\bar{\theta}_{t}$ and $\bar{\theta}_{t}^{\prime}$ in Lemma~\ref{examplelemma} are used solely for convergence analysis and are not required for the distributed implementation of Algorithm~\ref{algorithm}; no agent needs to compute or know $\bar{\theta}_{t}$ during algorithm execution.
\end{remark}}

In addition to distributed mean estimation problems, similar truthfulness issues also emerge in collaborative machine learning~\cite{honestmain}, distributed electric-vehicle charging~\cite{JDPziji}, and many
other collaborative optimization applications~\cite{VCGdefinition}.

To quantitatively analyze strategic interactions among agents, we adopt a game-theoretic framework that explicitly defines each agent's strategic behavior, reward, payment, and resulting net utility.

\subsection{Game-theoretic framework}\label{SectionIIC}
\noindent\textbf{Strategic behaviors and action space.} A self-interested agent can enhance its individual outcome through two strategic behaviors: amplifying its gradient estimates to bias the final model parameter in favor of its own data distribution, and injecting noise to degrade the performance of other agents' model parameters for competitive advantage. Both strategic behaviors can be modeled as an agent's action at each iteration. Formally, we  
define the action space for each strategic agent $i\in[N]$ in distributed optimization/learning as follows:
\begin{equation}
\!\mathcal{A}_{i}\!=\!\{\alpha_{i}|\alpha_{i}(g_{i}(\theta))\!=\!a_{i}g_{i}(\theta)\!+\!b_{i}\xi_{i},~a_{i}\!\geq 1,~ b_{i}\!\in\mathbb{R}\}, \label{strategy}
\end{equation}
where $g_{i}(\theta)$ represents agent $i$'s true gradient estimate and $\xi_{i}$ is a zero-mean noise vector with bounded variance, {\color{blue}$\alpha_{i}$ is an operator that maps agent $i$'s gradient estimate $g_{i}(\theta)$ to a manipulated gradient estimate through noise injection and amplification, and $\mathcal{A}_{i}$ denotes the set of admissible gradient-manipulation operators.} The scaling factor $a_{i}$ quantifies the degree of gradient amplification, while the noise factor $b_{i}$ specifies the magnitude of noise injection, both of which can be strategically chosen by agent $i$ at each iteration. For any action space $\mathcal{A}_{i}$, we assume that it includes the identity mapping, which maps $g_{i}$ to itself with probability one. Hence, truthfulness is always a feasible action.

The action space defined in~\eqref{strategy} is designed to capture the strategic behavior of agents in distributed learning rather than arbitrary malicious attacks. To this end, we consider $a_{i}\geq 1$ and ignore $a_{i}<1$, because the latter typically reduces the influence of agent $i$'s local data on cooperative learning, making it a non-utility-improving choice for a rational agent. This focus on strategic, utility-driven behavior excludes general malicious attacks, which are typically disruptive and do not align with an agent’s goal of improving its outcome within the learning framework. In addition, compared with the action spaces defined in existing results on federated learning, such as~\cite{honestmain} that only considers noise injection (i.e., fixing $a_{i}=1$) and~\cite{truthfulmain} that only considers gradient amplification (i.e., fixing $b_{i}=0$), our formulation accounts for a broader range of strategic manipulations.

{\color{blue}Our proposed action space is well motivated for two reasons. First, it encompasses all profitable strategic manipulations studied in the existing distributed learning literature and, to the best of our knowledge, is the broadest action space that can potentially improve an agent's utility under general objective functions. Second, 
it is grounded in physical intuition and prior literature, capturing strategic behaviors that are known or naturally expected to benefit self-interested agents in distributed learning. Although broader action spaces than the one defined in~\eqref{strategy} may exist, many actions outside the action space in~\eqref{strategy} do not provide a clear strategic advantage to a rational agent. For example, an agent may manipulate its received model parameters or modify the consensus weights (e.g., by increasing $w_{ii}$). However, these actions essentially reduce the impact of information received from neighbors on the agent's own model-parameter updates. This reduction weakens the benefits that collaborative learning can provide to the agent (see, e.g.,~\cite{benefit}, which discusses how collaborative learning increases the effective data-sample size and reduces gradient-noise variance).
\begin{remark}
Each agent $i$ can choose both a noise vector $\xi_{i}$ and noise variance $b_{i}$. However, within our framework, where $\xi_i$ is unbiased, it is the variance of $\xi_{i}$, rather than the specific realization of the noise vector, that affects agents' net utilities. In fact, replacing $\xi_i$ with another zero-mean noise vector having the same variance neither improves agent $i$'s net utility nor affects our theoretical results. Hence, we do not include the noise itself as part of the action space. In contrast, the variance of $\xi_i$ can significantly affect agent $i$'s net utility and thus constitutes a meaningful action choice, which is modeled by the scalar $b_i\!\in\! \mathbb{R}$ in~\eqref{strategy}. Such a parametrization is also used in the existing truthfulness literature (see, e.g.,~\cite{honestmain,truthfulmain,truthfulmain2}).
\end{remark}}

\noindent\textbf{Rewards.} We denote a distributed learning algorithm as~$\mathcal{M}$. At each iteration $t$, agent $i$ chooses an action $\alpha_{i,t}\in\mathcal{A}_{i}$. This action produces a (manipulated) gradient $m_{i,t}=\alpha_{i,t}(g_{i}(\theta_{i,t}))$, which is then used by agent $i$ in the update step of $\mathcal{M}$. Considering $T+1$ iterations of $\mathcal{M}$, we let $\boldsymbol{\alpha}_{i}=\{\alpha_{i,t}\}_{t=0}^{T}$ denote the action trajectory of agent $i$ from iteration $0$ to iteration $T$, $\boldsymbol{\theta}_{j}=\{\theta_{j,t}\}_{t=0}^{T}$ denote the model-parameter trajectory of agent $j$, and $\boldsymbol{\theta}_{-i}=\{\boldsymbol{\theta}_{j}\}_{j\in\mathcal{N}_{i}}$ denote the collection of model-parameter trajectories received by agent $i$ from all its neighbors. Given an initial model parameter $\theta_{i,0}$, an action trajectory $\boldsymbol{\alpha}_{i}$, and a collection of model-parameter trajectories $\boldsymbol{\theta}_{-i}$, an implementation of $\mathcal{M}$ generates a final model parameter $\theta_{i,T+1}=\mathcal{M}(\theta_{i,0},\boldsymbol{\alpha}_{i},\boldsymbol{\theta}_{-i})$ for agent $i$. We denote the reward that agent $i$ obtains from its final objective (cost) function value as $R_{i}(f_{i}(\theta_{i,T+1}))$.

In a minimization problem (see~\eqref{primal}), the reward function $R_{i}(f_{i}(\theta))$ increases as the objective function $f_{i}(\theta)$ decreases. Therefore, a self-interested agent can boost its reward by biasing the final solution toward the minimizer of its local objective function. Common choices for $R_{i}(f_{i}(\theta))$ include the linear function $R_{i}(f_{i}(\theta))\!=\!-f_{i}(\theta)$ and the sigmoid-like function $R_{i}(f_{i}(\theta))\!=\!(1+e^{-1/f_{i}(\theta)})^{-1}$~\cite{truthfulmain2}. We allow different agents to have different reward functions.

\noindent\textbf{Payments and net utilities.} To mitigate gradient manipulation by participating agents, we augment the distributed learning protocol with a payment mechanism. This mechanism can be computed efficiently and implemented in a fully decentralized manner between any pair of interacting agents (see details in Section~\ref{SectionIII}). We denote the augmented distributed learning protocol (with a payment mechanism) by $\mathcal{M}_{p}$. The net utility of agent $i$ from executing $\mathcal{M}_{p}$ over $T+1$ iterations is defined as follows:
\begin{equation}
\vspace{-0.2em}
U_{i,0\text{\tiny$\to$}T}^{\mathcal{M}_{p}}(\boldsymbol{\alpha}_{i},\boldsymbol{\alpha}_{-i})=R_{i}(f_{i}(\theta_{i,T+1}))-\sum_{t=0}^{T}\!P_{i,t},
\label{netutility}
\vspace{-0.2em}
\end{equation}
where $P_{i,t}$ is the total net payment made by agent $i$ to all its neighbors, and $\boldsymbol{\alpha}_{-i}=\{\boldsymbol{\alpha}_{j}\}_{j\neq i}$ denotes the action trajectories  of all agents except agent $i$.

We note that all existing incentive-based approaches for collaborative learning ensure truthfulness by incorporating a payment or penalty term into each agent’s net utility. Without such payments, a self-interested agent can freely manipulate its gradients to reduce its own loss and increase its rewards, thereby distorting the collaborative learning process. Therefore,~\eqref{netutility} includes the cumulative payments of each agent in its net utility. Accordingly, a rational agent must consider both its rewards and payments when maximizing its net utility.

Next, we introduce two truthfulness-related concepts in our game-theoretic framework.
\begin{definition}[$\delta$-truthful action~\cite{truthfulmain,truthfulmain2}]\label{D1}
For any given $\delta\geq 0$ and any $i\in[N]$, an action $\alpha_{i}\in\mathcal{A}_{i}$ (with $\mathcal{A}_{i}$ defined in~\eqref{strategy}) of agent $i$ is $\delta$-truthful if it satisfies $\mathbb{E}[\|\alpha_{i}(g_{i}(\theta))-g_{i}(\theta)\|]\leq \delta$ for any $\theta\in\mathbb{R}^{n}$. In particular, the action $\alpha_{i}$ is fully truthful when $\delta=0$.
\end{definition}
Definition~\ref{D1} quantifies the truthfulness of an agent's action in collaborative  optimization/learning. It can be seen that a smaller $\delta$ corresponds to a higher level of truthfulness in the agent's action.

\begin{definition}[$\epsilon$-incentive compatibility~\cite{IC2007}]\label{D3}
For any given $\epsilon\geq 0$, a distributed learning protocol $\mathcal{M}_{p}$ is $\epsilon$-incentive compatible if for all $i\in[N]$, $\mathbb{E}[U_{i,0\text{\tiny$\to$}T}^{\mathcal{M}_{p}}(\boldsymbol{h}_{i},\boldsymbol{h}_{-i})]\geq\mathbb{E}[U_{i,0\text{\tiny$\to$}T}^{\mathcal{M}_{p}}(\boldsymbol{\alpha}_{i},\boldsymbol{h}_{-i})]-\epsilon$ holds for any arbitrary action trajectory $\boldsymbol{\alpha}_{i}$ of agent $i$, where $\boldsymbol{h}_{i}$ is the truthful action trajectory of agent $i$ and $\boldsymbol{h}_{-i}$ is truthful action trajectories of all agents except agent~$i$.
\end{definition}
Definition~\ref{D3} (also called $\epsilon$-Bayesian-incentive compatibility) is a standard
and commonly used notion in the existing incentive-compatibility literature~\cite{IC1,IC2,truthfulmain,truthfulmain2}. It implies that if a distributed optimization protocol is~$\epsilon$-incentive compatible, then the expected net utility that an agent can gain from any (possibly untruthful) action trajectory is at most $\epsilon$ greater than that obtained by being truthful in all iterations. Clearly, a smaller $\epsilon$ corresponds to a lower gain that an untruthful agent can obtain. In addition, according to the definition of $\epsilon$-Nash equilibrium in~\cite{Nash}, if a distributed learning protocol is $\epsilon$-incentive compatible, then the truthful action trajectory profile of all agents $\boldsymbol{h}=(\boldsymbol{h}_{i},\boldsymbol{h}_{-i})$ forms an $\epsilon$-Nash equilibrium (see Lemma~\ref{ICNashlemma}).

{\color{blue}Although Bayesian incentive compatibility (BIC) is weaker than dominant-strategy incentive compatibility (DSIC), we adopt BIC for two reasons. First, existing works, such as~\cite{DSIC1} and~\cite{DSIC2}, have shown that DSIC is \textbf{unattainable} in the absence of a centralized authority with access to all agents' actual information. Consequently, DSIC is incompatible with our considered distributed learning settings, where no centralized coordinator exists to aggregate actual information from all agents. Second, BIC provides mechanism designers with greater flexibility to achieve key objectives than DSIC, and therefore \textbf{is more practical} in real-world distributed applications. For example, \cite{DSIC3} shows that budget balance can be achieved under BIC but is generally unattainable under DSIC. In auction design, \cite{DSIC4} demonstrates that both optimal social welfare and optimal revenue can be mathematically characterized within the BIC framework, whereas such characterizations are infeasible under DSIC.}

\section{Payment Mechanism Design for Distributed Stochastic Gradient Descent}\label{SectionIII}
In this section, we propose a fully distributed payment mechanism (see Mechanism~\ref{mechanism}) to incentivize truthful behaviors of participating agents in distributed stochastic gradient descent.

\begin{mechanism}[h]
\caption{Fully distributed payment mechanism  (for  interacting agents $i$ and $j$ at
iteration $t$)}
\label{mechanism}
\begin{algorithmic}[1]
\STATE {\bfseries Input:}
$\theta_{\iota,t-1}$, $\theta_{\iota,t}$, $\theta_{\iota,t+1}$ for $\iota\in\{i,j\}$ available to both agents $i$ and $j$ under Algorithm~\ref{algorithm} (note $\theta_{\iota,t+1}$ has been shared at the end of iteration $t$); initialization $\theta_{i,-1}=\theta_{j,-1}=\boldsymbol{0}_{n}$; $C_{t}>0$. 
\STATE Agents $i$ and $j$ simultaneously compute 
$\Delta_{\theta_{\iota},t}\triangleq\|\theta_{\iota,t+1}\!-2\theta_{\iota,t}\!+\theta_{\iota,t-1}\|^2$ for both $\iota\in\{i,j\}$.
\IF{$\Delta_{\theta_{i},t}\geq \Delta_{\theta_{j},t}$}
\STATE Agent $i$ transfers $P_{i,t}^{j}=C_{t}(\Delta_{\theta_{i},t}-\Delta_{\theta_{j},t})$ to agent $j$.
\ELSE
\STATE Agent $i$ receives $P_{j,t}^{i}\!=\!C_{t}(\Delta_{\theta_{j},t}-\Delta_{\theta_{i},t})$ from agent $j$.
\ENDIF
\end{algorithmic}
\end{mechanism}
Mechanism~\ref{mechanism} is implementable in a fully distributed manner without the assistance of any   server or aggregator. At each iteration $t$, if agent $i$ manipulates its gradient estimates such that its model-parameter increment $\|\theta_{i,t+1}-2\theta_{i,t}+\theta_{i,t-1}\|$ exceeds that of its neighbor agent~$j$, agent $i$ pays an amount $P_{i,t}^{j}> 0$ to agent $j$ (in this case, we denote the payment of agent $j$ as $P_{j,t}^{i}=-P_{i,t}^{j})$. Conversely, if agent $i$'s model-parameter increment $\|\theta_{i,t+1}-2\theta_{i,t}+\theta_{i,t-1}\|$ is no greater than that of its neighbor agent $j$, it receives a payment of amount $P_{j,t}^{i}\geq 0$ from agent $j$ (in this case, we denote the payment of agent $i$ as $P_{i,t}^{j}=-P_{j,t}^{i}$). We emphasize that our payment mechanism can be readily applied to any first-order  distributed gradient methods. For an agent $i$, its total payment to all its neighbors at iteration $t$ is given by $P_{i,t}=\sum_{j \in \mathcal{N}_{i}} P_{i,t}^{j}$.

In Mechanism~\ref{mechanism}, with both agents $i$ and $j$ having access to $\theta_{\iota,t-1}$, $\theta_{\iota,t}$, and $\theta_{\iota,t+1}$ for $\iota\in\{i,j\}$ from the update of Algorithm~\ref{algorithm} (note that $\theta_{\iota,t+1}$ has been shared at the end of iteration $t$), the two  agents can cross-verify the computed payment value, making the mechanism robust to unilateral manipulation. This represents a significant advance compared with the payment mechanism in~\cite{VCGauto} for server-assisted collaborative optimization, which requires all agents to truthfully report their local objective-function values for payment calculation—thereby creating a risk that strategic agents may manipulate the algorithmic update and the payment mechanism separately.

{\color{blue}In Mechanism~\ref{mechanism}, we use the second-order difference rather than the first-order difference because it is more sensitive to strategic gradient manipulation. This distinction is also reflected in our payment analysis (see \eqref{C1L104}). Specifically, with the second-order difference, the terms associated with strategic gradient manipulation scale as $\lambda_t^2$, whereas the residual term arising from the optimization error in the absence of strategic gradient manipulation scales as $\lambda_t^2\lambda_{t-1}^2$. In contrast, with the first-order difference, both terms scale as $\lambda_t^2$. Therefore, the second-order difference reduces the relative contribution of normal optimization error to the payment measure, making it more sensitive to strategic gradient manipulation.}

Our payment mechanism can effectively discourage agents from free-riding\footnote{Free-riding refers to the behavior in which agents skip computing gradients on their local data and instead update their parameters solely based on information received from neighbors. In this way, they benefit from the distributed optimization process without contributing their own data or gradients.}. Specifically, the model-parameter increment $\|\theta_{i,t+1}-2\theta_{i,t}+\theta_{i,t-1}\|$ of agent $i$ depends on both the consensus errors and the (sign-indefinite) local gradients. Consequently, even if agent $i$ uses a zero (or low) gradient, there is no guarantee that  $\|\theta_{i,t+1}-2\theta_{i,t}+\theta_{i,t-1}\|$ will be smaller than that of its neighbor $j$, meaning that leveraging zero gradients does not reliably increase agent $i$'s payment gains. In fact, free-riding behavior (or using low gradients) invariably degrades agent $i$’s own reward $R_{i}(f_{i}(\theta_{i,T}))$, as it weakens the influence of agent $i$'s data in collaborative learning and ultimately leads to a worse final model for the agent itself. Hence, free-riding is not a utility-improving choice for a rational agent.

Our payment mechanism is conceptually inspired by the classical VCG mechanism~\cite{VCG1,VCG2,VCG3} but has several fundamental differences: 1) conventional VCG mechanisms require a central server to calculate and collect payments~\cite{VCG1,VCG2,VCG3}, whereas our mechanism operates in a pairwise fashion without reliance on any third party. As a result, it can be implemented in a fully distributed manner; 2) conventional VCG mechanisms are typically designed for one-shot games~\cite{oneshot2,oneshot4,oneshot5,oneshot7}. In contrast, our payment mechanism is naturally compatible with iterative algorithms, encompassing a wide range of distributed optimization/learning methods. In fact, an iterative algorithm setting inherently forms a multi-stage game, where agents repeatedly adjust actions, posing challenges for both truthfulness and convergence analysis; and 3) conventional VCG mechanisms are not budget-balanced (see, e.g.,~\cite{iterative1,IcentiveFL0,IcentiveFL1,IcentiveFL2,VCGauto}), whereas our payment mechanism is budget-balanced, i.e., $\sum_{i=1}^{N}P_{i,t}=0$, making it financially sustainable and scalable in practice.

{\color{blue}
Mechanism~\ref{mechanism} offers computational advantages over conventional VCG mechanisms because it computes payments directly from differences in model parameters, thereby avoiding the potentially costly evaluation of each agent's externality, which typically requires solving additional optimization problems. Moreover, unlike conventional VCG mechanisms, our mechanism does not rely on a centralized regulator to calculate, collect, and redistribute payments. This computational efficiency and decentralized implementation are enabled by the fact that, in our setting, parameter differences provide tractable bounds on the gains from gradient manipulation (see, e.g., Eqs.~\eqref{C1L1},~\eqref{C1L104}, and~\eqref{C1L3}). Furthermore, by appropriately choosing $C_{t}$ (e.g., as specified in Lemma~\ref{ML1}), we can regulate agents' incentives to manipulate and bound their potential gains from manipulation, thereby ensuring $\epsilon$-incentive compatibility while maintaining accurate convergence.}

The existing approaches most closely related to ours are the payment mechanisms proposed by~\cite{truthfulmain,truthfulmain2}. However, there are several fundamental differences: 1) our payment mechanism is implementable in a fully distributed manner, and hence, is applicable to arbitrary connected communication graphs, whereas the mechanisms in~\cite{truthfulmain,truthfulmain2} rely on a centralized server to aggregate gradients from all agents, and  thus operate only under a centralized communication structure; and 2) our payment mechanism achieves $\epsilon$-incentive compatibility with a finite $\epsilon$ even in an infinite time horizon (see Theorem~\ref{MT2}), whereas the incentive $\epsilon$ in~\cite{truthfulmain,truthfulmain2} becomes unbounded as the number of iterations tends to infinity (see Claim 23 in~\cite{truthfulmain} or Theorem 5.1 in~\cite{truthfulmain2}), leading to a vanishing incentive compatibility guarantee over iterations.

{\color{blue}\begin{remark}
Note that we do not consider malicious or adversarial agents who seek to disrupt the system without regard for their own utility. Such agents may even receive payments from others by consistently using very small gradient values. Instead, our work focuses on strategic agents who manipulate their gradients through amplification and additive noise to influence the final network-level computation in ways that benefit themselves. Designing mechanisms to deter or mitigate attacks aimed at disrupting the system remains an important direction for future work.
\end{remark}}

\section{Convergence Rate and Incentive-Compatibility Analysis}\label{SectionIV}
\begin{assumption}\label{A1}
For any $i\in[N]$, the following conditions hold: (i) $R_{i}(f_{i}(\theta))$ is $L_{R,i}$-Lipschitz continuous \textit{w.r.t.}~$\theta$; (ii) the stochastic gradient $g_{i}(\theta)$ is $H_{i}$-Lipschitz continuous. Moreover, $g_{i}(\theta)$ is unbiased and has bounded variance $\sigma_{i}^2$; (iii) for a general convex $f_{i}(\theta)$, we assume that $f_{i}(\theta)$
is $L_{f,i}$-Lipschitz continuous. However, this assumption is not required for a strongly
convex $f_{i}(\theta)$.
\end{assumption}

Assumption~\ref{A1} is standard and commonly used in the distributed stochastic optimization/learning literature~\cite{Lip1,Lip2,Lip3}. It is worth noting that we allow $f_{i}(\theta)$ to be convex, which is more general than the strongly convex assumption used in existing truthfulness results in~\cite{JDP1,honestmain,truthfulmain,truthfulmain2,VCGauto}. 

Furthermore, for the sake of notational simplicity, we denote $L_{R}=\max_{i\in[N]}\{L_{R,i}\}$,  $L_{f}=\max_{i\in[N]}\{L_{f,i}\}$, $H=\max_{i\in[N]}\{H_{i}\}$, and $\sigma=\max_{i\in[N]}\{\sigma_{i}\}$.
\subsection{Convergence rate analysis}\label{sectionIVA}
We first prove that, under Mechanism~\ref{mechanism}, the incentive for a strategic agent in Algorithm~\ref{algorithm} to deviate from truthful behavior diminishes to zero. We then analyze the convergence rates of Algorithm~\ref{algorithm} in the presence of strategic behaviors, for strongly convex and general convex objective functions, respectively.
\begin{lemma}\label{ML1}
{\color{blue}We define $\mathcal{A}_{i}^{+}$ as the set of admissible actions $\alpha_{i,t}$ that yield agent~$i$ a nonnegative net-utility gain relative to its truthful action, with all other agents acting truthfully.} Under Assumptions~\ref{A2} and~\ref{A1}, for any $i\in[N]$, $\delta>0$, and $t\geq 0$, if we set $C_{t}= \frac{4L_{R}\sqrt{6d_{t+1}}}{\min_{i\in[N]}\{\deg(i)\}\lambda_{t}\kappa_{t}\delta}$ with $\lambda_{t}=\frac{\lambda_{0}}{(t+1)^{v}}$, $\kappa_{t}=\frac{1}{(t+1)^{r}}$, $v\in(\frac{1}{2},\frac{2}{3})$, $r\in(1-v,v)$, and $d_{t}=e^{\frac{20H^2\lambda_{0}^2}{(1-\rho)(2v-1)}}t^{1-2v}$, {\color{blue}then every action $\alpha_{i,t}\in\mathcal{A}_{i}$ of agent $i$} in Algorithm~\ref{algorithm} is $\kappa_{t}\delta$-truthful. That is, for all $t\geq 0$ and any $\delta>0$, the following inequality holds:
\begin{equation}
{\color{blue}\mathbb{E}[\|\alpha_{i,t}(g_{i}(\theta_{i,t}))-g_{i}(\theta_{i,t})\|]\leq \kappa_{t}\delta,\quad\forall i\in[N].}\label{ML1result}
\end{equation}
Moreover, as the number of iterations tends to infinity, each agent's incentive to deviate from truthful behavior diminishes to zero, i.e., $\lim_{t\rightarrow\infty}\mathbb{E}[\|\alpha_{i,t}(g_{i}(\theta_{i,t}))-g_{i}(\theta_{i,t})\|]=0$.
\end{lemma}
\begin{proof}
See Appendix C.
\end{proof}

{\color{blue}In Lemma~\ref{ML1}, $C_{t}$ denotes the payment coefficient in Mechanism~\ref{mechanism}. In the definition of $C_{t}$, $L_{R}$ is a Lipschitz constant for the reward function $R_{i}(f_{i}(\theta))$, as specified in Assumption~\ref{A1}; $\deg(i)$ denotes the number of neighbors of agent~$i$; $\lambda_{t}$ is the stepsize; and $\delta>0$ is a prescribed truthfulness parameter defined in Definition~\ref{D1}. The sequences $d_{t}$ and $\kappa_{t}$ are judiciously designed decaying sequences that gradually reduce agents' incentives to deviate from truthful behavior. In~\eqref{ML1result}, $g_{i}(\theta_{i,t})$ represents the truthful gradient, whereas $\alpha_{i,t}(g_{i}(\theta_{i,t}))$ represents agent $i$'s potentially manipulated gradient, obtained by applying the action operator $\alpha_{i,t}$.}

Lemma~\ref{ML1} proves that when the neighbors of agent $i$ are truthful, the incentive for agent $i$ to deviate from truthful behavior converges to zero. This stands in sharp contrast to existing payment mechanisms in~\cite{honestmain,truthfulmain,truthfulmain2} for federated learning, which guarantees only a bounded---yet non-diminishing---incentive for untruthful behavior at each iteration, thus leaving agents with a persistent motive to act untruthfully. Furthermore, Lemma~\ref{ML1} implies that by choosing an arbitrarily small $\delta>0$, we can ensure that the optimal action of each agent $i$ is arbitrarily close to being fully truthful at every iteration. It is worth noting that since the differences in model-parameter increments in Mechanism~\ref{mechanism} diminish to zero, we can ensure $\lim_{t\rightarrow\infty}\mathbb{E}[P_{i,t}]=0$ (as shown in Corollary~\ref{plimit} in Appendix G). This guarantees that no payment is required from agent $i$ when it behaves truthfully as the number of iterations tends to infinity.

{\color{blue}
The action $\alpha_{i,t}$ is a pair of scalars $(a_{i,t},b_{i,t})$ chosen by agent $i$ at iteration $t$ to achieve a nonnegative net-utility gain $\mathbb{E}[R_{i}(f_{i}(\theta_{i,T+1}^{\prime}))-P_{i,t}^{\prime}]-\mathbb{E}[R_{i}(f_{i}(\theta_{i,T+1}))-P_{i,t}]$ (see~\eqref{C1L3}). We note that such an action depends on the specific forms of the objective and reward functions. Consequently, a general closed-form characterization of the optimal action is typically unavailable for the general distributed optimization problem in~\eqref{primal}.}

\begin{theorem}[Convergence rate]\label{MT1}
We denote $\theta^{*}$ as a solution to the problem in~\eqref{primal}. Under our Mechanism~\ref{mechanism} and the conditions in Lemma~\ref{ML1}, for any $i\in[N]$, $\delta>0$, and $T\geq0$, the following results hold for Algorithm~\ref{algorithm} in the presence of strategic behaviors: 

(i) if $f_{i}(\theta)$ is $\mu$-strongly convex (not necessarily Lipschitz continuous), then we have
\begin{equation}
\mathbb{E}[\|\theta_{i,T}-\theta^{*}\|^2]\leq \mathcal{O}\left(\frac{H^{2}(\sigma^2+\delta^{2})}{\mu(1-\rho)^2(T+1)^{v}}\right);\nonumber
\end{equation}

(ii) if $f_{i}(\theta)$ is general convex, then we have
\begin{equation}
\frac{1}{T+1}\sum_{t=0}^{T}\mathbb{E}[F(\theta_{i,t})-F(\theta^{*})]\leq \mathcal{O}\left(\frac{H^{2}(\sigma^{2}+L_{f}^2+\delta^2)}{(1-\rho)^2(T+1)^{1-v}}\right).\nonumber
\end{equation}
\end{theorem}
\begin{proof}
See Appendix D.
\end{proof}
Theorem~\ref{MT1} proves that, even in the presence of strategic behaviors, our proposed Mechanism~\ref{mechanism} ensures convergence to an exact
optimal solution $\theta^{*}$ to the problem in~\eqref{primal} at rates $\mathcal{O}(T^{-v})$ and $\mathcal{O}(T^{-(1-v)})$ for strongly convex and general convex $f_{i}(\theta)$, respectively. It is broader than existing server-assisted truthfulness results in~\cite{JDP1,honestmain,truthfulmain,truthfulmain2,VCGauto} that focus solely on the strongly convex case. Moreover, our results are in stark contrast to existing JDP-based truthfulness results in~\cite{JDP1,JDP2,JDP3,JDP4,JDPziji} and incentive-based truthfulness results in~\cite{honestmain,truthfulmain,truthfulmain2}, all of which are subject to optimization errors.

{\color{blue}\begin{remark}
Our work fundamentally differs from existing robustness analyses in three aspects: (i) gradient manipulation in our setting is a strategic choice made by self-interested agents to increase their net utilities (see~\eqref{strategy}), rather than a prescribed disturbance; (ii) our goal is to design a payment mechanism to discourage strategic gradient manipulation (see Lemma \ref{ML1}) while ensuring accurate convergence (see Theorem \ref{MT1}), rather than merely evaluating robustness against prescribed disturbances; and (iii) we prove that each agent's net-utility gain from manipulation is bounded by $\epsilon$, ensuring $\epsilon$-incentive compatibility (see Theorem \ref{MT2}) and making the truthful action-trajectory profile an $\epsilon$-Nash equilibrium (see Lemma \ref{ICNashlemma}).
\end{remark}}

\subsection{Incentive-compatibility analysis}\label{SectionIVB}
In addition to achieving accurate convergence, our fully distributed payment mechanism also simultaneously ensures that Algorithm~\ref{algorithm} is $\epsilon$-incentive compatible. 
\begin{theorem}[Incentive compatibility]\label{MT2}
Under our fully distributed payment mechanism and the conditions in Lemma~\ref{ML1}, Algorithm~\ref{algorithm} is $\epsilon$-incentive compatible, regardless of whether $F(\theta)$ is general convex or strongly convex. Namely, for any $i\in[N]$, $\delta>0$, and $T\geq 0$ (which includes the case of $T=\infty$), the following inequality always holds:
\begin{equation}
\mathbb{E}[U_{i,0\text{\tiny$\to$}T}^{\mathcal{M}_{p}}(\boldsymbol{\alpha}_{i},\boldsymbol{h}_{-i})-U_{i,0\text{\tiny$\to$}T}^{\mathcal{M}_{p}}(\boldsymbol{h}_{i},\boldsymbol{h}_{-i})]\leq \epsilon,\label{MT2result1}
\end{equation}
with $U_{i,0\text{\tiny$\to$}T}^{\mathcal{M}_{p}}(\boldsymbol{\alpha}_{i},\boldsymbol{h}_{-i})$ and~$U_{i,0\text{\tiny$\to$}T}^{\mathcal{M}_{p}}(\boldsymbol{h}_{i},\boldsymbol{h}_{-i})$ defined in~\eqref{netutility} and $\epsilon$ given by $\epsilon=\mathcal{O}\left(\frac{(v+r)L_{R}\delta}{v+r-1}\right)$.
\end{theorem}
\begin{proof}
See Appendix E.
\end{proof}
{\color{blue}Theorem~\ref{MT2} provides an a posteriori analysis of the expected cumulative increase in agent $i$'s net utility over $T$ iterations that can be obtained through strategic manipulation of its gradient estimates, relative to truthful participation at every iteration.}
It ensures that the cumulative gain from
an agent $i$'s untruthful behaviors in Algorithm~\ref{algorithm} remains finite, even when $T\rightarrow\infty$. This contrasts with existing truthfulness results for server-assisted
federated learning (e.g.,~\cite{truthfulmain,truthfulmain2}), where $\epsilon$ explodes as the iteration proceeds, implying that truthfulness/incentive compatibility will eventually be lost.

Existing incentive-based truthfulness results for server-assisted federated learning in, e.g.,~\cite{honestmain,truthfulmain,truthfulmain2}, do not provide {\bf simultaneous} guarantees for both $\epsilon$-incentive compatibility and accurate convergence. Specifically, the convergence analysis in~\cite{honestmain} requires two conditions: (i) $P(\exists t\leq T: \Pi_{W}(\theta_t^s-\gamma_t\bar{m}_t)\neq \theta_t^s-\gamma_t\bar{m}_t)\in\mathcal{O}(\frac{1}{NT})$ and (ii) the boundedness of $W$ (see Theorem 6.1 in~\cite{honestmain}), where $\theta_t^s$ is the model parameter computed by the centralized server, $W$ is a projection set, $\gamma_t$ is the stepsize, and $\bar{m}_t$ is the average (manipulated) gradients reported by all agents. Moreover, in the Appendix section ``Discussion on the projection assumptions", they state that Condition (i) can be guaranteed  when $W$ grows at a rate of $\Omega(T)$ for general strongly convex functions, which is at odds with the boundedness requirement on $W$ in Condition (ii) when $T$ tends to infinity (their convergence error $\mathcal{O}(\frac{1+M+\epsilon^2}{NT})+\mathcal{O}(\frac{1}{T^2})$ in Theorem 6.1 is strictly larger than $0$ unless $T$ is allowed to approach infinity). Therefore, they did not provide a method for ensuring that both conditions hold simultaneously under general strongly convex objectives. A similar issue also exists in~\cite{truthfulmain} (see Theorem 9 and footnote 2 therein). Although the convergence analysis in~\cite{truthfulmain2} removes these two conditions, its definition  $G=\sum_{t=1}^T\gamma_{t}\sqrt{\mathcal{C}_t}$ in Theorem 5.1 implies that $\epsilon=\mathcal{O}(G)$ is finite only when $T$ is finite, indicating that both its $\epsilon$-incentive compatibility and convergence statements fail to hold in an infinite time horizon.

{\color{blue}\begin{remark}
Eliminating the need for a centralized coordinator or a trusted third party comes at the cost of weaker incentive compatibility guarantees. In a fully distributed setting, payments are computed and transferred based solely on local information. However, incentive compatibility requires limiting each agent's net-utility gain from deviations that may influence the learning process across the network. This discrepancy makes it challenging to guarantee exact incentive compatibility in the absence of a centralized coordinator or a trusted third party. Hence, we 
provide an $\epsilon$-incentive compatibility guarantee for some $\epsilon>0$, rather than exact incentive compatibility usually attained by existing VCG-based works (e.g.,\cite{newlyadd3}), reflecting a tradeoff between decentralization and the achievable strength of incentive compatibility. Consequently, the necessary and sufficient conditions in \cite{newlyadd3} for characterizing Groves-type mechanisms under exact (ex-post) incentive compatibility do not directly apply to our~setting.
\end{remark}
\begin{remark}
It is worth noting that strong incentive compatibility comes at the expense of an increased payment burden on agents. According to Theorem~\ref{MT2}, a smaller $\epsilon$ (i.e., stronger incentive compatibility) requires a smaller $\delta$, which, under Mechanism~\ref{mechanism}, leads to a larger $C_{t}$ and consequently higher payments by the agents. The increased payment may potentially increase the cost borne by some participating agents, thereby affecting the economic feasibility of their participation.
\end{remark}
\begin{definition}[$\epsilon$-Nash equilibrium]\label{DNash}
We let $(\boldsymbol{\alpha}_{1},\cdots,\boldsymbol{\alpha}_{N})\in \mathcal{P}(\mathcal{A}_{1}^{T})\times\cdots\times \mathcal{P}(\mathcal{A}_{N}^{T})$ be the action trajectory profile of $N$ agents, where $\mathcal{A}_{i}^{T}$ represents the $T$-fold Cartesian product of $\mathcal{A}_{i}$, and 
$\mathcal{P}(\mathcal{A}_{i}^{T})$ represents the set of all probability measures over $\mathcal{A}_{i}^{T}$. Then, we say $\boldsymbol{\alpha}^{*}=(\boldsymbol{\alpha}_{1}^{*},\cdots,\boldsymbol{\alpha}_{N}^{*})$ is an $\epsilon$-Nash equilibrium~\textit{w.r.t.} the net utility $U_{i,0\text{\tiny$\to$}T}^{\mathcal{M}_{p}}$ defined in~\eqref{netutility} if for any $i\in[N]$ and $\boldsymbol{\alpha}_{i}\in\mathcal{P}(\mathcal{A}_{i}^{T})$, the following inequality holds:
\begin{equation}
\begin{aligned}
&\textstyle\mathbb{E}\left[U_{i,0\text{\tiny$\to$}T}^{\mathcal{M}_{p}}(\boldsymbol{\alpha}_{1}^{*},\cdots,\boldsymbol{\alpha}_{i}^{*},\cdots,\boldsymbol{\alpha}_{N}^{*})\right]\\
&\textstyle\geq \mathbb{E}\left[U_{i,0\text{\tiny$\to$}T}^{\mathcal{M}_{p}}(\boldsymbol{\alpha}_{1}^{*},\cdots,\boldsymbol{\alpha}_{i},\cdots,\boldsymbol{\alpha}_{N}^{*})\right]-\epsilon.\nonumber
\end{aligned}
\end{equation}
\end{definition}
Definition~\ref{DNash} implies that, in an $\epsilon$-Nash equilibrium, no agent can improve its net utility by more than $\epsilon$ through a unilateral deviation from its equilibrium action trajectory.
\begin{lemma}\label{ICNashlemma}
Under our distributed payment mechanism and the conditions in Lemma~\ref{ML1}, the truthful action trajectory profile of all agents $\boldsymbol{h}=(\boldsymbol{h}_{1},\cdots,\boldsymbol{h}_{N})$ is an $\epsilon$-Nash equilibrium.
\end{lemma}
\begin{proof}
See Appendix F.
\end{proof}

We note that keeping the other agents' actions truthful when evaluating each agent's admissible unilateral deviations is the standard approach to determining whether the truthful action trajectory profile constitutes an $\epsilon$-Nash equilibrium, as adopted in Definition 1 of~\cite{honestmain} and the proof of Theorem D.1 therein.}

\begin{figure*}[!t]
\centering
\subfloat[FeMNIST dataset]{\includegraphics[width=1.73in]{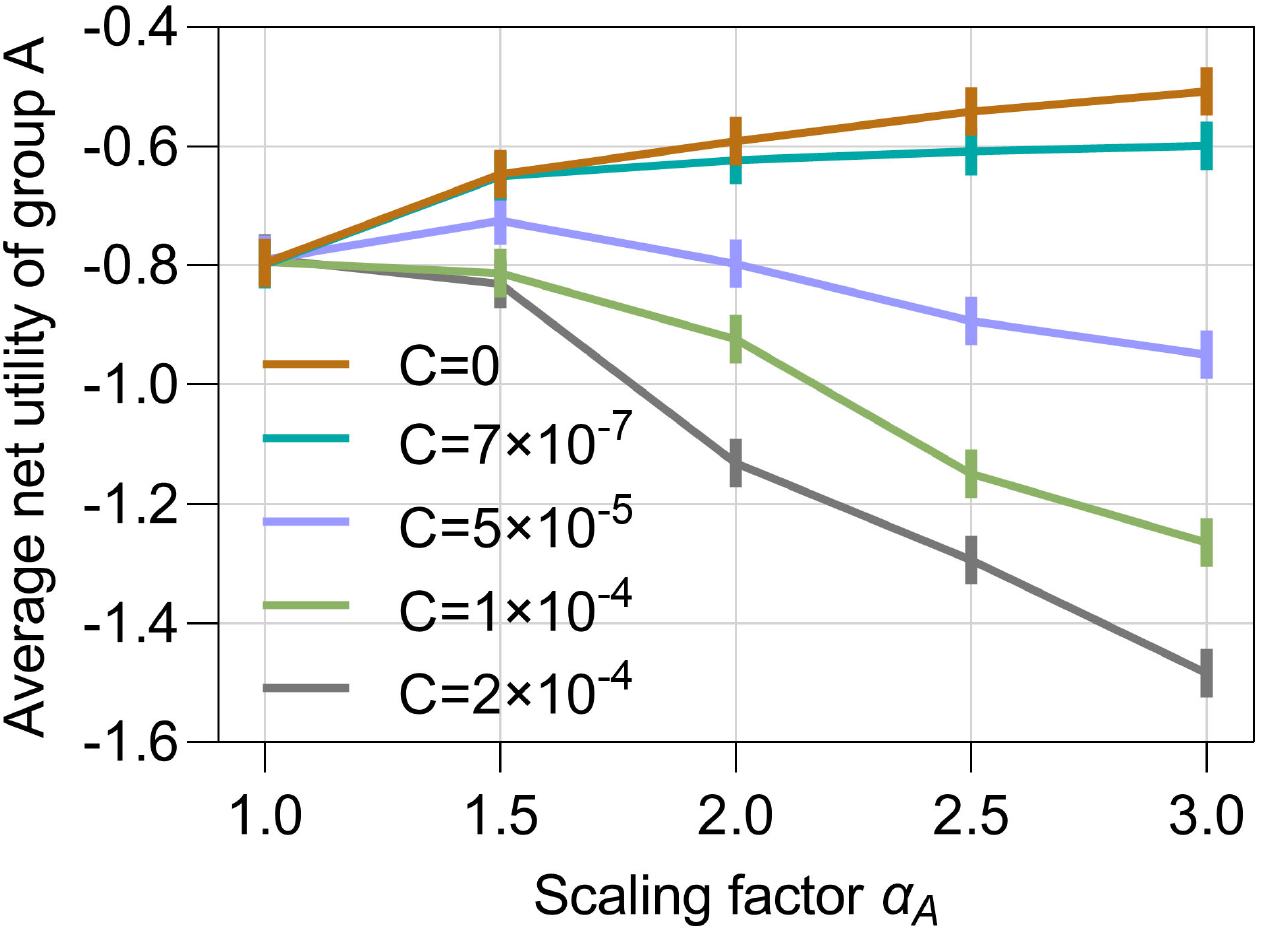}%
\label{fig11}}
\hspace{0.2em}
\subfloat[FeMNIST dataset]{\includegraphics[width=1.73in]{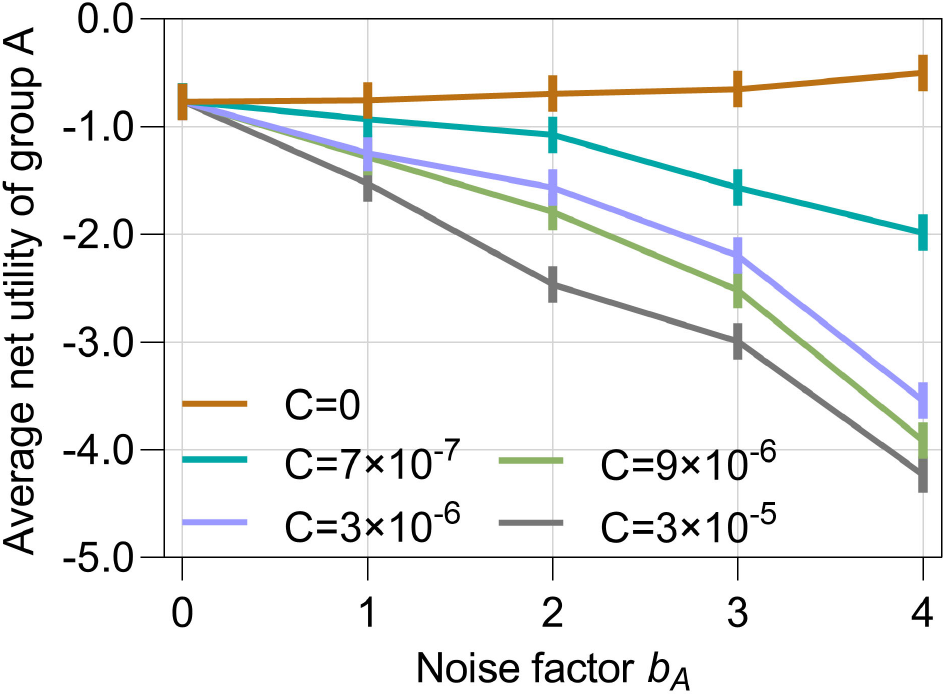}%
\label{fig12}}
\hspace{0.2em}
\subfloat[Shakespeare dataset]{\includegraphics[width=1.73in]{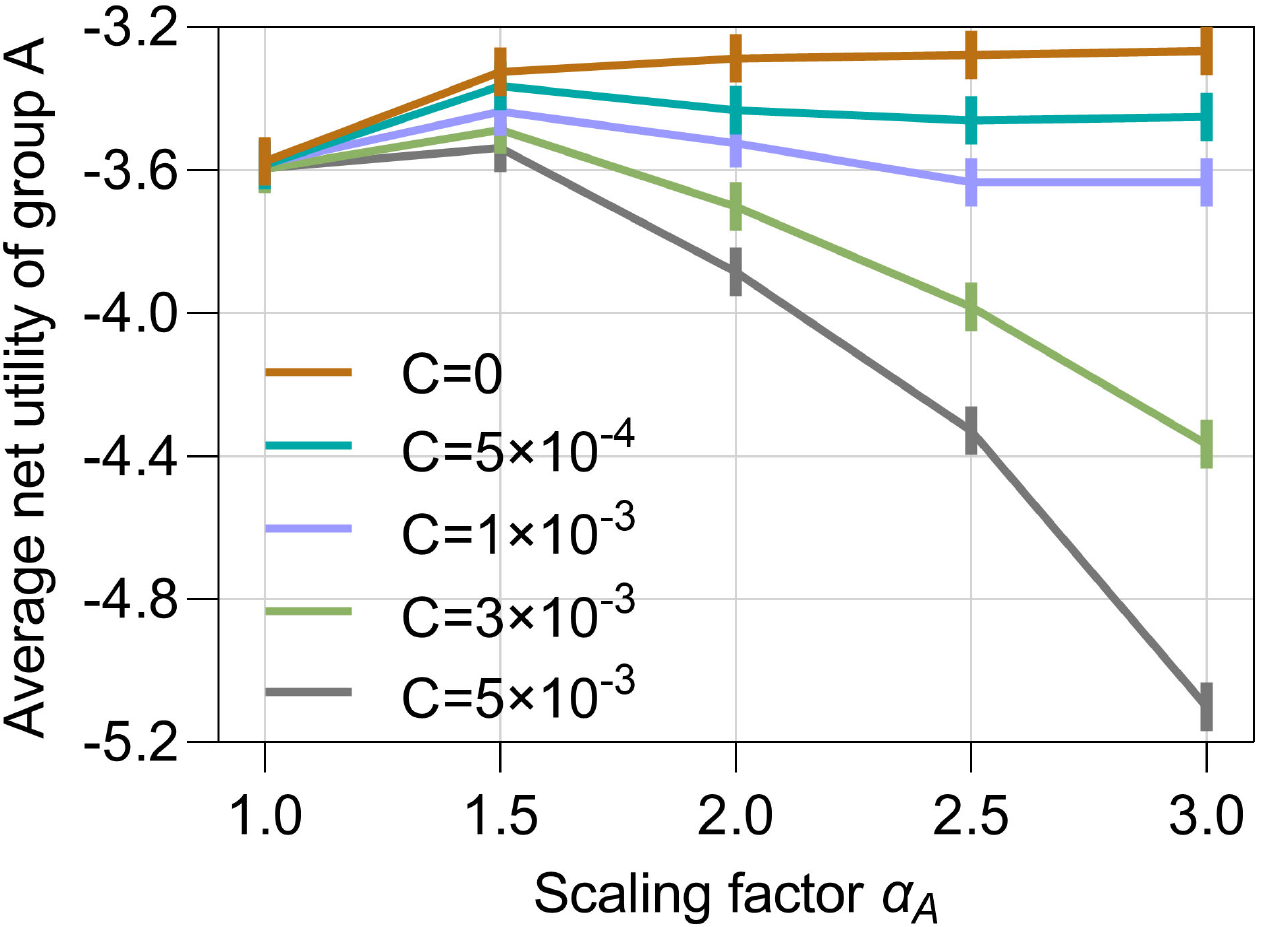}%
\label{fig13}}
\hspace{0.2em}
\subfloat[Shakespeare dataset]{\includegraphics[width=1.73in]{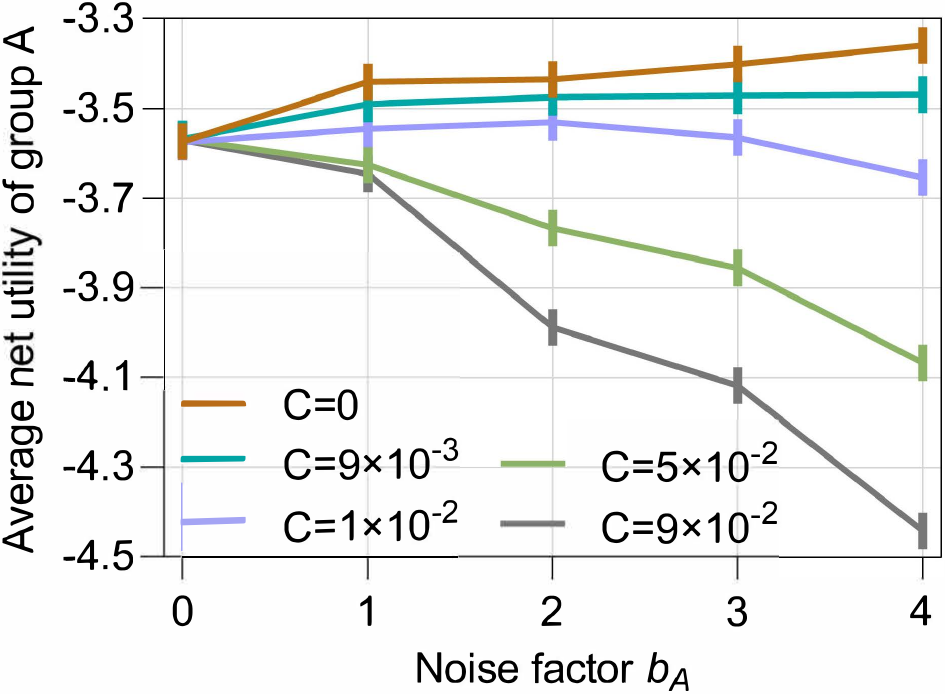}%
\label{fig14}}
\caption{Average net utilities of group-A agents under varying scaling factors $a_{A}$ (with $b_{A}=0$) and varying noise factors $b_{A}$ (with $a_{A}=1$) for different payment coefficient $C$, respectively. The error bars represent standard errors over $10$ runs.}
\label{Fig1}
\vspace{-0.5em}
\end{figure*}
\begin{figure*}[!t]
\centering
\subfloat[FeMNIST dataset]{\includegraphics[width=1.73in]{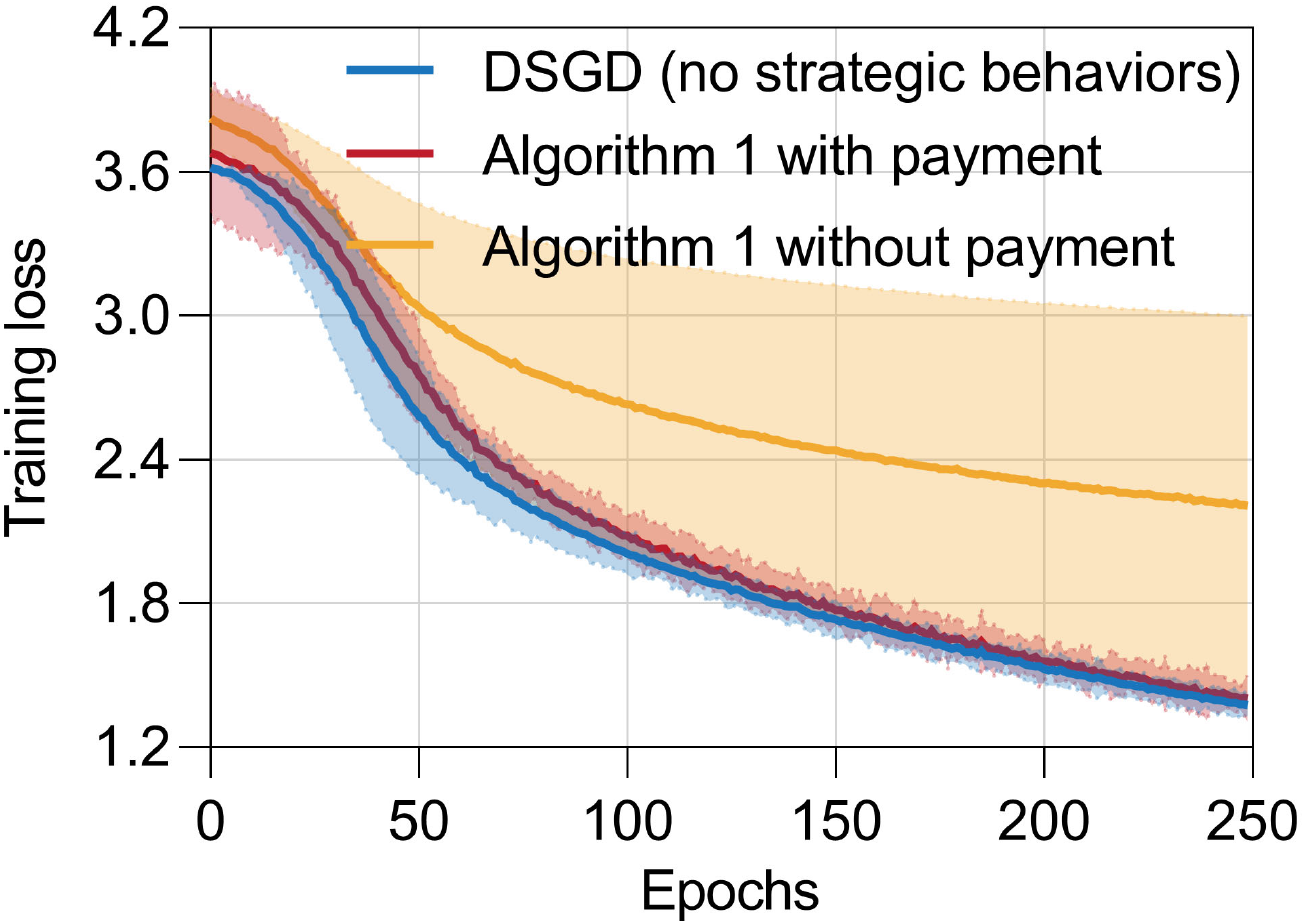}%
\label{fig21}}
\subfloat[FeMNIST dataset]{\includegraphics[width=1.73in]{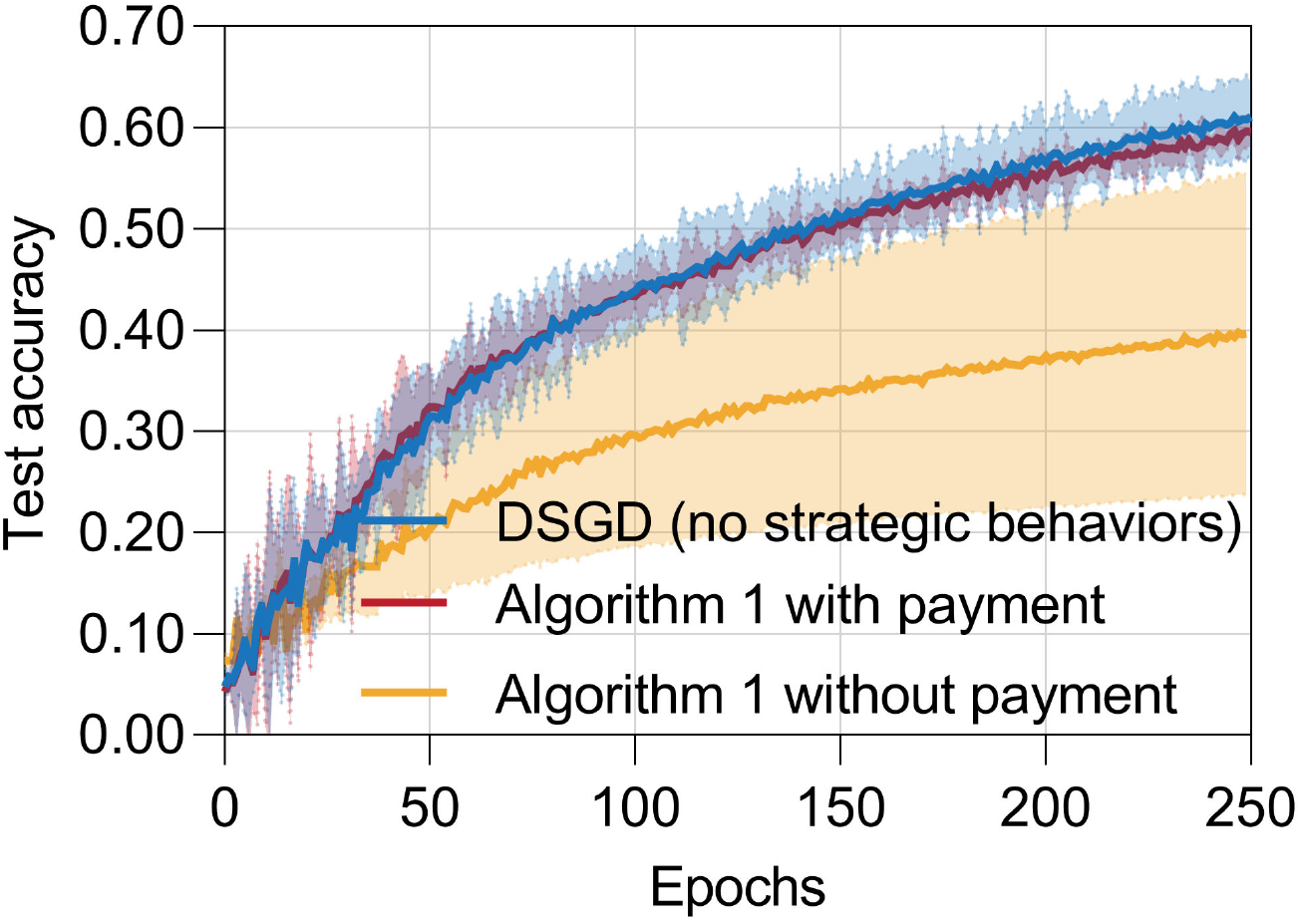}%
\label{fig22}}
\subfloat[Shakespeare dataset]{\includegraphics[width=1.73in]{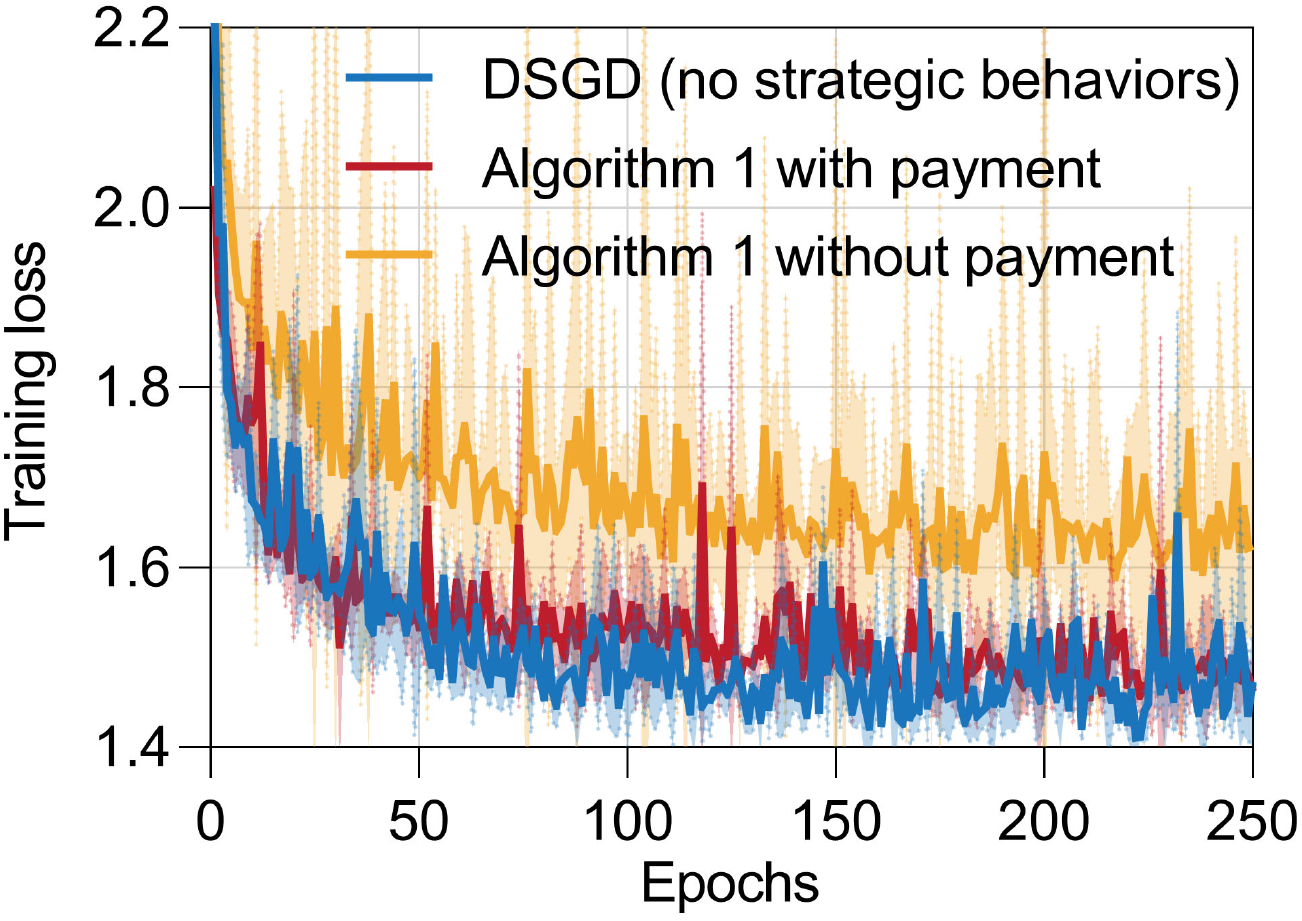}
\label{fig23}}
\subfloat[Shakespeare dataset]{\includegraphics[width=1.73in]{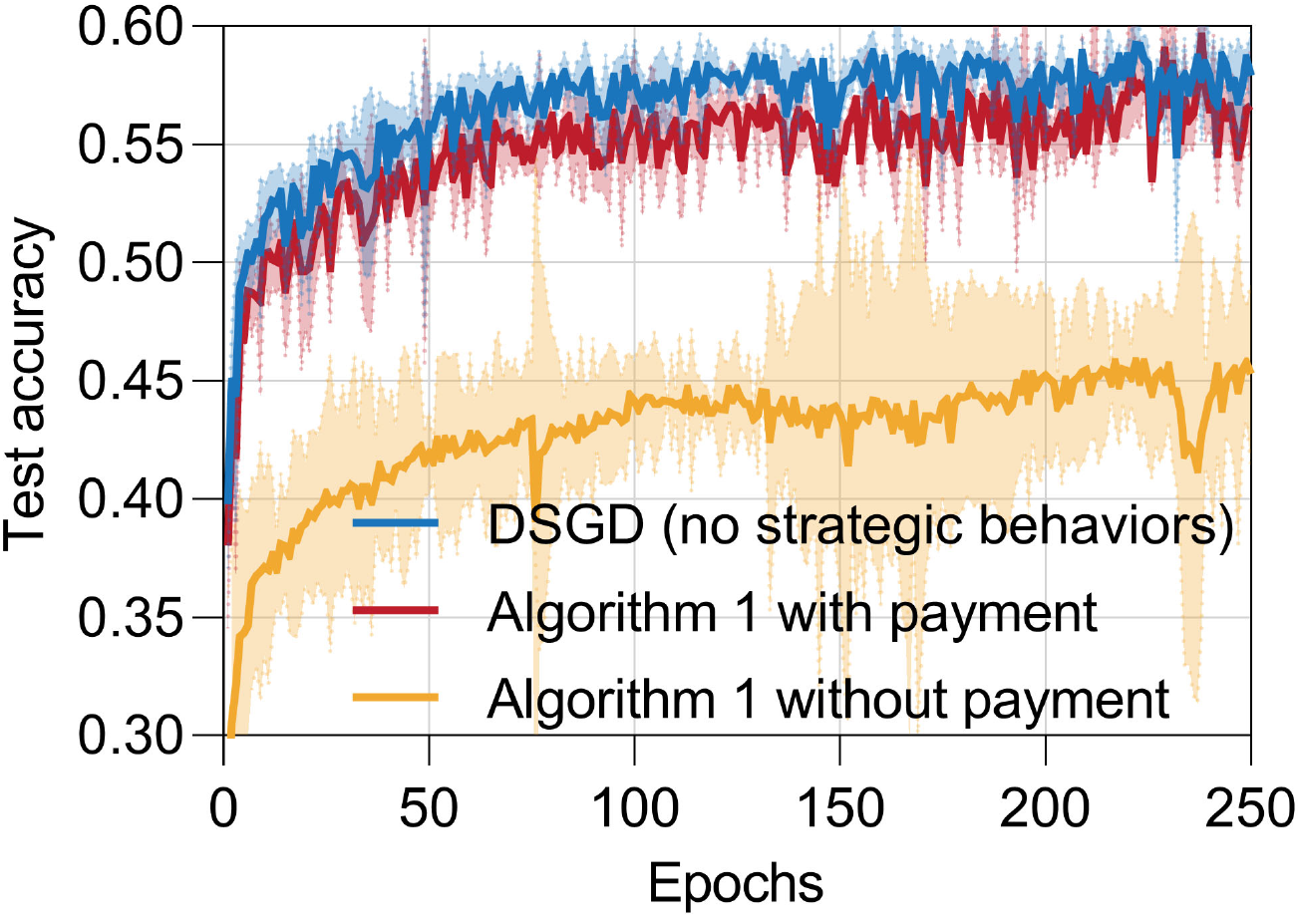}
\label{fig24}}
\caption{Comparison of training losses and test accuracies over epochs. The $95\%$ confidence intervals were computed from three independent runs with random seeds $42$, $126$, and $1010$.}
\label{Fig2}
\end{figure*}

\section{Experiment Evaluation}\label{SectionV}
We evaluate the effectiveness of our truthful mechanism using three representative distributed machine learning tasks: image classification on the FeMNIST dataset and next-character prediction on the Shakespeare dataset. The used datasets are from LEAF~\cite{Leaf}. In addition, each agent’s local dataset was split into $90\%$ for training and $10\%$ for testing. We considered five agents connected in a circle, where each agent communicates only with its two immediate neighbors. For the coupling matrix $W$, we set $w_{ij}=0.3$ if agents $i$ and $j$ are neighbors, and $w_{ij}=0$ otherwise. For each experiment, we distributed the data across agents using a Dirichlet distribution to ensure heterogeneity. In each experiment, we randomly divided the agents into two groups, A and B, containing two and three agents, respectively. Each agent $i$ in group A participates in Algorithm~\ref{algorithm} using manipulated gradients $m_{i,t}^{A}=a_{i,t}^{A}g_{i}(\theta_{i,t})+b_{i,t}^{A}\xi_{i,t}$, where each element of $\xi_{i,t}$ is drawn from a Laplace distribution with zero mean and unit variance. Agents in group B act truthfully. It is clear that when $a_{i,t}^{A}=1$ and $b_{i,t}^{A}=0$ hold for all $i\in[N]$ and $t\geq 0$, all agents act truthfully. For each experiment, we first compared the average net utility of agents in group A under varying scaling factors $a_{A}$ for different payment coefficients $C$. {Then, we conducted a similar comparison under different noise factors $b_{A}$ and payment coefficients $C$. Finally, we evaluated the training losses and test accuracies in three cases:} 1) Algorithm~\ref{algorithm} without manipulation (called DSGD (no strategic behaviors)), 2) Algorithm~\ref{algorithm} with manipulation under Mechanism~\ref{mechanism} (called Algorithm~\ref{algorithm} with payment), and 3) Algorithm~\ref{algorithm} with manipulation without Mechanism~\ref{mechanism} (called Algorithm~\ref{algorithm} without payment). In this comparison, each agent in group A strategically selects $(a_{i,t}^{A}, b_{i,t}^{A})$ to empirically maximize its net utility under a preset $C_{t}=\frac{10^{-6}\kappa_{t}^2}{\delta^2 (t+1)^{-2v}}$, where $\kappa_t$ and $\delta$ are specified for each experiment below.

\noindent \textbf{Image classification using the FeMNIST dataset.}
We conducted experiments using the FeMNIST dataset, which is a variant of EMNIST and contains $817,851$ grayscale handwritten characters of size $28\times 28$ across $62$ classes (digits $0$–$9$, uppercase letters A–Z, lowercase letters a–z). In this experiment, we trained a two-layer convolutional neural network following the LEAF architecture~\cite{Leaf}. The model consists of two convolutional layers with $32$ and $64$ filters (both using $5\times5$ kernels), each followed by ReLU activation and $2\times2$ max pooling. The extracted features are then flattened and fed into a fully connected layer with $2048$ hidden units, followed by a $62$-class output layer. 
We partitioned the training data among agents according to a Dirichlet distribution with parameter $0.5$. We used a batch size of $32$. We set the learning rate (stepsize) as $\lambda_{t}=0.1(t+1)^{-0.55}$ and 
the parameters in payment coefficient $C_{t}$ as
$\kappa_{t}=(t+1)^{-0.51}$ and $\delta=10^{-4}$. These parameter choices satisfy the conditions of Lemma~\ref{ML1}, Theorem~\ref{MT1}, and Theorem~\ref{MT2}.

\noindent \textbf{Next-character prediction on the Shakespeare dataset.} We trained a long short-term memory (LSTM) network~\cite{LSTM} using the Shakespeare dataset, which consists of lines from plays written by William Shakespeare and is formulated as a next-character prediction task over a vocabulary of $80$ characters. In this experiment, we used $3,982,028$ sequences from $1,080$ users for training. Each input character is first mapped to an $8$-dimensional embedding vector. The embedded sequence is then processed by a two-layer LSTM, with each layer comprising 256 hidden units and a dropout rate of 0.5 applied between layers. The output from the final LSTM layer at the last time step is passed through a fully connected layer to produce logits over the $80$-character vocabulary. This model setup enables the network to capture the sequential dependencies in Shakespearean text for effective character-level prediction.

In Fig.~\ref{Fig1}, the top orange lines ($C=0$) show that without our payment mechanism, increasing either $a_{A}$ or $b_{A}$ raises the average net utility of group-A agents. However, introducing payments ($C>0$) effectively reduces the gains from such strategic behaviors by agents in group A. Moreover, a large $C$ makes truthful participation the optimal action for group-A agents. Fig.~\ref{Fig2} shows that the strategic behavior of agents increases the training loss and decreases the test accuracy in conventional distributed learning algorithms, whereas our payment mechanism mitigates this degradation. This demonstrates the effectiveness of our approach in preserving the learning accuracy of Algorithm~\ref{algorithm} despite strategic manipulation.
\section{Conclusion}\label{SectionVI}
We propose the first fully distributed incentive mechanism for distributed stochastic gradient descent with strategic agents, without relying on any centralized server or aggregator. This represents a significant advance since all existing truthfulness approaches require the  assistance of a centralized server in computation or execution. Our payment mechanism ensures that the cumulative gain from strategic manipulation remains finite, even over an infinite time horizon---a property unattainable in most existing truthfulness results. Moreover, unlike most existing truthfulness results, our payment mechanism is budget-balanced and guarantees accurate convergence of distributed stochastic gradient descent under strategic manipulation. The results apply to both general convex and strongly convex objective functions, extending beyond existing work that focuses solely on the strongly convex case.
Experimental results on two distributed machine learning applications confirm the effectiveness of our approach.

\section*{Appendix}
Throughout the Appendix, we denote $\theta_{i,t+1}$ as the model parameter of agent $i$ generated by Algorithm~\ref{algorithm} using the gradient $g_{i}(\theta_{i,t})$ at iteration $t$, and $\theta_{i,t+1}^{\prime}$ as the model parameter of agent $i$ generated by Algorithm~\ref{algorithm} using a (manipulated) gradient $m_{i,t}=\alpha_{i,t}(g_{i}(\theta_{i,t}))$. We let $P_{i,t}$ denote the payment of agent $i$ at iteration $t$ when all agents act truthfully, and $P_{i,t}^{\prime}$ denote the payment of agent $i$ when agent $i$ deviates while all other agents remain truthful.
\subsection{Proof of Lemma~\ref{examplelemma}}
1) We first consider the case where all agents are truthful. 

Algorithm~\ref{mechanism} with $m_{i,t}=g_{i}(\theta_{i,t})$ implies
\begin{equation}
\bar{\theta}_{t+1}-\mu=(1-2\lambda_t)(\bar{\theta}_{t}-\mu)+2\lambda_{t}(\bar{\zeta}_{t}-\mu).\nonumber
\end{equation}
By using relations $\mathbb{E}[\bar{\zeta}_{t}]=\mu$ and $\mathbb{E}[\|\bar{\zeta}_{t}-\mu\|^2]\leq \frac{\sigma^2}{N}$, and the fact that $\bar{\zeta}_{t}-\mu$ is independent of $\bar{\theta}_{t}-\mu$, we have
\begin{equation}
\textstyle \mathbb{E}[\|\bar{\theta}_{t+1}-\mu\|^2]=(1-2\lambda_t)^2 \mathbb{E}[\|\bar{\theta}_{t}-\mu\|^2]+\frac{4\lambda^2_{t}\sigma^2}{N}.\label{dayu}
\end{equation}
By applying Lemma 5-(i) in the arxiv version of~\cite{JDPziji} to~\eqref{dayu}, we obtain
\begin{equation}
\textstyle 0\leq \mathbb{E}[\|\bar{\theta}_{t}-\mu\|^2]\leq c_{1}\lambda_{t},\label{addeq1}
\end{equation}
with $c_{1}=\max\left\{\frac{\mathbb{E}[\|\bar{\theta}_{0}-\mu\|^2]}{\lambda_{0}},\frac{4\lambda_{0}\sigma^2}{N(2\lambda_{0}-v)}\right\}$.

Using the relation $\mathbb{E}[\|a+b\|^2]\leq (\sqrt{\mathbb{E}[\|a\|^2]}+\sqrt{\mathbb{E}[\|b\|^2]})^2$ for any random variables $a$ and $b$, we have $\mathbb{E}[\|\bar{\theta}_{t}-\mu_{i}\|^2]\textstyle\leq (\sqrt{\mathbb{E}[\|\bar{\theta}_{t}-\mu\|^2]}+\|\mu-\mu_{i}\|)^2$, which implies 
{\color{blue}\begin{equation}
\textstyle \sqrt{\mathbb{E}[\|\bar{\theta}_{t}-\mu_{i}\|^2]}-\|\mu-\mu_{i}\|\leq \sqrt{\mathbb{E}[\|\bar{\theta}_{t}-\mu\|^2]}.\label{l21}
\end{equation}
Moreover, since $\|\mu-\mu_{i}\|\leq\sqrt{\mathbb{E}[\|\bar{\theta}_{t}-\mu_{i}\|^2]}+\sqrt{\mathbb{E}[\|\bar{\theta}_{t}-\mu\|^2}$ holds, we have $\|\mu-\mu_{i}\|-\sqrt{\mathbb{E}[\|\bar{\theta}_{t}-\mu_{i}\|^2]}\leq\sqrt{\mathbb{E}[\|\bar{\theta}_{t}-\mu\|^2}$, which, combined with~\eqref{addeq1} and~\eqref{l21}, leads to
\begin{equation}
\textstyle \left|\sqrt{\mathbb{E}[\|\bar{\theta}_{t}-\mu_{i}\|^2]}-\|\mu-\mu_{i}\|\right|\leq \sqrt{c_{1}\lambda_{t}}.\label{addeq2}
\end{equation}

By using \eqref{addeq2} and the definition $f_{i}(\bar{\theta}_{T})=\mathbb{E}[\|\bar{\theta}_{T}-\mu_{i}\|^2]$, we have that for any $T>0$, the mean squared error of agent $i$ satisfies
\begin{equation}
(\|\mu-\mu_{i}\|- \sqrt{c_{1}\lambda_{T}})^2\leq f_{i}(\bar{\theta}_{T})\leq (\|\mu-\mu_{i}\|+ \sqrt{c_{1}\lambda_{T}})^2,\label{addeq3}
\end{equation}
which implies that the mean squared error of agent $i$ satisfies $\lim_{T\rightarrow\infty}f_{i}(\bar{\theta}_{T})=\|\mu-\mu_{i}\|^2$ and the global mean squared error satisfies $\lim_{T\rightarrow\infty}F(\bar{\theta}_{T})=\frac{1}{N}\sum_{j=1}^{N}\|\mu-\mu_{j}\|^2$.}

2) Next, we consider the case where agent $i$ amplifies its gradient estimates as $a_{i}g_{i}(\theta_{i,t}^{\prime})\!=\!2a_{i}(\theta_{i,t}^{\prime}-\zeta_{i,t})$, while all other agents behave truthfully. According to Algorithm~\ref{algorithm}, we have
\begin{equation}
\begin{aligned}
\bar{\theta}_{t+1}^{\prime}&\textstyle=\bar{\theta}_{t}^{\prime}-\frac{2(a_{i}+N-1)}{N}\lambda_{t}\left(\bar{\theta}_{t}^{\prime}-\left(\frac{a_{i}\zeta_{i,t}+\sum_{j\neq i}^{N}\zeta_{j,t}}{a_{i}+N-1}\right)\right)\\
&\textstyle\quad-\frac{2(a_{i}-1)}{N}\lambda_{t}(\theta_{i,t}^{\prime}-\bar{\theta}_{t}^{\prime}).\label{l29}
\end{aligned}
\end{equation}

For the sake of notational simplicity, we define $\hat{a}_{i}\!\triangleq\!\frac{a_{i}+N-1}{N}$ and $\hat{\mu}_{i}\triangleq\frac{a_{i}\mu_{i}+\sum_{j\neq i}^{N}\mu_{j}}{a_{i}+N-1}$. Then, by using~\eqref{l29}, we obtain
\begin{flalign}
&\textstyle\mathbb{E}[\|\bar{\theta}_{t+1}^{\prime}-\hat{\mu}_{i}\|^2]=\left(1+\hat{a}_{i}\lambda_{t}\right)\mathbb{E}[\|\left(1-2\hat{a}_{i}\lambda_{t}\right)\left(\bar{\theta}_{t}^{\prime}-\hat{\mu}_{i}\right)\|^2]\nonumber\\
&\textstyle+\left(1\!+\!\frac{1}{\hat{a}_{i}\lambda_{t}}\right)\mathbb{E}\left[\left\|\frac{2(a_{i}-1)\lambda_{t}}{N}(\theta_{i,t}^{\prime}-\bar{\theta}_{t}^{\prime})\right\|^2\right]\!+\!\frac{4\hat{a}_{i}^2\lambda_{t}^2(a_{i}^2+N-1)\sigma^2}{(a_{i}+N-1)^2},\nonumber\\
&\textstyle \leq \left(1-\hat{a}_{i}\lambda_{t}\right)\mathbb{E}[\|\bar{\theta}_{t}^{\prime}-\hat{\mu}_{i}\|^2]+\frac{4\hat{a}_{i}^2(a_{i}^2+N-1)\sigma^2}{(a_{i}+N-1)^2}\lambda_{t}^2\nonumber\\
&\textstyle+\left(\lambda_{0}+\frac{1}{\hat{a}_{i}}\right)\frac{4(a_{i}-1)^2}{N^2}\lambda_{t}\mathbb{E}[\|\theta_{i,t}^{\prime}-\bar{\theta}_{t}^{\prime}\|^2].\label{l211}
\end{flalign}

Furthermore, Algorithm~\ref{mechanism} with $m_{i,t}=2a_{i}(\theta_{i,t}^{\prime}-\zeta_{i,t})$ and the dynamics in~\eqref{l29} imply
\begin{equation}
\begin{aligned}
&\textstyle\boldsymbol{\theta}_{t+1}^{\prime}-\boldsymbol{1}_{N}\otimes \bar{\theta}_{t+1}^{\prime}=(W\otimes I_n)\left(\boldsymbol{\theta}_{t}^{\prime}-\boldsymbol{1}_{N}\otimes \bar{\theta}_{t}^{\prime}\right)\\
&\textstyle\quad-2\lambda_t\left(\boldsymbol{\theta}_{t}^{\prime}-\boldsymbol{1}_{N}\otimes \bar{\theta}_{t}^{\prime}-\left(\boldsymbol{\zeta}_{t}-\boldsymbol{1}_{N}\otimes \bar{\zeta}_{t}\right)\right)\\
&\textstyle\quad+
2\lambda_{t}M_{i}(\theta_{i,t}^{\prime}-\zeta_{i,t}),\label{l213}
\end{aligned}
\end{equation}
where matrix $M_{i}$ is given by $M_{i}=(1-a_{i})(\boldsymbol{e}_{i}-\frac{1}{N}\boldsymbol{1}_{N})\otimes I_{n}$ with $\boldsymbol{e}_{i}\in\mathbb{R}^{N}$ denoting the $i$th standard basis vector. 

The last term on the right hand side of~\eqref{l213} satisfies
\begin{equation}
\begin{aligned}
&\textstyle\mathbb{E}[\|\theta_{i,t}^{\prime}-\zeta_{i,t}\|^2]\leq 3\mathbb{E}[\|\theta_{i,t}^{\prime}-\bar{\theta}_{t}^{\prime}\|^2]\\
&\textstyle\quad+3\mathbb{E}[\|\bar{\theta}_{t}^{\prime}-\hat{\mu}_{i}\|^2]+3\mathbb{E}[\|\hat{\mu}_{i}-\zeta_{i,t}\|^2].\label{l215}
\end{aligned}
\end{equation}

By taking the squared norm and expectations on both sides of~\eqref{l213} and using~\eqref{l215}, we obtain
\begin{flalign}
&\mathbb{E}[\|\boldsymbol{\theta}_{t+1}^{\prime}-\boldsymbol{1}_{N}\otimes \bar{\theta}_{t+1}^{\prime}\|^2]\nonumber\\
&\textstyle\leq \left(\frac{(\rho-2\lambda_{t})^2}{\rho}+\frac{24(2-\rho)\|M_{i}\|^2}{1-\rho}\lambda_t^2\right)\mathbb{E}[\|\boldsymbol{\theta}_{t}^{\prime}-\boldsymbol{1}_{N}\otimes \bar{\theta}_{t}^{\prime}\|^2]\nonumber\\
&\textstyle\quad+\frac{24(2-\rho)\|M_{i}\|^2}{1-\rho}\lambda_t^2\left(\mathbb{E}[\|\bar{\theta}_{t}^{\prime}-\hat{\mu}_{i}\|^2]+\mathbb{E}[\|\hat{\mu}_{i}-\zeta_{i,t}\|^2]\right)\nonumber\\
&\textstyle\quad+\frac{8(2-\rho)}{1-\rho}\lambda_t^2\mathbb{E}[\|\boldsymbol{\zeta}_{t}-\boldsymbol{1}_{N}\otimes \bar{\zeta}_{t}\|^2],\label{l216}
\end{flalign}
where in the derivation we have used the relations $1+(1-\rho)\leq\frac{1}{\rho}$ and $\mathbb{E}[\|\theta_{i,t}^{\prime}-\bar{\theta}_{t}^{\prime}\|^2]\leq \mathbb{E}[\|\boldsymbol{\theta}_{t}^{\prime}-\boldsymbol{1}_{N}\otimes \bar{\theta}_{t}^{\prime}\|^2]$.

Summing both sides of~\eqref{l211} and~\eqref{l216}, we arrive at
\begin{equation}
\begin{aligned}
&\textstyle\mathbb{E}[\|\bar{\theta}_{t+1}^{\prime}-\hat{\mu}_{i}\|^2]+\mathbb{E}[\|\boldsymbol{\theta}_{t+1}^{\prime}-\boldsymbol{1}_{N}\otimes \bar{\theta}_{t+1}^{\prime}\|^2]\\
&\textstyle\leq \left(1-\hat{a}_{i}\lambda_{t}+\frac{24(2-\rho)\|M_{i}\|^2}{1-\rho}\lambda_t^2\right)\mathbb{E}[\|\bar{\theta}_{t}^{\prime}-\hat{\mu}_{i}\|^2]\\
&\textstyle\quad+\left(\frac{(\rho-2\lambda_{t})^2}{\rho}+\frac{4(\lambda_{0}\hat{a}_{i}+1)(a_{i}-1)^2}{\hat{a}_{i}N^2}\lambda_{t}+\frac{24(2-\rho)\|M_{i}\|^2}{1-\rho}\lambda_t^2\right)\\
&\textstyle\quad\times \mathbb{E}[\|\boldsymbol{\theta}_{t}^{\prime}-\boldsymbol{1}_{N}\otimes \bar{\theta}_{t}^{\prime}\|^2]+c_{i,1}\lambda_t^2,\nonumber
\end{aligned}
\end{equation}
with $c_{i,1}=\frac{4\hat{a}_{i}^2(a_{i}^2+N-1)\sigma^2}{(a_{i}+N-1)^2}+\frac{24(2-\rho)\|M_{i}\|^2\sum_{j=1}^{N}\|\mu_{j}-\mu_{i}\|^2}{1-\rho}+\frac{24(2-\rho)\|M_{i}\|^2\sigma^2}{(1-\rho)(a_{i}+N-1)}+\frac{32(2-\rho)d_{\zeta}^2}{1-\rho}$ and $d_{\zeta}=\max_{i\in[N],t\in\mathbb{N}}\{\|\zeta_{i,t}\|\}$.

Since $\lambda_{t}$ is a decaying sequence, we have
\begin{equation}
\begin{aligned}
&\textstyle\mathbb{E}[\|\bar{\theta}_{t+1}^{\prime}-\hat{\mu}_{i}\|^2]+\mathbb{E}[\|\boldsymbol{\theta}_{t+1}^{\prime}-\boldsymbol{1}_{N}\otimes \bar{\theta}_{t+1}^{\prime}\|^2]\\
&\textstyle\leq \left(1-\frac{\hat{a}_{i}\lambda_{t}}{2}\right)\left(\mathbb{E}[\|\bar{\theta}_{t}^{\prime}-\hat{\mu}_{i}\|^2]+\mathbb{E}[\|\boldsymbol{\theta}_{t}^{\prime}-\boldsymbol{1}_{N}\otimes \bar{\theta}_{t}^{\prime}\|^2]\right)\!+\!c_{i,1}\lambda_{t}^2.\nonumber
\end{aligned}
\end{equation}
Applying Lemma 5-(i) in the arxiv version of~\cite{JDPziji} to the preceding inequality, we arrive at
\begin{equation}
\mathbb{E}[\|\bar{\theta}_{t}^{\prime}-\hat{\mu}_{i}\|^2]+\mathbb{E}[\|\boldsymbol{\theta}_{t}^{\prime}-\boldsymbol{1}_{N}\otimes \bar{\theta}_{t}^{\prime}\|^2]\leq c_{i,2}\lambda_{t},\label{l218}
\end{equation}
with $c_{i,2}=\max\{\mathbb{E}[\|\bar{\theta}_{0}-\hat{\mu}\|^2]+\mathbb{E}[\|\boldsymbol{\theta}_{0}-\boldsymbol{1}_{N}\otimes \bar{\theta}_{0}\|^2],\frac{2c_{i,1}\lambda_{0}^2}{\hat{a}_{i}\lambda_{0}-2v}\}$.

By using $\mathbb{E}[\|a+b\|^2]\leq (\sqrt{\mathbb{E}[\|a\|^2]}+\sqrt{\mathbb{E}[\|b\|^2]})^2$ for any random variables $a$ and $b$ and the definition $f_{i}(\bar{\theta}_{T})=\mathbb{E}[\|\bar{\theta}_{T}-\mu_{i}\|^2]$, for any $T>0$, we have
\begin{equation}
f_{i}(\bar{\theta}_{T}^{\prime})\textstyle\leq \left(\sqrt{c_{i,2}\lambda_{T}}+\left(\frac{N}{a_{i}+N-1}\right)\|\mu-\mu_{i}\|\right)^2.\label{addeq4}
\end{equation}

{\color{blue}Using the condition $T+1>\left(\frac{\lambda_{0}(a_{i}+N-1)^2(\sqrt{c_{1}}+\sqrt{c_{i,2}})^2}{(a_{i}-1)^2\|\mu-\mu_{i}\|^2}\right)^{1/v}$ from the lemma statement, we obtain the following inequality: $\sqrt{c_{i,2}\lambda_{T}}< \frac{a_{i}-1}{a_{i}+N-1}\|\mu-\mu_{i}\|-\sqrt{c_{1}\lambda_{T}}$, which, after adding $\frac{N}{a_{i}+N-1}\|\mu-\mu_{i}\|$ to both sides, yields
\begin{equation}
\textstyle \frac{N}{a_{i}+N-1}\|\mu-\mu_{i}\|+\sqrt{c_{i,2}\lambda_{T}}=\|\mu-\mu_{i}\|-\sqrt{c_{1}\lambda_{T}}.\label{addeq5}
\end{equation}

By combining~\eqref{addeq4} with~\eqref{addeq5} and using~\eqref{addeq3}, we arrive at~\eqref{eq:3a} in Lemma~\ref{examplelemma}.} 

Furthermore, using $\lim_{T\rightarrow\infty}\lambda_{T}=0$ and~\eqref{addeq3}, we obtain
\begin{equation}
\begin{aligned}
\lim_{T\rightarrow\infty} f_{i}(\bar{\theta}_{T}^{\prime})
&\textstyle\leq \left(\frac{N}{a_{i}+N-1}\right)^2\|\mu-\mu_{i}\|^2\\
&\textstyle <\|\mu-\mu_{i}\|^2=\lim_{T\rightarrow\infty} f_{i}(\bar{\theta}_{T}).\nonumber
\end{aligned}
\end{equation}

{\color{blue}By using the identity $\|a+b\|^2=\|a\|^2+\|b\|^2+2\langle a,b\rangle$ for any $a,b\in\mathbb{R}^{n}$ and the relationship $\sum_{j=1}^{N}(\mu-\mu_{j})=0$, we have $F(\bar{\theta}_{T})=\frac{1}{N}\sum_{j=1}^{N}\|\mu-\mu_{j}\|^2+\mathbb{E}[\|\bar{\theta}_{T}-\mu\|^2]$ and $F(\bar{\theta}_{T}^{\prime})=\frac{1}{N}\sum_{j=1}^{N}\|\mu-\mu_{j}\|^2+\mathbb{E}[\|\bar{\theta}_{T}^{\prime}-\mu\|^2]$, which imply
\begin{equation}
F(\bar{\theta}_{T}^{\prime})-F(\bar{\theta}_{T})=\mathbb{E}[\|\bar{\theta}_{T}^{\prime}-\mu\|^2]-\mathbb{E}[\|\bar{\theta}_{T}-\mu\|^2].\label{addeq6}
\end{equation}

We characterize the first term on the right hand side of~\eqref{addeq6}. 

By using $\mathbb{E}[\|a+b\|^2]\leq (\sqrt{\mathbb{E}[\|a\|^2]}+\sqrt{\mathbb{E}[\|b\|^2]})^2$ for any random variables $a$ and $b$, we have
\begin{equation}
\begin{aligned}
\textstyle\sqrt{\mathbb{E}[\|\bar{\theta}_{T}^{\prime}-\mu\|^2]}&\textstyle\leq \sqrt{\mathbb{E}[\|\bar{\theta}_{T}^{\prime}-\hat{\mu}_{i}\|^2]}+\|\hat{\mu}_{i}-\mu\|\\
&\textstyle\leq \sqrt{c_{i,2}\lambda_{T}}+\frac{a_{i}-1}{a_{i}+N-1}\|\mu-\mu_{i}\|,\label{addeq7}
\end{aligned}
\end{equation}
where we have used~\eqref{l218} and the relationship $\|\hat{\mu}_{i}-\mu\|=\frac{a_{i}-1}{a_{i}+N-1}\|\mu-\mu_{i}\|$ in the second inequality.

Using again $\mathbb{E}[\|a+b\|^2]\leq (\sqrt{\mathbb{E}[\|a\|^2]}+\sqrt{\mathbb{E}[\|b\|^2]})^2$ for any random variables $a$ and $b$, we have $\|\hat{\mu}_{i}-\mu\|\leq \sqrt{\mathbb{E}[\|\bar{\theta}_{T}^{\prime}-\mu\|^2]}+\sqrt{\mathbb{E}[\|\bar{\theta}_{T}^{\prime}-\hat{\mu}_{i}\|^2]}$, which implies
\begin{equation}
\begin{aligned}
\textstyle\sqrt{\mathbb{E}[\|\bar{\theta}_{T}^{\prime}-\mu\|^2]}&\textstyle\geq \|\hat{\mu}_{i}-\mu\|-\sqrt{\mathbb{E}[\|\bar{\theta}_{T}^{\prime}-\hat{\mu}_{i}\|^2]}\\
&\textstyle\geq \frac{a_{i}-1}{a_{i}+N-1}\|\mu-\mu_{i}\|-\sqrt{c_{i,2}\lambda_{T}}.\label{addeq8}
\end{aligned}
\end{equation}
Combining~\eqref{addeq7} and~\eqref{addeq8}, we obtain the following bounds on the first term on the right-hand side of~\eqref{addeq6}:
\begin{equation}
\begin{aligned}
&\textstyle\frac{a_{i}-1}{a_{i}+N-1}\|\mu-\mu_{i}\|-\sqrt{c_{i,2}\lambda_{T}}\leq \sqrt{\mathbb{E}[\|\bar{\theta}_{T}^{\prime}-\mu\|^2]}\\
&\textstyle\leq \sqrt{c_{i,2}\lambda_{T}}+\frac{a_{i}-1}{a_{i}+N-1}\|\mu-\mu_{i}\|.\label{addeq9}
\end{aligned}
\end{equation}

Substituting the relation $0<\mathbb{E}[\|\bar{\theta}_{T}-\mu\|^2]\leq c_{1}\lambda_{T}$ from~\eqref{addeq1} and~\eqref{addeq9} into~\eqref{addeq6}, we arrive at~\eqref{eq:3b} in Lemma~\ref{examplelemma}.}

Furthermore, by using $\lim_{T\rightarrow\infty}\lambda_{T}=0$, we obtain
\begin{equation}
\begin{aligned}
&\textstyle\lim_{T\rightarrow\infty} F(\bar{\theta}_{T}^{\prime})\\
&\textstyle\leq \lim_{T\rightarrow\infty} F(\bar{\theta}_{T})+\left(\frac{a_{i}-1}{a_{i}+N-1}\right)^2\|\mu-\mu_{i}\|^2.\nonumber
\end{aligned}
\end{equation}

\subsection{Auxiliary lemmas}\label{AppedixA}
In this subsection, we introduce some auxiliary lemmas that will be used in our subsequent convergence analysis.
\begin{lemma}\label{lemma5}
Denoting $\eta_{t}$ as a nonnegative sequence, if there exist sequences $\beta_{1,t}=\frac{\beta_{1}}{(t+1)^{r_{1}}}$ and $\beta_{2,t}=\frac{\beta_{2}}{(t+1)^{r_{2}}}$ with some $1>r_{1}>0$, $r_{2}>r_{1}$, $\beta_{1}>0$, and $\beta_{2}>0$ such that $\eta_{t+1}\geq (1-\beta_{1,t})\eta_{t}+\beta_{2,t}$ holds, then we always have $\eta_{t}\geq \frac{\beta_{1}}{\beta_{2}}\min\{\eta_{0},\frac{\beta_{2}}{\beta_{1}}\}\frac{\beta_{2,t}}{\beta_{1,t}}.$	
\end{lemma}
\begin{proof}
We prove Lemma~\ref{lemma5} using mathematical induction. 

We define $c_{0}=\min\{\eta_{0},\frac{\beta_{2}}{\beta_{1}}\}$, which implies $\eta_{0}\geq c_{0}$ at initialization. Assuming $\eta_{t} \geq \frac{c_{0}}{(t+1)^{r_{2}-r_{1}}}$ at the $t$th iteration, we proceed to prove $\eta_{t+1} \geq \frac{c_{0}}{(t+2)^{r_{2}-r_{1}}}$ at the $(t+1)$th iteration. 

By using the relationship $\eta_{t+1}\geq (1-\beta_{1,t})\eta_{t}+\beta_{2,t}$, we have
\begin{equation}
\begin{aligned}
&\textstyle\eta_{t+1}\geq \frac{c_{0}}{(t+1)^{r_{2}-r_{1}}}-\frac{c_{0}\beta_{1}}{(t+1)^{r_{2}}}+\frac{\beta_{2}}{(t+1)^{r_{2}}}\\
&\textstyle\geq \frac{c_{0}}{(t+2)^{r_{2}-r_{1}}}+\left(\frac{c_{0}}{(t+1)^{r_{2}-r_{1}}}-\frac{c_{0}}{(t+2)^{r_{2}-r_{1}}}-\frac{c_{0}\beta_{1}-\beta_{2}}{(t+1)^{r_{2}}}\right).\label{AL1}
\end{aligned}
\end{equation}

Using the mean value theorem, we have $\frac{c_{0}}{(t+1)^{r_{2}-r_{1}}}-\frac{c_{0}}{(t+2)^{r_{2}-r_{1}}}=\frac{c_{0}(r_{2}-r_{1})}{\varsigma^{r_{2}-r_{1}+1}}> \frac{c_{0}(r_{2}-r_{1})}{(t+2)^{r_{2}-r_{1}+1}}$ with some $\varsigma\in(t+1,t+2)$, which, combined with $c_{0}\beta_{1}-\beta_{2}\leq 0$, leads to $\frac{c_{0}(r_{2}-r_{1})}{(t+2)^{r_{2}-r_{1}+1}}\geq  \frac{c_{0}\beta_{1}-\beta_{2}}{(t+1)^{r_{2}}}$. Hence, the inequality $\frac{c_{0}}{(t+1)^{r_{2}-r_{1}}}-\frac{c_{0}}{(t+2)^{r_{2}-r_{1}}}\geq \frac{c_{0}\beta_{1}-\beta_{2}}{(t+1)^{r_{2}}}$ holds for any $t\geq 0$. Further using~\eqref{AL1}, we arrive at $\eta_{t+1}\geq \frac{c_{0}}{(t+2)^{r_{2}-r_{1}}}$, which proves Lemma~\ref{lemma5}.
\end{proof}
\begin{lemma}\label{2Clemma3}
For a nonnegative sequence $\eta_{t}=\frac{\eta_{0}}{(t+1)^{r}}$ with $\eta_{0}>0$ and $r\in(0,2)$, the inequality $\sum_{k=0}^{t-1} \eta_{k}^2\gamma^{2(t-1-k)}\leq c\eta_{t}^{2}$ always holds for any $t\geq 1$ and $\gamma\in(0,1)$, where the constant $c$ is given by $c=\frac{4^{2(r+1)}}{(1-\gamma)(\ln(\sqrt{\gamma})e)^{4}}$.
\end{lemma}
\begin{proof}
By defining $c_{0}=4^{-1}(\ln(\sqrt{\gamma})e)^{2}$ and using Lemma 7 in~\cite{zijiGT}, we have
\begin{equation} \textstyle\eta_{k}\sqrt{\gamma}^{t-1-k}\leq\eta_{k}\frac{1}{c_{0}((t-1)-k)^{2}}= \frac{\eta_{0}}{c_{0}(k+1)^{r}((t-1)-k)^{2}}.\label{B3L1}
\end{equation}
For some real numbers $a,b,c,d>0$ satisfying $\frac{c}{d}<\frac{d}{b}$, the inequality $\frac{d}{b}<\frac{c+d}{a+b}<\frac{c}{a}$ always holds. Therefore, for any $t>0$ and $k\in[0,t-1)$, by setting $a=k$, $b=(t-1)-k$, $c=k$, and $d=1$, we have
\begin{equation}
\textstyle\frac{1}{((t-1)-k)^2}<\left(\frac{k+1}{t-1}\right)^2.\nonumber
\end{equation}

Combining~\eqref{B3L1} and the relation $\frac{1}{((t-1)-k)^2}<(\frac{k+1}{t-1})^{r}$ for any $k\in[0,t-1)$ and $r\in(0,2)$, we obtain
\begin{equation} \textstyle\eta_{k}\sqrt{\gamma}^{t-1-k}<\frac{\eta_{0}}{c_{0}(k+1)^{r}}\left(\frac{k+1}{t-1}\right)^{r}\leq\frac{4^{r}\eta_{0}}{c_{0}(t+1)^{r}}=\frac{4^{r}\eta_{t}}{c_{0}},\label{B3L2}
\end{equation}
where we have used the relationship $\frac{1}{t-1} \leq \frac{4}{t+1}$ for any $t>1$ in the second inequality.~\eqref{B3L2} further implies $\eta_{k}^2\gamma^{t-1-k}\leq 4^{2r}c_{0}^{-2}\eta_{t}^{2}$ for any $t>1$ and $k\in[0,t-1)$.

Given a constant $c>1$, the inequality $\eta_{0}^{2}\leq c\eta_{0}^{2}$ holds for $t=1$ and the inequality $\eta_{t-1}^2\leq c\eta_{t-1}^2$ holds for any $t>1$ and $k=t-1$, which, combined with~\eqref{B3L2}, leads to
\begin{equation}
\textstyle\sum_{k=0}^{t-1} \eta_{k}^2\gamma^{2(t-1-k)}\leq \sum_{k=0}^{t-1}\gamma^{t-1-k}\frac{4^{2r}\eta_{t}^{2}}{c_{0}^{2}}\leq c\eta_{t}^{2},\nonumber
\end{equation}
which completes the proof of Lemma~\ref{2Clemma3}.
\end{proof}
\begin{lemma}\label{sensitiveLemma1}
We denote $\theta_{i,T}$ and $\theta_{i,T}^{\prime}$ as the model parameters generated by the standard distributed SGD at iteration $T-1$ from two different initializations $\theta_{i,t}$ and $\theta_{i,t}^{\prime}$, respectively. Under the conditions in Lemma~\ref{ML1}, the following result holds for the standard distributed SGD:
\begin{equation}
\begin{aligned}
&\textstyle\sum_{i=1}^{N}\mathbb{E}[\|\theta_{i,T}-\theta_{i,T}^{\prime}\|^2]\leq 2d_{t\text{\tiny$\to$}T}\left(N\mathbb{E}[\|\bar{\theta}_{t}-\bar{\theta}_{t}^{\prime}\|^2]\right.\\
&\left.\textstyle+\sum_{i=1}^{N}\mathbb{E}[\|\theta_{i,t}-\theta_{i,t}^{\prime}-(\bar{\theta}_{t}-\bar{\theta}_{t}^{\prime})\|^2]\right),\label{SLresult}
\end{aligned}
\end{equation}
where $d_{t\text{\tiny$\to$}T}$ is given by $d_{t\text{\tiny$\to$}T}=e^{\frac{20H^2\lambda_{0}^2}{(1-\rho)(2v-1)}}(t^{1-2v}-T^{1-2v})$.
\end{lemma}

\begin{proof}
By defining $\Xi_{i,t}=\theta_{i,t}-\theta_{i,t}^{\prime}$ and using the update of $\theta_{i,t+1}$ in Algorithm~\ref{algorithm}, we obtain
\begin{equation}
\textstyle\Xi_{i,t+1}=\sum_{j\in\mathcal{N}_{i}\cup\{i\}}w_{ij}\Xi_{j,t}-\lambda_{t}(g_{i}(\theta_{i,t})-g_{i}(\theta_{i,t}^{\prime})),\label{SL2}
\end{equation}
which implies $\bar{\Xi}_{t+1}=\bar{\Xi}_{t}-\lambda_{t}(\bar{g}(\boldsymbol{\theta}_{t})-\bar{g}(\boldsymbol{\theta}_{t}^{\prime}))$.

By using the squared norm expansion, we obtain
\begin{equation}
\begin{aligned}
&\textstyle\mathbb{E}[\|\bar{\Xi}_{t+1}\|^2]=\mathbb{E}[\|\bar{\Xi}_{t}\|^2]+\mathbb{E}[\|\lambda_{t}(\bar{g}(\boldsymbol{\theta}_{t})-\bar{g}(\boldsymbol{\theta}_{t}^{\prime}))\|^2]\\
&\textstyle\quad-2\mathbb{E}\left[\left\langle\bar{\Xi}_{t},\lambda_{t}(\bar{g}(\boldsymbol{\theta}_{t})-\bar{g}(\boldsymbol{\theta}_{t}^{\prime}))\right\rangle\right].\label{SL3}
\end{aligned}
\end{equation}
Assumption~\ref{A1}-(ii) implies that the second term on the right hand side of~\eqref{SL3} satisfies 
\begin{flalign}
&\textstyle\mathbb{E}[\|\lambda_{t}(\bar{g}(\boldsymbol{\theta}_{t})-\bar{g}(\boldsymbol{\theta}_{t}^{\prime}))\|^2]\leq \frac{\lambda_{t}^2}{N}\sum_{i=1}^{N}\mathbb{E}[\|g_{i}(\theta_{i,t})-g_{i}(\theta_{i,t}^{\prime})\|^2]\nonumber\\
&\textstyle\leq 2H^2\lambda_{t}^2\mathbb{E}[\|\bar{\Xi}_{t}\|^2]+\frac{2H^2\lambda_{t}^2}{N}\sum_{i=1}^{N}\mathbb{E}[\|\Xi_{i,t}-\bar{\Xi}_{t}\|^2].\label{SL7}
\end{flalign}

The relation $\mathbb{E}[g_{i}(\theta)]=\nabla f_{i}(\theta)$ in Assumption~\ref{A1}-(ii) implies
\begin{equation}
\begin{aligned}
&2\mathbb{E}[\langle\bar{\Xi}_{t},\lambda_{t}(g_{i}(\theta_{i,t})-g_{i}(\theta_{i,t}^{\prime}))\rangle]\\
&\textstyle=2\lambda_{t}\mathbb{E}[\langle\Xi_{i,t},\nabla f_{i}(\theta_{i,t})-\nabla f_{i}(\theta_{i,t}^{\prime})\rangle]\\
&\quad-2\lambda_{t}\mathbb{E}[\langle\Xi_{i,t}-\bar{\Xi}_{t},\nabla f_{i}(\theta_{i,t})-\nabla f_{i}(\theta_{i,t}^{\prime})\rangle]\\
&\textstyle\geq -(1-\rho)\mathbb{E}[\|\Xi_{i,t}-\bar{\Xi}_{t}\|^2]\\
&\textstyle\quad-\frac{\lambda_{t}^2}{1-\rho}\mathbb{E}[\|\nabla f_{i}(\theta_{i,t})-\nabla f_{i}(\theta_{i,t}^{\prime})\|^2],\label{SL4}
\end{aligned}
\end{equation}
where we have used the convexity of $f_{i}(\theta)$ implying that the first term on the right hand side of~\eqref{SL4} is non-negative in the last inequality.

According to~\eqref{SL4} and Assumption~\ref{A1}-(ii), the last term on the right hand side of~\eqref{SL3} satisfies
\begin{flalign}
&-2\mathbb{E}\left[\left\langle\bar{\Xi}_{t},\lambda_{t}(\bar{g}(\boldsymbol{\theta}_{t})-\bar{g}(\boldsymbol{\theta}_{t}^{\prime}))\right\rangle\right]\textstyle\leq\frac{1-\rho}{N}\mathbb{E}[\|\boldsymbol{\Xi}_{t}-\boldsymbol{1}_{N}\otimes\bar{\Xi}_{t}\|^2]\nonumber\\
&\textstyle\quad+\frac{1}{(1-\rho) N}\sum_{i=1}^{N}\lambda_{t}^2\mathbb{E}[\|\nabla f_{i}(\theta_{i,t})-\nabla f_{i}(\theta_{i,t}^{\prime})\|^2]\nonumber\\
&\textstyle\leq \frac{1-\rho}{N}\mathbb{E}[\|\boldsymbol{\Xi}_{t}-\boldsymbol{1}_{N}\otimes\bar{\Xi}_{t}\|^2]+\frac{2H^2\lambda_{t}^2}{1-\rho}\mathbb{E}[\|\bar{\Xi}_{t}\|^2]\nonumber\\
&\textstyle\quad+\frac{2H^2\lambda_{t}^2}{(1-\rho)N}\mathbb{E}[\|\boldsymbol{\Xi}_{t}-\boldsymbol{1}_{N}\otimes\bar{\Xi}_{t}\|^2].\label{SL6}
\end{flalign}

Substituting~\eqref{SL7} and~\eqref{SL6} into~\eqref{SL3}, we arrive at
\begin{equation}
\begin{aligned}
&\textstyle\mathbb{E}[\|\bar{\Xi}_{t+1}\|^2]\leq \left(1+2H^2\lambda_{t}^2+\frac{2H^2\lambda_{t}^2}{1-\rho}\right)\mathbb{E}[\|\bar{\Xi}_{t}\|^2]\\
&\textstyle+\left(\frac{1-\rho}{N}+\frac{2H^2\lambda_{t}^2}{N}+\frac{2H^2\lambda_{t}^2}{(1-\rho)N}\right)\mathbb{E}[\|\boldsymbol{\Xi}_{t}-\boldsymbol{1}_{N}\otimes\bar{\Xi}_{t}\|^2].\label{SL8}
\end{aligned}
\end{equation}

We proceed to characterize the last term on the right hand side of~\eqref{SL8}. According to~\eqref{SL2}, we have
\begin{equation}
\begin{aligned}
&\textstyle\Xi_{i,t+1}-\bar{\Xi}_{t+1}= \sum_{j\in\mathcal{N}_{i}\cup\{i\}}w_{ij}(\Xi_{j,t}-\bar{\Xi}_{t})\\
&\quad-\lambda_{t}(g_{i}(\theta_{i,t})-g_{i}(\theta_{i,t}^{\prime}))-\lambda_{t}(\bar{g}(\boldsymbol{\theta}_{t})-\bar{g}(\boldsymbol{\theta}_{t}^{\prime})).\label{SL9}
\end{aligned}
\end{equation}
By taking the squared norm and expectation on both sides of~\eqref{SL9} and using Assumption~\ref{A1}-(ii), we obtain
\begin{flalign}
&\textstyle\mathbb{E}[\|\boldsymbol{\Xi}_{t+1}-\boldsymbol{1}_{N}\otimes\bar{\Xi}_{t+1}\|^2]\nonumber\\
&\textstyle\leq (1+(1-\rho))\rho^2\mathbb{E}[\|\boldsymbol{\Xi}_{t}-\boldsymbol{1}_{N}\otimes\bar{\Xi}_{t}\|^2]+\left(1+\frac{1}{1-\rho}\right)\lambda_{t}^2\nonumber\\
&\textstyle\quad\times\sum_{i=1}^{N}\mathbb{E}[\|(g_{i}(\theta_{i,t})-g_{i}(\theta_{i,t}^{\prime}))-(\bar{g}(\boldsymbol{\theta}_{t})-\bar{g}(\boldsymbol{\theta}_{t}^{\prime}))\|^2]\nonumber\\
&\textstyle\leq \left(\rho+\frac{8(2-\rho)H^2\lambda_{t}^2}{1-\rho}\right) \mathbb{E}[\|\boldsymbol{\Xi}_{t}-\boldsymbol{1}_{N}\otimes\bar{\Xi}_{t}\|^2]\nonumber\\
&\textstyle\quad+\frac{8N(2-\rho)H^2\lambda_{t}^2}{1-\rho}\mathbb{E}[\|\bar{\Xi}_{t}\|^2].\label{SL10}
\end{flalign}

Multiplying both sides of~\eqref{SL8} by $N$ and adding the resulting expression to both sides of~\eqref{SL10} yield 
\begin{equation}
\begin{aligned}
&N\mathbb{E}[\|\bar{\Xi}_{t+1}\|^2]+\mathbb{E}[\|\boldsymbol{\Xi}_{t+1}-\boldsymbol{1}_{N}\otimes\bar{\Xi}_{t+1}\|^2]\\
&\textstyle\leq \left(1-\rho+2H^2\lambda_{t}^2+\frac{2H^2\lambda_{t}^2}{1-\rho}+\frac{8(2-\rho)H^2\lambda_{t}^2}{1-\rho}\right)N\mathbb{E}[\|\bar{\Xi}_{t}\|^2]\\
&\textstyle+\left(\frac{8(2-\rho)H^2\lambda_{t}^2}{1-\rho}+1+2H^2\lambda_{t}^2+\frac{2H^2\lambda_{t}^2}{1-\rho}\right)\!\mathbb{E}[\|\boldsymbol{\Xi}_{t}-\boldsymbol{1}_{N}\otimes\bar{\Xi}_{t}\|^2],\nonumber
\end{aligned}
\end{equation}
which can be simplified as follows:
\begin{equation}
\begin{aligned}
&\textstyle N\mathbb{E}[\|\bar{\Xi}_{t+1}\|^2]+\mathbb{E}[\|\boldsymbol{\Xi}_{t+1}-\boldsymbol{1}_{N}\otimes\bar{\Xi}_{t+1}\|^2]\\
&\textstyle\leq (1+d_{0}\lambda_{t}^2)\left(N\mathbb{E}[\|\bar{\Xi}_{t}\|^2]+\mathbb{E}[\|\boldsymbol{\Xi}_{t}-\boldsymbol{1}_{N}\otimes\bar{\Xi}_{t}\|^2]\right),\label{SL11}
\end{aligned}
\end{equation}
where the constant $d_{0}$ is given by $d_{0}=2H^2+\frac{2H^2+8(2-\rho)H^2}{1-\rho}$.

By iterating~\eqref{SL11} from $t$ to $T$, we obtain
\begin{flalign}
&\textstyle N\mathbb{E}[\|\bar{\Xi}_{T}\|^2]+\mathbb{E}[\|\boldsymbol{\Xi}_{T}-\boldsymbol{1}_{N}\otimes\bar{\Xi}_{T}\|^2]\leq \prod_{k=t}^{T-1}(1+d_{0}\lambda_{k}^2)\nonumber\\
&\textstyle \times\left(N\mathbb{E}[\|\bar{\Xi}_{t}\|^2]+\mathbb{E}[\|\boldsymbol{\Xi}_{t}-\boldsymbol{1}_{N}\otimes\bar{\Xi}_{t}\|^2]\right).\label{SL12}
\end{flalign}

Since $\ln(1+x)\leq x$ holds for all $x>0$, we always have $\prod_{k=t}^{T-1}(1+d_{0}\lambda_{k}^2)\leq e^{d_{0}\lambda_{0}^2\sum_{k=t}^{T-1}\frac{1}{(k+1)^{2v}}}$. Further using the relationship $\sum_{k=t}^{T-1}\frac{1}{(k+1)^{2v}}\leq\int_{t}^{T}\frac{1}{x^{2v}}dx\leq \frac{1}{2v-1}(\frac{1}{t^{2v-1}}-\frac{1}{T^{2v-1}})$, we arrive at
\begin{flalign}
&N\mathbb{E}[\|\bar{\Xi}_{T}\|^2]+\mathbb{E}[\|\boldsymbol{\Xi}_{T}-\boldsymbol{1}_{N}\otimes\bar{\Xi}_{T}\|^2] \nonumber\\
&\leq d_{t\text{\tiny$\to$}T}\left(N\mathbb{E}[\|\bar{\Xi}_{t}\|^2]+\mathbb{E}[\|\boldsymbol{\Xi}_{t}-\boldsymbol{1}_{N}\otimes\bar{\Xi}_{t}\|^2]\right),\label{SL13}
\end{flalign}
where $d_{t\text{\tiny$\to$}T}$ is given by $d_{t\text{\tiny$\to$}T}=e^{\frac{20H^2\lambda_{0}^2}{(1-\rho)(2v-1)}}(t^{1-2v}-T^{1-2v})$.

By using  the inequality $\|\boldsymbol{\Xi}_{T}\|^2\leq 2N\|\bar{\Xi}_{T}\|^2+2\|\boldsymbol{\Xi}_{t}-\boldsymbol{1}_{N}\otimes\bar{\Xi}_{t}\|^2$ and the definition $\Xi_{i,t}=\theta_{i,t}-\theta_{i,t}^{\prime}$, we arrive at~\eqref{SLresult}.
\end{proof}

\subsection{Proof of Lemma~\ref{ML1}}
According to Mechanism~\ref{mechanism}, we have
\begin{flalign}
&\textstyle\mathbb{E}\left[P_{i,t}-P_{i,t}^{\prime}\right]=-C_{t}\deg(i)\left(\mathbb{E}[\|\theta_{i,t+1}^{\prime}-2\theta_{i,t}+\theta_{i,t-1}\|^2]\right.\nonumber\\
&\left.\textstyle\quad-\mathbb{E}[\|\theta_{i,t+1}-2\theta_{i,t}+\theta_{i,t-1}\|^2]\right),\label{C1L1}
\end{flalign}
where we have used the fact that agent~$i$'s action $\alpha_{i,t}$ at iteration $t$ does not affect the model parameters $\theta_{i,t}$ and $\theta_{i,t-1}$.

By using the inequality $\|a-b\|^2\geq\frac{1}{2}\|a\|^2-\|b\|^2$ for any $a,b\in\mathbb{R}^n$, the right hand side of~\eqref{C1L1} satisfies
\begin{flalign}
&\textstyle\mathbb{E}[\|\theta_{i,t+1}^{\prime}-2\theta_{i,t}+\theta_{i,t-1}\|^2-\|\theta_{i,t+1}-2\theta_{i,t}+\theta_{i,t-1}\|^2]\nonumber\\
&\textstyle\geq \frac{1}{2}\mathbb{E}[\|\theta_{i,t+1}^{\prime}-\theta_{i,t+1}\|^2]\nonumber\\
&\textstyle\quad-2\mathbb{E}[\|\theta_{i,t+1}-2\theta_{i,t}+\theta_{i,t-1}\|^2]\nonumber\\
&\textstyle=\frac{1}{2}\lambda_{t}^2\mathbb{E}[\|(a_{i,t}-1)g_{i}(\theta_{i,t})+b_{i,t}\xi_{i,t}\|^2]\nonumber\\
&\textstyle\quad-2\mathbb{E}[\|\theta_{i,t+1}-2\theta_{i,t}+\theta_{i,t-1}\|^2].\label{C1L1001}
\end{flalign}

The last term on the right hand side of~\eqref{C1L1001} satisfies
\begin{flalign}
&2\mathbb{E}[\|\theta_{i,t+1}-2\theta_{i,t}+\theta_{i,t-1}\|^2]\nonumber\\
&\textstyle\leq4\mathbb{E}[\|\lambda_{t}g_{i}(\theta_{i,t})-\lambda_{t-1}g_{i}(\theta_{i,t-1})\|^2]+4\mathbb{E}[\|\sum_{j\in\mathcal{N}_{i}\cup\{i\}}\nonumber\\
&\textstyle\quad\times w_{ij}(\theta_{j,t}-\theta_{i,t}-(\theta_{j,t-1}-\theta_{i,t-1}))\|^2].\label{C1L1002}
\end{flalign}
The first term on the right hand side of~\eqref{C1L1002} satisfies
\begin{equation}
\begin{aligned}
&\!\!\!4\mathbb{E}[\left\|\lambda_{t}g_{i}(\theta_{i,t})-\lambda_{t-1}g_{i}(\theta_{i,t-1})\right\|^2]\leq 8(\lambda_{t}\!-\!\lambda_{t-1})^2\\
&\!\!\!\times\mathbb{E}[\|g_{i}(\theta_{i,t})\|^2]\!+\!8H^2\lambda_{t}^2\mathbb{E}[\|\theta_{i,t}\!-\theta_{i,t-1}\|^2].\label{C1L1003}
\end{aligned}
\end{equation}

We proceed to estimate an upper bound on $\mathbb{E}[\|g_{i}(\theta_{i,t})\|^2]$ by using the following decomposition:
\begin{equation}
\begin{aligned}
&\mathbb{E}[\|g_{i}(\theta_{i,t})\|^2]=\mathbb{E}[\|g_{i}(\theta_{i,t})-\nabla f_{i}(\theta_{i,t})\|^2]\\
&\quad+\mathbb{E}[\|\nabla f_{i}(\theta_{i,t})-\nabla f_{i}(\theta_{i}^{*})\|^2]\\
&\leq \sigma^2+2H_{i}^2\mathbb{E}[\|\theta_{i,t}-\theta^{*}\|^2]+2H_{i}^2\|\theta_{i}^{*}-\theta^{*}\|^2.\label{XXXnnn}
\end{aligned}
\end{equation}
When $f_i(\theta)$ is strongly convex, according to Lemma 8 in~\cite{pushi1DSGD4}, we have that $\theta_{i,t}$ satisfies $\mathbb{E}[\|\theta_{i,t}-\theta^{*}\|^2]\leq \mathcal{O}(1)$. Hence, there must exist a constant $c_{3}>0$ such that $c_{3}\geq \sigma^2+2H_{i}^2\mathcal{O}(1)+2H_{i}^2\|\theta_{i}^{*}-\theta^{*}\|^2$ holds, which, combined with~\eqref{XXXnnn}, leads to $\mathbb{E}[\|g_{i}(\theta_{i,t})\|^2]\leq c_{3}$. 

On the other hand, when $f_{i}(\theta)$ is convex, we have $\mathbb{E}[\|g_{i}(\theta)\|^2]\leq L_{f}^2+\sigma^2$. By defining $c_{4}=\max\{c_{3},L_{f}^2+\sigma^2\}$, we arrive at $\mathbb{E}[\|g_i(\theta)\|^2]\leq c_{4}$.

The last term on the right hand side of~\eqref{C1L1003} satisfies
\begin{flalign}
&\textstyle8H^2\lambda_{t}^2\mathbb{E}[\|\theta_{i,t}-\theta_{i,t-1}\|^2]\nonumber\\
&\textstyle\leq16H^2\lambda_{t}^2\mathbb{E}[\|\sum_{j\in\mathcal{N}_{i}\cup\{i\}}w_{ij}(\theta_{j,t-1}-\theta_{i,t-1})\|^2]\nonumber\\
&\textstyle\quad+16H^2\lambda_{t}^2\lambda^2_{t-1}\mathbb{E}[\|g_i(\theta_{i,t-1})\|^2]\nonumber\\
&\leq 16H^2\lambda_{t}^2\mathbb{E}[\|\boldsymbol{\theta}_{t-1}-\boldsymbol{1}_{N}\otimes \bar{\theta}_{t-1}\|^2]+16H^2\lambda_{t}^2\lambda^2_{t-1}c_{4}\nonumber\\
&\leq 16H^2(c_{4}+c_{5})\lambda_{t}^2\lambda_{t-1}^2,\label{C1L101}
\end{flalign}
where in the last inequality we have used an argument similar to the consensus analysis in~\cite{DSGD}.

Substituting $\mathbb{E}[\|g_i(\theta)\|^2]\leq c_{4}$ and~\eqref{C1L101} into~\eqref{C1L1003} yields
\begin{equation}
\begin{aligned}
&\textstyle 4\mathbb{E}[\|\lambda_{t}g_{i}(\theta_{i,t})-\lambda_{t-1}g_{i}(\theta_{i,t-1})\|^2]\\
&\leq 8c_{4}(\lambda_{t}-\lambda_{t-1})^2+16H^2(c_{4}+c_{5})\lambda_{t}^2\lambda_{t-1}^2.\label{C1L102}
\end{aligned}
\end{equation}

We proceed to characterize the second term on the right hand side of~\eqref{C1L1002}. We define an auxiliary variable $\Lambda_{i,t}=\theta_{i,t}-\bar{\theta}_{t}-(\theta_{i,t-1}-\bar{\theta}_{t-1})$ and obtain
\begin{flalign}
&\textstyle4\mathbb{E}[\|\sum_{j\in\mathcal{N}_{i}\cup\{i\}}w_{ij}(\theta_{j,t}-\theta_{i,t}-(\theta_{j,t-1}-\theta_{i,t-1}))\|^2]\nonumber\\
&\textstyle=4\mathbb{E}[\|\sum_{j\in\mathcal{N}_{i}\cup\{i\}}w_{ij}(\Lambda_{j,t}-\Lambda_{i,t})\|^2].\label{C1L103}
\end{flalign}
According to the definition of $\Lambda_{i,t}$, its stacked form satisfies
\begin{equation}
\begin{aligned}
&\mathbb{E}[\|\boldsymbol{\Lambda}_{t+1}\|^2]\leq (1+(1-\rho))\rho^2\mathbb{E}[\|\boldsymbol{\Lambda}_{t}\|^2]\\
&\textstyle\quad+\frac{2-\rho}{1-\rho}\mathbb{E}[\|\lambda_{t}(\boldsymbol{g}(\boldsymbol{\theta}_{t})-\boldsymbol{1}_{N}\otimes\bar{g}(\boldsymbol{\theta}_{t}))\\
&\textstyle\quad+\lambda_{t-1}(\boldsymbol{g}(\boldsymbol{\theta}_{t-1})-\boldsymbol{1}_{N}\otimes\bar{g}(\boldsymbol{\theta}_{t-1}))\|^2].\label{C1L1031}
\end{aligned}
\end{equation}
By applying the inequality $\sum_{i=1}^{N}\|a_i-\bar{a}\|^2\leq\sum_{i=1}^{N}\|a_i\|^2$ for any $a_{i}\in\mathbb{R}^{n}$ and~\eqref{C1L102} to~\eqref{C1L1031}, we obtain
\begin{equation}
\begin{aligned}
&\mathbb{E}[\|\boldsymbol{\Lambda}_{t+1}\|^2]\leq (1+(1-\rho))\rho^2\mathbb{E}[\|\boldsymbol{\Lambda}_{t}\|^2]\\
&\textstyle\quad+\frac{2-\rho}{1-\rho}\sum_{i=1}^N\mathbb{E}[\|\lambda_{t}g_{i}(\theta_{i,t})-\lambda_{t-1}g_{i}(\theta_{i,t-1})\|^2]\\
&\textstyle\leq(1-(1-\rho))\mathbb{E}[\|\boldsymbol{\Lambda}_{t}\|^2]+c_{6}\lambda_{t}^2\lambda_{t-1}^2,\label{C1L1032}
\end{aligned}
\end{equation}
with $c_{6}=\frac{2Nc_{4}(2-\rho)(\lambda_{1}-\lambda_{0})^2}{1-\rho\lambda_{1}^2\lambda_{0}^2}+\frac{4NH^2(c_{4}+c_{5})(2-\rho)}{1-\rho}$

Applying Lemma 11 in~\cite{zijiGT} to~\eqref{C1L1032}, we obtain $\mathbb{E}[\|\boldsymbol{\Lambda}_{t}\|^2]\!\leq\! c_{7}\lambda_{t}^{2}\lambda_{t-1}^2$ with $c_{7}=\left(\frac{16v}{e\ln(\frac{2}{1+\rho})}\right)^{4v}\left(\frac{\mathbb{E}[\|\boldsymbol{\Lambda}_{1}\|^2]\rho}{c_{6}\lambda_{1}^{4}}+\frac{2}{1-\rho}\right)$.

Substituting $\mathbb{E}[\|\boldsymbol{\Lambda}_{t}\|^2]\leq c_{7}\lambda_{t}^{2}\lambda_{t-1}^2$ into~\eqref{C1L103} and then substituting~\eqref{C1L102} and~\eqref{C1L103} into~\eqref{C1L1002}, we obtain
\begin{equation}
2\mathbb{E}[\|\theta_{i,t+1}-2\theta_{i,t}+\theta_{i,t-1}\|^2]\leq c_{8}\lambda_{t}^2\lambda_{t-1}^2,\label{C1L1061}
\end{equation}
where the constant $c_{8}$ is given by $c_{8}=4c_{7}+\frac{4c_{6}(1-\rho)}{N(2-\rho)}$.

Substituting \eqref{C1L1061} into \eqref{C1L1001} and then substituting \eqref{C1L1001} into \eqref{C1L1}, we arrive at
\begin{equation}
\begin{aligned}
&\textstyle\mathbb{E}\left[P_{i,t}-P_{i,t}^{\prime}\right]\leq -\frac{C_{t}\deg(i)\lambda_{t}^2}{2}\mathbb{E}[\|(a_{i,t}-1)g_{i}(\theta_{i,t})\|^2]\\
&\textstyle\quad-\frac{C_{t}\deg(i)\lambda_{t}^2}{2}\mathbb{E}[\|b_{i,t}\xi_{i,t}\|^2]+c_{8}\deg(i)C_{t}\lambda_{t}^2\lambda_{t-1}^2.\label{C1L104}
\end{aligned}
\end{equation}

We proceed to characterize the rewards of agent $i$. According to Algorithm~\ref{algorithm}, we have $\theta_{i,t+1}^{\prime}=\theta_{i,t+1}-\lambda_{t}(a_{i,t}-1)g_{i}(\theta_{i,t})-\lambda_{t}b_{i,t}\xi_{i,t}$, which implies
\begin{equation}
\begin{aligned} 
&\textstyle N\mathbb{E}[\|\bar{\theta}_{t+1}-\bar{\theta}_{t+1}^{\prime}\|^2]\\
&\textstyle+\sum_{i=1}^{N}\mathbb{E}\left[\|\theta_{i,t+1}-\theta_{i,t+1}^{\prime}-(\bar{\theta}_{t+1}-\bar{\theta}_{t+1}^{\prime})\|^2\right]\\
&\leq3\lambda_{t}^{2}\left(\mathbb{E}\left[\|(a_{i,t}-1)g_{i}(\theta_{i,t})\|^2\right]+\mathbb{E}\left[\|b_{i,t}\xi_{i,t}\|^2\right]\right),\label{C1L105}
\end{aligned}
\end{equation}
where in the derivation we have used $1+\frac{N-1}{N}+\left(\frac{N-1}{N}\right)^2<3$.

By combining~\eqref{SLresult} from Lemma~\ref{sensitiveLemma1} and~\eqref{C1L105}, we obtain
\begin{equation}
\begin{aligned}
&\mathbb{E}[\|\theta_{i,T+1}-\theta_{i,T+1}^{\prime}\|^2]\leq 6d_{t+1\text{\tiny$\to$}T+1}\lambda_{t}^{2}\\
&\quad\times\left(\mathbb{E}[\|(a_{i,t}-1)g_{i}(\theta_{i,t})\|^2]+\mathbb{E}[\|b_{i,t}\xi_{i,t}\|^2]\right).\label{C1L106}
\end{aligned}
\end{equation}
which, combined with Assumption~\ref{A1}-(i), leads to
\begin{equation}
\begin{aligned}
&|\mathbb{E}[ R_{i}(f_{i}(\theta_{i,T+1}^{\prime}))-R_{i}(f_{i}(\theta_{i,T+1}))]|\leq L_{R,i}\sqrt{6d_{t+1\text{\tiny$\to$}T+1}}\lambda_{t}\\
&\times\left(\sqrt{\mathbb{E}[(a_{i,t}-1)^2\|g_{i}(\theta_{i,t})\|^2]}+\sqrt{\mathbb{E}[\|b_{i,t}\xi_{i,t}\|^2]}\right).\nonumber
\end{aligned}
\end{equation}
By using the preceding inequality and~\eqref{C1L104}, we obtain
\begin{equation}
\begin{aligned}
&\!\!\!\mathbb{E}\left[R_{i}(f_{i}(\theta_{i,T+1}^{\prime}))-P_{i,t}^{\prime}\right]-\mathbb{E}\left[R_{i}(f_{i}(\theta_{i,T+1}))-P_{i,t}\right]\\
&\!\!\!\textstyle\leq \!-\frac{C_{t}\deg(i)\lambda_{t}^2}{2}\mathbb{E}[\|(a_{i,t}\!-\!1)g_{i}(\theta_{i,t})\|^2]\!+\!\frac{c_{8}\deg(i)C_{t}\lambda_{t}^2\lambda_{t-1}^2}{2}\\
&\!\!\!\textstyle \quad+L_{R,i}\sqrt{6d_{t+1\text{\tiny$\to$}T+1}}\lambda_{t}\sqrt{\mathbb{E}\left[\|(a_{i,t}-1)g_{i}(\theta_{i,t})\|^2\right]}\\
&\!\!\!\textstyle\quad-\frac{C_{t}\deg(i)\lambda_{t}^2}{2}\mathbb{E}[\|b_{i,t}\xi_{i,t}\|^2]+\frac{c_{8}\deg(i)C_{t}\lambda_{t}^2\lambda_{t-1}^2}{2}\\
&\!\!\!\textstyle\quad+L_{R,i}\sqrt{6d_{t+1\text{\tiny$\to$}T+1}}\lambda_{t}\sqrt{\mathbb{E}\left[\|b_{i,t}\xi_{i,t}\|^2\right]}.\label{C1L3}
\end{aligned}
\end{equation}
It can be seen that the first three terms on the right hand side of~\eqref{C1L3} together form a downward-opening quadratic, whose positive root is
\begin{equation}
\begin{aligned}
&\!\!\!\textstyle\sqrt{\mathbb{E}[\|(a_{i,t}-1)g_{i}(\theta_{i,t})\|^2]}\\
&\!\!\!\textstyle=\frac{L_{R,i}\sqrt{6d_{t+1\text{\tiny$\to$}T+1}}+\sqrt{6L_{R,i}^2d_{t+1\text{\tiny$\to$}T+1}+c_{8}C_{t}^2\deg(i)^2\lambda_{t}^2\lambda_{t-1}^2}}{C_{t}\deg(i)\lambda_{t}}\\
&\!\!\!\textstyle\leq \frac{2L_{R,i}\sqrt{6d_{t+1\text{\tiny$\to$}T+1}}}{C_{t}\deg(i)\lambda_{t}}+\sqrt{c_{8}}\lambda_{t-1}.\label{C1L4}
\end{aligned}
\end{equation}
Similarly, the last three terms on the right hand side of~\eqref{C1L3} together form a downward-opening quadratic, whose positive root satisfies
\begin{equation}
\begin{aligned}
&\!\!\!\textstyle\sqrt{\mathbb{E}[\|b_{i,t}\xi_{i,t}\|^2]}\leq \frac{2L_{R,i}\sqrt{6d_{t+1\text{\tiny$\to$}T+1}}}{C_{t}\deg(i)\lambda_{t}}+\sqrt{c_{8}}\lambda_{t-1}.\label{C1L5}
\end{aligned}
\end{equation}

Further using definition $C_{t}=\frac{4L_{R,i}\sqrt{6d_{t+1\text{\tiny$\to$}T+1}}}{\deg(i)\lambda_{t}\kappa_{t}\delta}$, we have
\begin{equation}
\begin{aligned}
\textstyle\sqrt{\mathbb{E}[\|(a_{i,t}-1)g_{i}(\theta_{i,t})\|^2]}&\textstyle\leq \frac{1}{2}\kappa_{t}\delta+\sqrt{c_{8}}\lambda_{t-1};\\
\textstyle\sqrt{\mathbb{E}[\|b_{i,t}\xi_{i,t}\|^2]}&\textstyle\leq \frac{1}{2}\kappa_{t}\delta+\sqrt{c_{8}}\lambda_{t-1}.\label{C1L6ax}
\end{aligned}
\end{equation}
Since the decaying rate of $\lambda_{t}$ is higher than that of $\kappa_{t}$, i.e., $r<v$, we can choose $\lambda_{0}$ such that $\sqrt{c_{8}}\lambda_{t-1}\leq \sqrt{c_{8}}\lambda_{0}\leq \frac{1}{2}\kappa_{t}\delta$ holds, which, combined with~\eqref{C1L6ax}, proves
\begin{equation}
\begin{aligned}
\textstyle\sqrt{\mathbb{E}[\|(a_{i,t}-1)g_{i}(\theta_{i,t})\|^2]}&\textstyle\leq \kappa_{t}\delta;\\
\textstyle\sqrt{\mathbb{E}[\|b_{i,t}\xi_{i,t}\|^2]}&\textstyle\leq \kappa_{t}\delta,\label{C1L6}
\end{aligned}
\end{equation}
which establishes Lemma~\ref{ML1}.

{\color{blue}
\begin{remark}
The derivations following~\eqref{C1L3} show that at each iteration $t$, if agent $i$'s action $(a_{i,t},b_{i,t})$ satisfies the inequalities in~\eqref{C1L6}, i.e., $\sqrt{\mathbb{E}[\|(a_{i,t}-1)g_{i}(\theta_{i,t})\|^2]}\leq \kappa_{t}\delta$ and $ \sqrt{\mathbb{E}[\|b_{i,t}\xi_{i,t}\|^2]}\leq \kappa_{t}\delta$, and all other agents act truthfully, then we have $\mathbb{E}[R_{i}(f_{i}(\theta_{i,T+1}^{\prime}))\!-P_{i,t}^{\prime}]\!-\!\mathbb{E}[R_{i}(f_{i}(\theta_{i,T+1}))-P_{i,t}]\geq0$, which implies a nonnegative net-utility gain for agent $i$. Furthermore, 
since the truthful action corresponds to zero net-utility gain, a rational and self-interested agent will not choose an action with negative net-utility gain. Thus, any action chosen by such an agent satisfies~\eqref{C1L6}. Accordingly, Lemma~\ref{ML1} and the subsequent theorems characterize and apply to all actions of agent $i$ in the set $\mathcal{A}_{i}^{+}$, i.e., actions that yield a nonnegative net-utility gain when other agents act truthfully.
\end{remark}}

\subsection{Proof of Theorem~\ref{MT1}}
(i) When $f_{i}$ is $\mu$-strongly convex.

By using the dynamics of $\theta_{i,t+1}$ in Algorithm~\ref{algorithm}, we obtain
\begin{flalign}
&\mathbb{E}[\|\boldsymbol{\theta}_{t+1}^{\prime}-\boldsymbol{1}_{N}\otimes \theta^{*}\|^2]\nonumber\\
&\textstyle\leq \mathbb{E}[\|(W\otimes I_{n}-I_{Nn})\boldsymbol{\theta}_{t}^{\prime}-\lambda_{t}\boldsymbol{m}_{t}\|^2]+\mathbb{E}[\|\boldsymbol{\theta}_{t}^{\prime}-\boldsymbol{1}_{N}\otimes \theta^{*}\|^2]\nonumber\\
&\textstyle\quad+2\mathbb{E}[\langle(W\otimes I_{n}-I_{Nn})\boldsymbol{\theta}_{t}^{\prime}\!-\!\lambda_{t}\boldsymbol{m}_{t},\boldsymbol{\theta}_{t}^{\prime}-\boldsymbol{1}_{N}\otimes \theta^{*}\rangle].\label{D2L1}
\end{flalign}
Using the relationship $W\boldsymbol{1}_{N}=\boldsymbol{1}_{N}$, the first term on the right hand side of~\eqref{D2L1} satisfies
\begin{equation}
\begin{aligned}
&\textstyle\mathbb{E}[\|(W\otimes I_{n}-I_{Nn})\boldsymbol{\theta}_{t}^{\prime}-\lambda_{t}\boldsymbol{m}_{t}\|^2]\leq 2\lambda_{t}^2\mathbb{E}[\|\boldsymbol{m}_{t}\|^2]\\
&\quad+\!2\mathbb{E}[\|(W\otimes I_{n}-I_{Nn})(\boldsymbol{\theta}_{t}^{\prime}-\boldsymbol{1}_{N}\otimes \bar{\theta}_{t}^{\prime})\|^2].\label{D2L2}
\end{aligned}
\end{equation}
By using Assumption~\ref{A1} and~\eqref{ML1result}, we have
\begin{flalign}
&\textstyle\mathbb{E}[\|\boldsymbol{m}_{t}\|^2]=\sum_{i=1}^{N}\mathbb{E}[\|a_{i,t}g_{i}(\theta_{i,t}^{\prime})-g_{i}(\theta_{i,t}^{\prime})\|^2]\nonumber\\
&\textstyle\quad+\sum_{i=1}^{N}\mathbb{E}[\|g_{i}(\theta_{i,t}^{\prime})-\nabla f_{i}(\theta_{i,t}^{\prime})\|^2]\!+\!\sum_{i=1}^{N}\mathbb{E}[\|\nabla f_{i}(\theta_{i,t}^{\prime})\|^2]\nonumber\\
&\textstyle\quad+\sum_{i=1}^{N}\mathbb{E}[\|b_{i,t}\xi_{i,t}\|^2]\nonumber\\
&\textstyle\quad+\sum_{i=1}^{N}2\mathbb{E}[\langle a_{i,t}g_{i}(\theta_{i,t}^{\prime})-g_{i}(\theta_{i,t}^{\prime}),\nabla f_{i}(\theta_{i,t}^{\prime})\rangle]\nonumber\\
&\textstyle\leq N(3\kappa_{t}^2\delta^2+\sigma^2)+4H^2\mathbb{E}[\|\boldsymbol{\theta}_{t}^{\prime}-\boldsymbol{1}_{N}\otimes \theta^{*}\|^2]\nonumber\\
&\textstyle\quad+4H^2\mathbb{E}[\|\boldsymbol{1}_{N}\otimes\theta^{*}- \boldsymbol{\theta}^{*}\|^2],\label{D2L21}
\end{flalign}
where $\theta^{*}$ represents an optimal solution to problem~\eqref{primal} and $\boldsymbol{\theta}^{*}$ is denoted as $\boldsymbol{\theta}^{*}=\col(\theta_{1}^{*},\cdots,\theta_{N}^{*})$.

Using the relation $W\boldsymbol{1}_{N}^{\top}=\boldsymbol{1}_{N}$, the third term on the right hand side of~\eqref{D2L1} satisfies
\begin{flalign}
&2\mathbb{E}[\langle(W\otimes I_{n}-I_{Nn})\boldsymbol{\theta}_{t}^{\prime}-\lambda_{t}\boldsymbol{m}_{t},\boldsymbol{\theta}_{t}^{\prime}-\boldsymbol{1}_{N}\otimes \theta^{*}\rangle]\nonumber\\
&=2\mathbb{E}\left[(\boldsymbol{\theta}_{t}^{\prime}-\boldsymbol{1}_{N}\otimes \theta^{*})^{\top}(W\otimes I_{n}-I_{Nn})(\boldsymbol{\theta}_{t}^{\prime}-\boldsymbol{1}_{N}\otimes \theta^{*})\right]\nonumber\\
&\textstyle\quad-2\mathbb{E}[\langle\lambda_{t}\boldsymbol{m}_{t},\boldsymbol{\theta}_{t}^{\prime}-\boldsymbol{1}_{N}\otimes \theta^{*}\rangle]\nonumber\\
&\textstyle \leq 2\sum_{i=1}^{N}\mathbb{E}\left[\left\langle-\lambda_{t}(a_{i,t}g_{i}(\theta_{i,t}^{\prime})+b_{i,t}\xi_{i,t}),\theta_{i,t}^{\prime}-\theta^{*}\right\rangle\right],\label{D2L4}
\end{flalign}
where in the derivation we have omitted the negative term $\mathbb{E}\left[(\boldsymbol{\theta}_{t}^{\prime}-\boldsymbol{1}_{N}\otimes \theta^{*})^{\top}(W\otimes I_{n}-I_{Nn})(\boldsymbol{\theta}_{t}^{\prime}-\boldsymbol{1}_{N}\otimes \theta^{*})\right]$.

By using the Young's inequality and the relations $\mathbb{E}[\xi_{i,t}]=0$ and $\mathbb{E}[g_{i}(\theta_{i,t}^{\prime})]=\nabla f_{i}(\theta_{i,t}^{\prime})$, we have
\begin{flalign}
&\textstyle2\lambda_{t}\sum_{i=1}^{N}\mathbb{E}[\langle-(a_{i,t}g_{i}(\theta_{i,t}^{\prime})+b_{i,t}\xi_{i,t})+\nabla f_{i}(\bar{\theta}_{t}^{\prime}),\theta_{i,t}^{\prime}-\theta^{*}\rangle]\nonumber\\
&\textstyle\quad-2\lambda_{t}\sum_{i=1}^{N}\mathbb{E}[\langle\nabla f_{i}(\bar{\theta}_{t}^{\prime}),\theta_{i,t}^{\prime}-\theta^{*}\rangle]\nonumber\\
&\textstyle\leq \frac{4N\lambda_{t}\kappa_{t}^2\delta^2}{\mu}+\frac{\mu\lambda_{t}}{4}\mathbb{E}[\|\boldsymbol{\theta}_{t}^{\prime}-\boldsymbol{1}_{N}\otimes \theta^{*}\|^2]\nonumber\\
&\textstyle\quad+\frac{4H^2\lambda_{t}}{\mu}\mathbb{E}[\|\boldsymbol{\theta}_{t}^{\prime}-\boldsymbol{1}_{N}\otimes \bar{\theta}_{t}^{\prime}\|^2]\!+\!\frac{\mu\lambda_{t}}{4}\mathbb{E}[\|\boldsymbol{\theta}_{t}^{\prime}-\boldsymbol{1}_{N}\otimes \theta^{*}\|^2]\nonumber\\
&\textstyle\quad-2\sum_{i=1}^{N}\mathbb{E}[\langle\lambda_{t}\nabla f_{i}(\bar{\theta}_{t}^{\prime}),\theta_{i,t}^{\prime}-\bar{\theta}_{t}^{\prime}\rangle]\nonumber\\
&\quad-2N\mathbb{E}[\langle\lambda_{t}(\nabla F(\bar{\theta}_{t}^{\prime})-\nabla F(\theta^{*})),\bar{\theta}_{t}^{\prime}-\theta^{*}\rangle].\label{D2L5}
\end{flalign}

The fifth term on the right hand side of~\eqref{D2L5} satisfies
\begin{flalign}
&\textstyle-2\sum_{i=1}^{N}\mathbb{E}[\langle\lambda_{t}\nabla f_{i}(\bar{\theta}_{t}^{\prime}),\theta_{i,t}^{\prime}-\bar{\theta}_{t}^{\prime}\rangle]\nonumber\\
&\textstyle=-2\sum_{i=1}^{N}\mathbb{E}[\langle\lambda_{t}\left(\nabla f_{i}(\bar{\theta}_{t}^{\prime})-\nabla f_{i}(\theta_{i}^{*})\right),\theta_{i,t}^{\prime}-\bar{\theta}_{t}^{\prime}\rangle]\nonumber\\
&\textstyle\leq2\lambda_{t}^2H^2N\mathbb{E}[\|\bar{\theta}_{t}^{\prime}-\theta^{*}\|^2]+2\lambda_{t}^2H^2\mathbb{E}[\|\boldsymbol{\theta}^{*}-\boldsymbol{1}_{N}\otimes \theta^{*}\|^2]\nonumber\\
&\textstyle\quad+\mathbb{E}[\|\boldsymbol{\theta}_{t}^{\prime}-\boldsymbol{1}_{N}\otimes \bar{\theta}_{t}^{\prime}\|^2].\label{D2L6}
\end{flalign}

The $\mu$-strong convexity of $f_{i}(\theta)$ implies that the last term on the right hand side of~\eqref{D2L5} satisfies $-2N\mathbb{E}[\langle\lambda_{t}(\nabla F(\bar{\theta}_{t}^{\prime})-\nabla F(\theta^{*})),\bar{\theta}_{t}^{\prime}-\theta^{*}\rangle]\leq -2N\lambda_{t}\mu\mathbb{E}[\|\bar{\theta}_{t}^{\prime}-\theta^{*}\|^2]$, which, combined with~\eqref{D2L6}, leads to
\begin{equation}
\begin{aligned}
&\textstyle-2\sum_{i=1}^{N}\mathbb{E}[\langle\lambda_{t}\nabla f_{i}(\bar{\theta}_{t}^{\prime}),\theta_{i,t}^{\prime}-\theta^{*}\rangle]\\
&\leq -\lambda_{t}(\mu-2H^2\lambda_{t})\mathbb{E}[\|\boldsymbol{\theta}_{t}^{\prime}-\boldsymbol{1}_{N}\otimes \theta^{*}\|^2]\\
&\quad+\left(2\lambda_{t}\left(\mu-2H^2\lambda_{t}\right)+1\right)\mathbb{E}[\|\boldsymbol{\theta}_{t}^{\prime}-\boldsymbol{1}_{N}\otimes \bar{\theta}_{t}^{\prime}\|^2]\\
&\textstyle\quad+2H^2\lambda_{t}^2\mathbb{E}[\|\boldsymbol{1}_{N}\otimes \theta^{*}-\boldsymbol{\theta}^{*}\|^2],
\label{D2L61}
\end{aligned}
\end{equation}
where we have used $N\mathbb{E}[\|\bar{\theta}_{t}^{\prime}-\theta^{*}\|^2]=\sum_{i=1}^{N}\mathbb{E}[\|\bar{\theta}_{t}^{\prime}-\theta_{i,t}^{\prime}+\theta_{i,t}^{\prime}-\theta^{*}\|^2]\geq \frac{1}{2}\sum_{i=1}^{N}\mathbb{E}[\|\theta_{i,t}^{\prime}-\theta^{*}\|^2]-\sum_{i=1}^{N}\mathbb{E}[\|\theta_{i,t}^{\prime}-\bar{\theta}_{t}^{\prime}\|^2]$.

Combining~\eqref{D2L4}, ~\eqref{D2L5}, and~\eqref{D2L61}, we arrive at
\begin{flalign}
&2\mathbb{E}[\langle(W\otimes I_{n}-I_{Nn})\boldsymbol{\theta}_{t}^{\prime}-\lambda_{t}\boldsymbol{m}_{t},\boldsymbol{\theta}_{t}^{\prime}-\boldsymbol{1}_{N}\otimes \theta^{*}\rangle]\nonumber\\
&\textstyle\leq \left(2H^2\lambda_{t}-\frac{\mu}{2}\right)\lambda_{t}\mathbb{E}[\|\boldsymbol{\theta}_{t}^{\prime}-\boldsymbol{1}_{N}\otimes \theta^{*}\|^2]\nonumber\\
&\textstyle\quad+\left(\frac{4H^2\lambda_{t}}{\mu}+2\lambda_{t}\left(\mu-2H^2\lambda_{t}\right)+1\right)\mathbb{E}[\|\boldsymbol{\theta}_{t}^{\prime}-\boldsymbol{1}_{N}\otimes \bar{\theta}_{t}^{\prime}\|^2]\nonumber\\
&\textstyle\quad+2H^2\lambda_{t}^2\mathbb{E}[\|\boldsymbol{1}_{N}\otimes \theta^{*}-\boldsymbol{\theta}^{*}\|^2]+\frac{4N\lambda_{t}\kappa_{t}^2\delta^2}{\mu}.\label{D2L7}
\end{flalign}

Further substituting~\eqref{D2L21} into~\eqref{D2L2}, and then substituting~\eqref{D2L2} and~\eqref{D2L7} into~\eqref{D2L1}, we obtain
\begin{equation}
\begin{aligned}
&\textstyle\mathbb{E}[\|\boldsymbol{\theta}_{t+1}^{\prime}-\boldsymbol{1}_{N}\otimes\theta^{*}\|^2]\\
&\textstyle\leq \left(1-\frac{\mu}{2}\lambda_{t}+10H^2\lambda_{t}^2\right)\mathbb{E}[\|\boldsymbol{\theta}_{t}^{\prime}-\boldsymbol{1}_{N}\otimes\theta^{*}\|^2]\\
&\quad+\left(5+2\lambda_{t}\mu\right)\mathbb{E}[\|\boldsymbol{\theta}_{t}^{\prime}-\boldsymbol{1}_{N}\otimes \bar{\theta}_{t}^{\prime}\|^2]\\
&\textstyle\quad+10H^2\lambda_{t}^2\mathbb{E}[\|\boldsymbol{\theta}^{*}-\boldsymbol{1}_{N}\otimes \theta^{*}\|^2]\\
&\quad+N(6\lambda_{0}+4\mu^{-1})\lambda_{t}\kappa_{t}^2\delta^2+2N\lambda_{t}^2\sigma^2,\label{D2L8}
\end{aligned}
\end{equation}
where in the derivation we have omitted the negative term $-4H^2\lambda_{t}^2\mathbb{E}[\|\boldsymbol{\theta}_{t}^{\prime}-\boldsymbol{1}_{N}\otimes \bar{\theta}_{t}^{\prime}\|^2]$.

We proceed to characterize the second term on the right hand side of~\eqref{D2L8}. By using the relation $\mathbb{E}[\|\boldsymbol{m}_{t}-\boldsymbol{1}_{N}\otimes\bar{m}_{t}\|^2]\leq \mathbb{E}[\|\boldsymbol{m}_{t}\|^2]$ and~\eqref{D2L21}, we have  
\begin{flalign}
&\mathbb{E}[\|\boldsymbol{\theta}_{t+1}^{\prime}-\boldsymbol{1}_{N}\otimes \bar{\theta}_{t+1}^{\prime}\|^2]\nonumber\\
&\textstyle\leq \rho\mathbb{E}[\|\boldsymbol{\theta}_{t}^{\prime}-\boldsymbol{1}_{N}\otimes \bar{\theta}_{t}^{\prime}\|^2]+\frac{4H^2(2-\rho)}{1-\rho}\lambda_{t}^2\mathbb{E}[\|\boldsymbol{\theta}_{t}^{\prime}-\boldsymbol{1}_{N}\otimes\theta^{*}\|^2]\nonumber\\
&\textstyle\quad+\frac{4H^2(2-\rho)}{1-\rho}\lambda_{t}^2\mathbb{E}[\|\boldsymbol{\theta}^{*}-\boldsymbol{1}_{N}\otimes\theta^{*}\|^2]\nonumber\\
&\textstyle\quad+\frac{3N(2-\rho)\lambda_{0}}{1-\rho}\lambda_{t}\kappa_{t}^2\delta^2+\frac{N(2-\rho)}{1-\rho}\lambda_{t}^2\sigma^2.\label{D2L81}
\end{flalign}

Multiplying both sides of~\eqref{D2L81} by $\frac{12}{1-\rho}$ and adding the resulting expression to both sides of~\eqref{D2L8} yield 
\begin{equation}
\begin{aligned}
&\textstyle\mathbb{E}[\|\boldsymbol{\theta}_{t+1}^{\prime}-\boldsymbol{1}_{N}\otimes\theta^{*}\|^2]+\frac{12}{1-\rho}\mathbb{E}[\|\boldsymbol{\theta}_{t+1}^{\prime}-\boldsymbol{1}_{N}\otimes \bar{\theta}_{t+1}^{\prime}\|^2]\\
&\textstyle\leq \left(1-\frac{\mu}{2}\lambda_{t}+10H^2\lambda_{t}^2+\frac{48H^2(2-\rho)}{(1-\rho)^2}\lambda_{t}^2\right)\mathbb{E}[\|\boldsymbol{\theta}_{t}^{\prime}-\boldsymbol{1}_{N}\otimes\theta^{*}\|^2]\\
&\textstyle+\!\left(5+2\lambda_{t}\mu+\frac{12\rho}{1-\rho}\right)\mathbb{E}[\|\boldsymbol{\theta}_{t}^{\prime}\!-\!\boldsymbol{1}_{N}\otimes \bar{\theta}_{t}^{\prime}\|^2]\!+c_{9}\lambda_{t}\kappa_{t}^2\delta^2\!+c_{10}\lambda_{t}^2,\nonumber
\end{aligned}
\end{equation}
with $c_{9}=N(6\lambda_{0}+\frac{4}{\mu})+\frac{72N\lambda_{0}}{(1-\rho)^2}$ and $c_{10}=2N\sigma^2+\frac{12N\sigma^2(2-\rho)}{(1-\rho)^2}$.

Since $\lambda_{t}$ is a decaying sequence, we can set $\lambda_{0}$ such that $10H^2\lambda_{0}+\frac{48H^2(2-\rho)}{(1-\rho)^2}\lambda_{0}\leq \frac{\mu}{4}$, $2\lambda_{0}\mu<1$, and $1-\frac{\mu}{4}\lambda_{0}>\frac{1+\rho}{2}$ hold. In this case, the preceding inequality can be rewritten as follows:
\begin{equation}
\begin{aligned}
&\textstyle\mathbb{E}[\|\boldsymbol{\theta}_{t+1}^{\prime}-\boldsymbol{1}_{N}\otimes\theta^{*}\|^2]+\frac{12}{1-\rho}\mathbb{E}[\|\boldsymbol{\theta}_{t+1}^{\prime}-\boldsymbol{1}_{N}\otimes \bar{\theta}_{t+1}^{\prime}\|^2]\\
&\textstyle\leq \left(1\!-\!\frac{\mu\lambda_{t}}{4}\right)\left(\mathbb{E}[\|\boldsymbol{\theta}_{t}^{\prime}-\boldsymbol{1}_{N}\otimes\theta^{*}\|^2]\right.\\
&\textstyle\left.\textstyle\quad+\frac{12}{1-\rho}\mathbb{E}[\|\boldsymbol{\theta}_{t}^{\prime}-\boldsymbol{1}_{N}\otimes \bar{\theta}_{t}^{\prime}\|^2]\right)+\frac{c_{9}\delta^2+c_{10}\lambda_{0}}{\lambda_{0}}\lambda_{t}^2.\nonumber
\end{aligned}
\end{equation}

Combining Lemma 5-(i) in the arxiv version of~\cite{JDPziji} and~the preceding inequality yields
\begin{equation}
\textstyle\sum_{i=1}^{N}\mathbb{E}[\|\theta_{i,t}^{\prime}-\theta^{*}\|^2]\leq C_{1}\lambda_{t},\label{D2L10}
\end{equation}
where  $C_{1}=\max\{\mathbb{E}[\|\boldsymbol{\theta}_{0}^{\prime}-\boldsymbol{1}_{N}\otimes\theta^{*}\|^2]+\frac{12}{1-\rho}\mathbb{E}[\|\boldsymbol{\theta}_{0}^{\prime}-\boldsymbol{1}_{N}\otimes \bar{\theta}_{0}^{\prime}\|^2],\frac{4(c_{9}\delta^2+c_{10}\lambda_{0})\lambda_{0}}{\mu\lambda_{0}-4v}\}$.

Furthermore, based on the definition of $C_{1}$, we have $C_{1}\leq \mathcal{O}\left(\frac{H^{2}(\sigma^{2}+\delta^2)}{\mu(1-\rho)^2}\right)$, which, combined with~\eqref{D2L10}, proves Theorem~\ref{MT1}-(i).
\vspace{0.6em}

(ii) When $f_{i}$ is general convex.

By using an argument similar to the derivations of~\eqref{D2L1},~\eqref{D2L2}, and~\eqref{D2L4}, we have
\begin{equation}
\begin{aligned}
&\textstyle\mathbb{E}[\|\boldsymbol{\theta}_{t+1}^{\prime}-\boldsymbol{1}_{N}\otimes \theta^{*}\|^2]\leq \mathbb{E}[\|\boldsymbol{\theta}_{t}^{\prime}-\boldsymbol{1}_{N}\otimes \theta^{*}\|^2]\\
&\textstyle\quad+2\mathbb{E}[\|(W\otimes I_{n}-I_{Nn})(\boldsymbol{\theta}_{t}^{\prime}-\boldsymbol{1}_{N}\otimes \bar{\theta}_{t}^{\prime})\|^2]\\
&\textstyle\quad+2\lambda_{t}^2\mathbb{E}[\|\boldsymbol{m}_{t}\|^2]-2\mathbb{E}[\langle\lambda_{t}\boldsymbol{m}_{t},\boldsymbol{\theta}_{t}^{\prime}-\boldsymbol{1}_{N}\otimes \theta^{*}\rangle].\label{3E4}
\end{aligned}
\end{equation}

The last term on the right hand side of~\eqref{3E4} satisfies
\begin{equation}
\begin{aligned}
&\textstyle-2\mathbb{E}[\langle\lambda_{t}\boldsymbol{m}_{t},\boldsymbol{\theta}_{t}^{\prime}-\boldsymbol{1}_{N}\otimes \theta^{*}\rangle]\\
&\textstyle =-2\sum_{i=1}^{N}\mathbb{E}[\langle\lambda_{t}g_{i}(\theta_{i,t}^{\prime}),\theta_{i,t}^{\prime}- \theta^{*}\rangle]\\
&\textstyle\quad+2\sum_{i=1}^{N}\mathbb{E}[\langle\lambda_{t}g_{i}(\theta_{i,t}^{\prime})-\lambda_{t}m_{i,t},\theta_{i,t}^{\prime}- \theta^{*}\rangle].\label{3E5}
\end{aligned}
\end{equation}
The first term on the right hand side of~\eqref{3E5} satisfies
\begin{equation}
\begin{aligned}
&\textstyle-2\sum_{i=1}^{N}\mathbb{E}\left[\left\langle\lambda_{t}g_{i}(\theta_{i,t}^{\prime}),\theta_{i,t}^{\prime}- \theta^{*}\right\rangle\right]\\
&\textstyle=-2\lambda_{t}\sum_{i=1}^{N}\mathbb{E}[\langle \nabla f_{i}(\theta_{i,t}^{\prime})-\nabla f_{i}(\bar{\theta}_{t}^{\prime}),\theta_{i,t}^{\prime}- \theta^{*}\rangle]\\
&\textstyle\quad-2\lambda_{t}\sum_{i=1}^{N}\mathbb{E}[\langle \nabla f_{i}(\bar{\theta}_{t}^{\prime}),\theta_{i,t}^{\prime}- \theta^{*}\rangle].\label{3E6}
\end{aligned}
\end{equation}
By introducing an auxiliary sequence $\beta_{t}=\frac{1}{(t+1)^{u}}$ with $1<u<2v$, the first term on the right hand side of~\eqref{3E6} satisfies
\begin{equation}
\begin{aligned}
&\textstyle-2\lambda_{t}\sum_{i=1}^{N}\mathbb{E}[\langle \nabla f_{i}(\theta_{i,t}^{\prime})-\nabla f_{i}(\bar{\theta}_{t}^{\prime}),\theta_{i,t}^{\prime}- \theta^{*}\rangle]\\
&\textstyle\leq \frac{\lambda_{t}^{2}}{\beta_{t}}H^2\mathbb{E}[\|\boldsymbol{\theta}_{t}^{\prime}-\boldsymbol{1}_{N}\otimes \bar{\theta}_{t}^{\prime}\|^2]+\beta_{t}\mathbb{E}[\|\boldsymbol{\theta}_{t}^{\prime}-\boldsymbol{1}_{N}\otimes \theta^{*}\|^2].\label{3E7}
\end{aligned}
\end{equation}
By using the relations $\sum_{i=1}^{N}\nabla f_{i}(\bar{\theta}_{t}^{\prime})=N\nabla F(\bar{\theta}_{t}^{\prime})$ and $\sum_{i=1}^{N}\theta_{i,t}^{\prime}=N\bar{\theta}_{t}^{\prime}$, and the convexity of $F(\theta)$, the second term on the right hand side of~\eqref{3E6} satisfies
\begin{equation}
\textstyle-\sum_{i=1}^{N}\mathbb{E}[\langle \nabla f_{i}(\bar{\theta}_{t}^{\prime}),\theta_{i,t}^{\prime}- \theta^{*}\rangle]\leq -N\mathbb{E}[F(\bar{\theta}_{t}^{\prime})-F(\theta^{*})].\nonumber
\end{equation}
Substituting~\eqref{3E7} and the preceding inequality into~\eqref{3E6}, the first term on the right hand side of~\eqref{3E5} satisfies
\begin{equation}
\begin{aligned}
&\textstyle-2\sum_{i=1}^{N}\mathbb{E}[\langle\lambda_{t}g_{i}(\theta_{i,t}^{\prime}),\theta_{i,t}^{\prime}- \theta^{*}\rangle]\\
&\textstyle\leq -2N\lambda_{t}\mathbb{E}[F(\bar{\theta}_{t}^{\prime})-F(\theta^{*})]+\frac{\lambda_{t}^{2}}{\beta_{t}}H^2\mathbb{E}[\|\boldsymbol{\theta}_{t}^{\prime}-\boldsymbol{1}_{N}\otimes \bar{\theta}_{t}^{\prime}\|^2]\\
&\textstyle\quad+\beta_{t}\mathbb{E}[\|\boldsymbol{\theta}_{t}^{\prime}-\boldsymbol{1}_{N}\otimes \theta^{*}\|^2].\label{3E9}
\end{aligned}
\end{equation}

We proceed to characterize the second term on the right hand side of~\eqref{3E5}. By using the Young's inequality, we have
\begin{equation}
\begin{aligned}
&\textstyle2\sum_{i=1}^{N}\mathbb{E}[\langle\lambda_{t}g_{i}(\theta_{i,t}^{\prime})-\lambda_{t}m_{i,t},\theta_{i,t}^{\prime}- \theta^{*}\rangle]\\
&\textstyle\leq \frac{\lambda_{t}^{2}}{\beta_{t}}\sum_{i=1}^{N}\mathbb{E}[\|a_{i,t}g_{i}(\theta_{i,t}^{\prime})-g_{i}(\theta_{i,t}^{\prime})\|^2]\\
&\textstyle\quad+\beta_{t}\mathbb{E}[\|\boldsymbol{\theta}_{t}^{\prime}-\boldsymbol{1}_{N}\otimes \theta^{*}\|^2].\label{3E10}
\end{aligned}
\end{equation}

Substituting~\eqref{3E9} and~\eqref{3E10} into~\eqref{3E5} and then substituting~\eqref{3E5} into~\eqref{3E4}, we arrive at
\begin{equation}
\begin{aligned}
&\textstyle\mathbb{E}[\|\boldsymbol{\theta}_{t+1}^{\prime}-\boldsymbol{1}_{N}\otimes \theta^{*}\|^2]\leq-2N\lambda_{t}\mathbb{E}[F(\bar{\theta}_{t}^{\prime})-F(\theta^{*})]\\
&\quad+(1+2\beta_{t}) \mathbb{E}[\|\boldsymbol{\theta}_{t}^{\prime}-\boldsymbol{1}_{N}\otimes \theta^{*}\|^2]+\Phi_{t},\label{3E11}
\end{aligned}
\end{equation}
where the term $\Phi_{t}$ is given by
\begin{flalign}
\Phi_{t}&\textstyle=\frac{\lambda_{t}^{2}}{\beta_{t}}\sum_{i=1}^{N}\mathbb{E}[\|(a_{i,t}-1)g_{i}(\theta_{i,t}^{\prime})\|^2]+2\lambda_{t}^2\mathbb{E}[\|\boldsymbol{m}_{t}\|^2]\nonumber\\
&\textstyle\quad+\left(2+\frac{\lambda_{t}^{2}}{\beta_{t}}H^2\right)\mathbb{E}[\|\boldsymbol{\theta}_{t}^{\prime}-\boldsymbol{1}_{N}\otimes \bar{\theta}_{t}^{\prime}\|^2].\label{3E12}
\end{flalign}
Since the relation $F(\bar{\theta}_{t}^{\prime})\geq F(\theta^{*})$ always holds, we omit the negative term $-2N\lambda_{t}\mathbb{E}[F(\bar{\theta}_{t}^{\prime})-F(\theta^{*})]$ in~\eqref{3E11} to obtain
\begin{flalign}
&\textstyle\mathbb{E}[\|\boldsymbol{\theta}_{t+1}^{\prime}-\boldsymbol{1}_{N}\otimes \theta^{*}\|^2]\!\leq \!(1+2\beta_{t}) \mathbb{E}[\|\boldsymbol{\theta}_{t}^{\prime}-\boldsymbol{1}_{N}\otimes \theta^{*}\|^2]\!+\!\Phi_{t},\nonumber\\
&\textstyle\leq e^{\frac{2\beta_{0}(u-1)+1}{u-1}}\left(\mathbb{E}[\|\boldsymbol{\theta}_{0}^{\prime}-\boldsymbol{1}_{N}\otimes \theta^{*}\|^2]+\sum_{k=0}^{t}\Phi_{k}\right),\label{3E13}
\end{flalign}
where in the derivation we have used an argument similar to the derivation of~\eqref{SL12}.

Next, we estimate an upper bound on $\sum_{k=0}^{t}\Phi_{k}$. To this end, we first characterize the last term on the right hand side of~\eqref{3E12}. By using the relation $\mathbb{E}[\|\boldsymbol{m}_{t}-\boldsymbol{1}_{N}\otimes\bar{m}_{t}\|^2]\leq \mathbb{E}[\|\boldsymbol{m}_{t}\|^2]$ and $\left(1+(1-\rho)\right)\rho^2\leq \rho$, we have
\begin{equation}
\begin{aligned}
\mathbb{E}[\|\boldsymbol{\theta}_{t+1}^{\prime}-\boldsymbol{1}_{N}\otimes \bar{\theta}_{t+1}^{\prime}\|^2]&\textstyle\leq \rho\mathbb{E}[\|\boldsymbol{\theta}_{t}^{\prime}-\boldsymbol{1}_{N}\otimes \bar{\theta}_{t}^{\prime}\|^2]\\
&\textstyle+\frac{2-\rho}{1-\rho}\lambda_{t}^2\mathbb{E}[\|\boldsymbol{m}_{t}\|^2].\label{cons2}
\end{aligned}
\end{equation}
Assumption~\ref{A1} and~\eqref{ML1result} imply that the last term on the right hand side of~\eqref{cons2} satisfies
\begin{flalign}
\mathbb{E}[\|\boldsymbol{m}_{t}\|^2]\leq N(2\delta^2+\sigma^2+L_{f}^2+2\delta L_{f}).\label{cons3}
\end{flalign}

Substituting~\eqref{cons3} into~\eqref{cons2}, we obtain
\begin{equation}
\begin{aligned}
&\textstyle\mathbb{E}[\|\boldsymbol{\theta}_{t+1}^{\prime}-\boldsymbol{1}_{N}\otimes \bar{\theta}_{t+1}^{\prime}\|^2]\leq\rho\mathbb{E}\left[\|\boldsymbol{\theta}_{t}^{\prime}-\boldsymbol{1}_{N}\otimes \bar{\theta}_{t}^{\prime}\|^2\right]\\
&\textstyle\quad+ \frac{4N(2-\rho)(2\delta^2+\sigma^2+L_{f}^2+2\delta L_{f})}{1-\rho}\lambda_{t}^2.\label{cons4}
\end{aligned}
\end{equation}
By using the relationship $\mathbb{E}[\|\boldsymbol{\theta}_{0}^{\prime}-\boldsymbol{1}_{N}\otimes \bar{\theta}_{0}^{\prime}\|^2]=\mathbb{E}[\|\boldsymbol{\theta}_{0}-\boldsymbol{1}_{N}\otimes \bar{\theta}_{0}\|^2]$ and iterating~\eqref{cons4} from $0$ to $t$, we arrive at
\begin{equation}
\begin{aligned}
&\textstyle\mathbb{E}[\|\boldsymbol{\theta}_{t}^{\prime}-\boldsymbol{1}_{N}\otimes \bar{\theta}_{t}^{\prime}\|^2]\leq\rho^t\mathbb{E}[\|\boldsymbol{\theta}_{0}-\boldsymbol{1}_{N}\otimes \bar{\theta}_{0}\|^2]\\
&\textstyle\quad+\frac{4N(2-\rho)(2\delta^2+\sigma^2+L_{f}^2+2\delta L_{f})}{(1-\rho)^2}\lambda_{t}^2\leq c_{11}\lambda_{t}^2,\label{cons5}
\end{aligned}
\end{equation}
with $c_{11}=\frac{16\mathbb{E}[\|\boldsymbol{\theta}_{0}-\boldsymbol{1}_{N}\otimes \bar{\theta}_{0}\|^2]}{e^2(1-\rho)^2\lambda_0^2}+\frac{4N(2-\rho)(2\delta^2+\sigma^2+L_{f}^2+2\delta L_{f})}{(1-\rho)^2}$. Here, in the last inequality, we have used the fact that Lemma~7 in~\cite{zijiGT} proves $\rho^t\leq\frac{16}{e^2(\ln(\rho))^2(t+1)^2}\leq\frac{16\lambda_t^2}{e^2(\ln(\rho))^2\lambda_0^2}$, which implies that for any $\rho\in(0,1)$, we have $-\ln(\rho)\geq1-\rho$.

By substituting~\eqref{cons3} and~\eqref{cons5}  into~\eqref{3E12} and using the relationship $\mathbb{E}[\|a_{i,k}g_{i}(\theta_{i,k}^{\prime})-g_{i}(\theta_{i,k}^{\prime})\|^2]\leq \kappa_{k}^2\delta^2$, we obtain
\begin{equation}
\begin{aligned}
\textstyle\sum_{k=0}^{t}\Phi_{k}&\textstyle\leq \sum_{k=0}^{t}\left(\frac{N\lambda_{0}^{2}\delta^2}{(k+1)^{2v+2r-u}}+\frac{c_{12}}{(k+1)^{2v}}\right),\label{3E16}
\end{aligned}
\end{equation}
with $c_{12}= 4N\delta^2+2N\sigma^2+2NL_f^2+4N\delta L_{f}+c_{3}(2+\lambda_{0}^2H^2)$. Here, we have used the relation $\frac{\lambda_{k}^{2}}{\beta_{k}}=\frac{\lambda_{0}^{2}}{(k+1)^{2v-u}}\leq \lambda_{0}^{2}$.

By using the following inequality:
\begin{equation} \textstyle\sum_{k=0}^{t}\frac{1}{(k+1)^{u}}\leq 1+\int_{k=1}^{\infty}\frac{1}{x^{u}}dx\leq \frac{u}{u-1},\label{3E17}
\end{equation}
which is true for any $u>1$, we can rewrite~\eqref{3E16} as follows:
\begin{equation}
\textstyle\sum_{k=0}^{t}\Phi_{k}\leq \frac{(2v+2r-u)N\lambda_{0}^{2}\delta^2}{2v+2r-u-1}+\frac{2vc_{12}}{2v-1}\triangleq c_{13},\label{3E18}
\end{equation}
where the constant $c_{12}$ is given in~\eqref{3E16}.

Substituting~\eqref{3E18} into~\eqref{3E13}, we can arrive at
\begin{equation}
\begin{aligned}
&\textstyle\mathbb{E}[\|\boldsymbol{\theta}_{t+1}^{\prime}-\boldsymbol{1}_{N}\otimes \theta^{*}\|^2]\\
&\textstyle\leq e^{\frac{2\beta_{0}(u-1)+1}{u-1}}\left(\mathbb{E}[\|\boldsymbol{\theta}_{0}^{\prime}-\boldsymbol{1}_{N}\otimes \theta^{*}\|^2]+c_{13}\right).\label{3E19}
\end{aligned}
\end{equation}

We proceed to sum both sides of~\eqref{3E11} from $0$ to $T$:
\begin{equation}
\begin{aligned}
&\textstyle\sum_{t=0}^{T}2N\lambda_{t}\mathbb{E}[F(\bar{\theta}_{t}^{\prime })-F(\theta^{*})]\\
&\textstyle\leq -\sum_{t=0}^{T}\mathbb{E}[\|\boldsymbol{\theta}_{t+1}^{\prime}-\boldsymbol{1}_{N}\otimes \theta^{*}\|^2]\\
&\textstyle\quad+\sum_{t=0}^{T}(1+2\beta_{t})\mathbb{E}[\|\boldsymbol{\theta}_{t}^{\prime}-\boldsymbol{1}_{N}\otimes \theta^{*}\|^2]+\sum_{t=0}^{T}\Phi_{t}.\label{3E20}
\end{aligned}
\end{equation}
By using $\beta_{0}=1$, the first and second terms on the right hand side of~\eqref{3E20} can be simplified as follows:
\begin{flalign}
&\textstyle\sum_{t=0}^{T}(1+2\beta_{t})\mathbb{E}[\|\boldsymbol{\theta}_{t}^{\prime}-\boldsymbol{1}_{N}\otimes \theta^{*}\|^2]\nonumber\\
&\textstyle\quad-\sum_{t=0}^{T}\mathbb{E}[\|\boldsymbol{\theta}_{t+1}^{\prime}-\boldsymbol{1}_{N}\otimes \theta^{*}\|^2]\nonumber\\
&\textstyle\leq 2\mathbb{E}[\|\boldsymbol{\theta}_{0}^{\prime}-\boldsymbol{1}_{N}\otimes \theta^{*}\|^2]+\sum_{t=1}^{T}2\beta_{t}\mathbb{E}[\|\boldsymbol{\theta}_{t}^{\prime}-\boldsymbol{1}_{N}\otimes \theta^{*}\|^2]\nonumber\\
&\textstyle\quad+\mathbb{E}[\|\boldsymbol{\theta}_{0}^{\prime}-\boldsymbol{1}_{N}\otimes \theta^{*}\|^2]-\mathbb{E}[\|\boldsymbol{\theta}_{t+1}^{\prime}-\boldsymbol{1}_{N}\otimes \theta^{*}\|^2]\nonumber\\
&\textstyle\leq \sum_{t=1}^{T}\frac{2}{(t+1)^{u}}\left(e^{\frac{2(u-1)+1}{u-1}}\left(\mathbb{E}[\|\boldsymbol{\theta}_{0}^{\prime}-\boldsymbol{1}_{N}\otimes \theta^{*}\|^2]+c_{13}\right)\right)\nonumber\\
&\textstyle\quad+3\mathbb{E}[\|\boldsymbol{\theta}_{0}^{\prime}-\boldsymbol{1}_{N}\otimes \theta^{*}\|^2]\nonumber\\
&\textstyle\leq \left(\frac{2ue^{\frac{2u-1}{u-1}}}{u-1}+3\right)\mathbb{E}[\|\boldsymbol{\theta}_{0}^{\prime}\!-\!\boldsymbol{1}_{N}\otimes \theta^{*}\|^2]\!+\!\frac{2c_{13}u}{u-1}\triangleq c_{14},\label{3E21}
\end{flalign}
where we have used~\eqref{3E19} in the second inequality and~\eqref{3E17} in the last inequality.

Substituting~\eqref{3E18} and~\eqref{3E21} into~\eqref{3E20}, and using $\lambda_{T}\leq\lambda_{t}$ for any $t\in[0,T]$, we have $\sum_{t=0}^{T}2N\lambda_{t}\mathbb{E}[F(\bar{\theta}_{t}^{\prime })-F(\theta^{*})]\leq c_{13}+c_{14},$ which further implies
\begin{equation}
\textstyle\frac{1}{T+1}\sum_{t=0}^{T}\mathbb{E}\left[F(\bar{\theta}_{t}^{\prime })-F(\theta^{*})\right]\leq \frac{C_{2}}{N(T+1)^{1-v}},\label{3E22}
\end{equation}
where $C_{2}\!=\!\frac{c_{13}+c_{14}}{2\lambda_{0}}$ with $c_{13}$ given in~\eqref{3E18} and $c_{6}$ given in~\eqref{3E21}.

Assumption~\ref{A1} implies $\mathbb{E}[F(\theta_{i,t}^{\prime})-F(\bar{\theta}_{t}^{\prime})]\leq L_{f}\mathbb{E}[\|\theta_{i,t}^{\prime}-\bar{\theta}_{t}^{\prime }\|]$. By using~\eqref{cons5}, we have
\begin{equation}
\textstyle\mathbb{E}\left[F(\theta_{i,t}^{\prime})-F(\bar{\theta}_{t}^{\prime})\right]\leq \frac{L_{f}\sqrt{c_{11}}\lambda_{0}}{(t+1)^{v}}.\label{3E23}
\end{equation}
Since $\sum_{t=0}^{T}\frac{1}{(t+1)^{p}}\leq \int_{x=0}^{T+1}\frac{1}{x^{p}}dx\leq \frac{(T+1)^{1-p}}{1-p}$ always holds for any $p\in(0,1)$, we arrive at 
\begin{equation}
\textstyle\frac{1}{T+1}\sum_{t=0}^{T}\mathbb{E}\left[F(\theta_{i,t}^{\prime})-F(\bar{\theta}_{t}^{\prime})\right]\leq \frac{L_{f}\sqrt{c_{11}}\lambda_{0}}{(1-v)(T+1)^{v}}.\label{3E24}
\end{equation}

By combining~\eqref{3E22} and~\eqref{3E24}, we arrive at 
\begin{equation}
\textstyle\frac{1}{T+1}\sum_{t=0}^{T}\mathbb{E}\left[F(\theta_{i,t}^{\prime})-F(\theta^{*})\right]\leq C_{3}(T+1)^{-(1-v)},\label{3E25}
\end{equation}
where the constant $C_{3}$ is given by $C_{3}=\frac{NL_{f}\sqrt{c_{11}}\lambda_{0}}{1-v}+C_{2}$ with $c_{11}$ given in~\eqref{cons5} and $C_{2}$ given in~\eqref{3E22}.

The definition of $C_{3}$ implies $C_{3}\leq\mathcal{O}\left(\frac{H^{2}(\sigma^{2}+L_{f}^2+\delta^2)}{(1-\rho)^2}\right)$, which, combined with~\eqref{3E25}, proves Theorem~\ref{MT1}-(ii).

\subsection{Proof of Theorem~\ref{MT2}}
To prove Theorem \ref{MT2}, we introduce the following auxiliary lemma, which is used in the analysis of strongly convex objective functions, while it is not needed in the convex case.

\begin{lemma}\label{F1lemma1}
Under the conditions in Lemma~\ref{ML1}, for any $t\geq0$, if $F(\theta)$ is $\mu$-strongly convex, the following inequality holds:
\begin{equation}
\textstyle\sum_{i=1}^{N}\mathbb{E}\left[\|\theta_{i,t}-\theta_{i,t}^{\prime}\|^2\right]\leq C_{4}\lambda_{t},\label{F1lemma1result}
\end{equation}
where $C_{4}$ is given by $C_{4}=\max\{\mathbb{E}[\|\boldsymbol{\theta}_{0}-\boldsymbol{\theta}_{0}^{\prime}\|^2],\frac{2d_{3}\lambda_{0}^2}{\mu\lambda_{0}-2v}\}$ with $d_{3}=(2d_{1}d_{2}+\frac{4N\delta^2}{\lambda_{0}})(\lambda_{0}+\frac{2}{\mu})$, $d_{1}=\frac{4H^2C_{1}(3-\rho)(N+3-\rho)}{1-\rho}+\frac{4N(3-\rho)(4-\rho)\delta^2}{(1-\rho)\lambda_{0}}$, and $d_{2}=\frac{4^{3v+1}}{(1-\bar{\rho})(\ln(\sqrt{\bar{\rho}})e)^{4}}$, $C_{1}$ given in~\eqref{D2L10},  and $\bar{\rho}=\frac{1+\rho}{2}$.
\end{lemma}
\begin{proof}
The dynamics of $\theta_{i,t}$ in Algorithm~\ref{algorithm} implies
\begin{equation}
\begin{aligned}
&\textstyle\mathbb{E}[\|\boldsymbol{\theta}_{t+1}-\boldsymbol{\theta}_{t+1}^{\prime}\|^2]\\
&\textstyle\leq\mathbb{E}[\|(W\otimes I_{n})(\boldsymbol{\theta}_{t}-\boldsymbol{\theta}_{t}^{\prime})-\lambda_{t}(\boldsymbol{g}(\boldsymbol{\theta}_{t})-\boldsymbol{m}_{t})\|^2],\label{1F1}
\end{aligned}
\end{equation}
where $\boldsymbol{g}(\boldsymbol{\theta}_{t})$ and $\boldsymbol{m}_{t}$ satisfy
\begin{equation}
\begin{aligned}
&\textstyle\boldsymbol{g}(\boldsymbol{\theta}_{t})-\boldsymbol{m}_{t}=\underbrace{\boldsymbol{g}(\boldsymbol{\theta}_{t})-\boldsymbol{g}(\boldsymbol{\theta}_{t}^{\prime})
}_{G_{1,t}}\\
&\textstyle+
\underbrace{
\begin{pmatrix}
(1-a_{1,t})g_{1}(\theta_{1,t}^{\prime})+b_{1,t}\xi_{1,t} \\
\vdots \\
(1-a_{N,t})g_{N}(\theta_{N,t}^{\prime})+b_{N,t}\xi_{N,t}
\end{pmatrix}
}_{G_{2,t}}.\label{decomg}
\end{aligned}
\end{equation}
By substituting~\eqref{decomg} into~\eqref{1F1} and using the Young's inequality, we obtain
\begin{equation}
\begin{aligned}
&\textstyle\mathbb{E}[\|\boldsymbol{\theta}_{t+1}-\boldsymbol{\theta}_{t+1}^{\prime}\|^2]\leq (1+\tau_{1})\mathbb{E}[\|\boldsymbol{\theta}_{t}-\boldsymbol{\theta}_{t}^{\prime}-\lambda_{t}G_{1,t}\|^2]\\
&\textstyle+\left(1\!+\!\frac{1}{\tau_{1}}\right)\mathbb{E}[\|(W\otimes I_{n}-I_{Nn})(\boldsymbol{\theta}_{t}\!-\!\boldsymbol{\theta}_{t}^{\prime})-\lambda_{t}G_{2,t}\|^2],\label{1F2}
\end{aligned}
\end{equation}
for any $\tau_{1}>0$.

According to the definition of $G_{1,t}$ in~\eqref{decomg}, the first term on the right hand side of~\eqref{1F2} satisfies
\begin{flalign}
&\textstyle\mathbb{E}[\|\boldsymbol{\theta}_{t}-\boldsymbol{\theta}_{t}^{\prime}-\lambda_{t}G_{1,t}\|^2]\nonumber\\
&\textstyle\leq \mathbb{E}[\|\boldsymbol{\theta}_{t}-\boldsymbol{\theta}_{t}^{\prime}\|^2]+\lambda_{t}^{2}\mathbb{E}[\|G_{1,t}\|^2]+2\mathbb{E}[\langle\boldsymbol{\theta}_{t}-\boldsymbol{\theta}_{t}^{\prime},-\lambda_{t}G_{1,t}\rangle]\nonumber\\
&\leq \left(1-2\mu\lambda_{t}+H^2\lambda_{t}^2\right)\mathbb{E}[\|\boldsymbol{\theta}_{t}-\boldsymbol{\theta}_{t}^{\prime}\|^2],\label{1F3}
\end{flalign}
where we have used the $\mu$-strong convexity of $f_{i}(\theta)$ in the second inequality and the $H_{i}$-Lipschitz continuity of $g_{i}(\theta)$ in the last inequality.

The second term on the right hand side of~\eqref{1F2} satisfies
\begin{flalign}
&\textstyle\mathbb{E}[\|(W\otimes I_{n}-I_{Nn})(\boldsymbol{\theta}_{t}-\boldsymbol{\theta}_{t}^{\prime})-\lambda_{t}G_{2,t}\|^2]\nonumber\\
&\textstyle\leq 2\mathbb{E}[\|(W\otimes I_{n}-I_{Nn})((\boldsymbol{\theta}_{t}-\boldsymbol{\theta}_{t}^{\prime})-\boldsymbol{1}_{N}\otimes(\bar{\theta}_{t}-\bar{\theta}_{t}^{\prime}))\|^2]\nonumber\\
&\textstyle\quad+2\lambda_{t}^2\mathbb{E}[\|G_{2,t}\|^2].\label{1F4}
\end{flalign}
Algorithm~\ref{algorithm} implies that the first term on the right hand side of~\eqref{1F4} satisfies
\begin{equation}
\begin{aligned}
&\mathbb{E}[\|(\boldsymbol{\theta}_{t+1}-\boldsymbol{\theta}_{t+1}^{\prime})-\boldsymbol{1}_{N}\otimes(\bar{\theta}_{t+1}-\bar{\theta}_{t+1}^{\prime})\|^2]\\
&\textstyle\leq \left(1+\tau_{2}\right)\left(1+\tau_{3}\right)\\
&\textstyle\quad\times\mathbb{E}[\|(W\otimes I_{n})((\boldsymbol{\theta}_{t}-\boldsymbol{\theta}_{t}^{\prime})-\boldsymbol{1}_{N}\otimes(\bar{\theta}_{t}-\bar{\theta}_{t}^{\prime}))\|^2]\\
&\textstyle \quad+\left(1+\tau_{2}\right)\left(1+\frac{1}{\tau_{3}}\right)\lambda_{t}^2\mathbb{E}[\|\boldsymbol{g}(\boldsymbol{\theta}_{t})-\boldsymbol{m}_{t}\|^2]\\
&\textstyle\quad +N\left(1+\frac{1}{\tau_{2}}\right)\lambda_{t}^2\mathbb{E}[\|\bar{g}(\boldsymbol{\theta}_{t})-\bar{m}_{t}\|^2],\nonumber
\end{aligned}
\end{equation}
for any $\tau_{2}>0$ and $\tau_{3}>0$.

By setting $\tau_{2}=\tau_{3}=\frac{1-\rho}{2}$ and using the the relationship $\left(1+\frac{1-\rho}{2}\right)\rho<1-\frac{1-\rho}{2}<1$, we obtain
\begin{equation}
\begin{aligned}
&\textstyle\mathbb{E}[\|(\boldsymbol{\theta}_{t+1}-\boldsymbol{\theta}_{t+1}^{\prime})-\boldsymbol{1}_{N}\otimes(\bar{\theta}_{t+1}-\bar{\theta}_{t+1}^{\prime})\|^2]\\
&\textstyle\leq \left(\frac{1+\rho}{2}\right)^2\mathbb{E}[\|(\boldsymbol{\theta}_{t}-\boldsymbol{\theta}_{t}^{\prime})-\boldsymbol{1}_{N}\otimes(\bar{\theta}_{t}-\bar{\theta}_{t}^{\prime})\|^2]\\
&\textstyle\quad+\frac{(3-\rho)^2}{1-\rho}\lambda_{t}^2\mathbb{E}[\|\boldsymbol{g}(\boldsymbol{\theta}_{t})-\boldsymbol{m}_{t}\|^2]\\
&\textstyle\quad +\frac{N(3-\rho)}{1-\rho}\lambda_{t}^2\mathbb{E}[\|\bar{g}(\boldsymbol{\theta}_{t})-\bar{m}_{t}
\|^2].\label{1F6}
\end{aligned}
\end{equation}
From~\eqref{ML1result}, we have $\mathbb{E}[\|G_{2,t}\|^2]\leq 2N\kappa_{t}^2\delta^2$, which implies that the second term on the right hand side of~\eqref{1F6} satisfies
\begin{flalign}
&\textstyle\mathbb{E}[\|\boldsymbol{g}(\boldsymbol{\theta}_{t})-\boldsymbol{m}_{t}\|^2]\leq 2\mathbb{E}[\|G_{1,t}\|^2]+2\mathbb{E}[\|G_{2,t}\|^2]\nonumber\\
&\textstyle\leq 2H^2\mathbb{E}[\|\boldsymbol{\theta}_{t}-\boldsymbol{1}_{N}\otimes\theta^{*}-(\boldsymbol{\theta}_{t}^{\prime}-\boldsymbol{1}_{N}\otimes\theta^{*})\|^2]+4N\kappa_{t}^2\delta^2\nonumber\\
&\textstyle\leq 4H^2C_{1}\lambda_{t}+4N\kappa_{t}^2\delta^2,\label{1F7}
\end{flalign}
where we have used~\eqref{D2L10} in the last inequality and the constant $C_{1}$ is given in~\eqref{D2L10}.

The last term on the right hand side of~\eqref{1F6} satisfies
\begin{equation}
\begin{aligned}
&\mathbb{E}[\|\bar{g}(\boldsymbol{\theta}_{t})-\bar{m}_{t}
\|^2]\leq \mathbb{E}[\|\bar{g}(\boldsymbol{\theta}_{t})-\bar{g}(\boldsymbol{\theta}_{t}^{\prime})\\
&\textstyle\quad+\frac{1}{N}\sum_{i=1}^{N}\left((1-a_{i,t})g_{i}(\theta_{i,t}^{\prime})+b_{i,t}\xi_{i,t}\right)\|^2]\\
&\textstyle\leq 4H^2C_{1}\lambda_{t}+4\kappa_{t}^2\delta^2.\label{1F8}
\end{aligned}
\end{equation}
Substituting~\eqref{1F7} and~\eqref{1F8} into~\eqref{1F6}, we obtain
\begin{equation}
\begin{aligned}
&\mathbb{E}[\|(\boldsymbol{\theta}_{t+1}-\boldsymbol{\theta}_{t+1}^{\prime})-\boldsymbol{1}_{N}\otimes(\bar{\theta}_{t+1}-\bar{\theta}_{t+1}^{\prime})\|^2]\\
&\textstyle\leq \left(\frac{1+\rho}{2}\right)^2\mathbb{E}[\|(\boldsymbol{\theta}_{t}-\boldsymbol{\theta}_{t}^{\prime})-\boldsymbol{1}_{N}\otimes(\bar{\theta}_{t}-\bar{\theta}_{t}^{\prime})\|^2]\\
&\textstyle\quad+\frac{4(3-\rho)(N+3-\rho)}{1-\rho}H^2C_{1}\lambda_{t}^3+\frac{4N(3-\rho)(4-\rho)}{1-\rho}\lambda_{t}^2\kappa_{t}^2\delta^2.\label{1F9}
\end{aligned}
\end{equation}
Given that the decaying rates of $\lambda_{t}$ and $\kappa_{t}$ satisfy $v\leq 2r$,~\eqref{1F9} can be simplified as follows:
\begin{equation}
\begin{aligned}
&\textstyle\mathbb{E}[\|(\boldsymbol{\theta}_{t+1}-\boldsymbol{\theta}_{t+1}^{\prime})-\boldsymbol{1}_{N}\otimes(\bar{\theta}_{t+1}-\bar{\theta}_{t+1}^{\prime})\|^2]\\
&\textstyle\leq \left(\frac{1+\rho}{2}\right)^2\mathbb{E}[\|(\boldsymbol{\theta}_{t}-\boldsymbol{\theta}_{t}^{\prime})-\boldsymbol{1}_{N}\otimes(\bar{\theta}_{t}-\bar{\theta}_{t}^{\prime})\|^2]+d_{1}\lambda_{t}^3,\label{1F10}
\end{aligned}
\end{equation}
where the constant $d_{1}$ is given by $d_{1}=\frac{4H^2C_{1}(3-\rho)(N+3-\rho)}{1-\rho}+\frac{4N(3-\rho)(4-\rho)\delta^2}{(1-\rho)\lambda_{0}}$ with $C_{1}$ given in~\eqref{D2L10}.

By telescoping~\eqref{1F10} from $0$ to $t-1$, we obtain
\begin{equation}
\begin{aligned}
&\textstyle\mathbb{E}[\|(\boldsymbol{\theta}_{t}-\boldsymbol{\theta}_{t}^{\prime})-\boldsymbol{1}_{N}\otimes(\bar{\theta}_{t}-\bar{\theta}_{t}^{\prime})\|^2]\\
&\textstyle\leq d_{1}\sum_{k=0}^{t-1}\lambda_{k}^{3}\left(\frac{1+\rho}{2} \right)^{2(t-1-k)}.\label{1F11}
\end{aligned}
\end{equation}
Substituting~\eqref{1F11} and the relationship $\mathbb{E}[\|G_{2,t}\|^2]\leq 2N\kappa_{t}^2\delta^2$ into~\eqref{1F4}, we obtain
\begin{equation}
\begin{aligned}
&\textstyle\mathbb{E}[\|(W\otimes I_{n}-I_{Nn})(\boldsymbol{\theta}_{t}-\boldsymbol{\theta}_{t}^{\prime})-\lambda_{t}G_{2,t}\|^2]\\
&\textstyle\leq 2d_{1}\sum_{k=0}^{t-1}(\lambda_{k}\sqrt{\lambda_{k}})^{2}\left(\frac{1+\rho}{2} \right)^{2(t-1-k)}+4N\lambda_{t}^2\kappa_{t}^2\delta^2,\label{1F12}
\end{aligned}
\end{equation}
which, combined with Lemma~\ref{2Clemma3}, leads to
\begin{equation}
\textstyle\mathbb{E}[\|(W\otimes I_{n}-I_{Nn})(\boldsymbol{\theta}_{t}-\boldsymbol{\theta}_{t}^{\prime})-\lambda_{t}G_{2,t}\|^2]\leq \left(2d_{1}d_{2}+\frac{4N\delta^2}{\lambda_{0}}\right)\lambda_{t}^{3},\nonumber
\end{equation}
with $d_{2}=\frac{4^{3v+1}}{(1-\bar{\rho})(\ln(\sqrt{\bar{\rho}})e)^{4}}$ with $\bar{\rho}=\frac{1+\rho}{2}$.

Substituting~\eqref{1F3} and the preceding inequality into~\eqref{1F2} and letting $\tau_{1}=\frac{\mu\lambda_{t}}{2}$, we arrive at
\begin{equation}
\begin{aligned}
&\textstyle\mathbb{E}[\|\boldsymbol{\theta}_{t+1}-\boldsymbol{\theta}_{t+1}^{\prime}\|^2]\leq \left(1-\frac{\mu\lambda_{t}}{2}\right)\mathbb{E}[\|\boldsymbol{\theta}_{t}-\boldsymbol{\theta}_{t}^{\prime}\|^2]+d_{3}\lambda_{t}^2,\label{1F14}
\end{aligned}
\end{equation}
with $d_{3}=(2d_{1}d_{2}+\frac{4N\delta^2}{\lambda_{0}})(\lambda_{0}+\frac{2}{\mu})$.

By combining Lemma 5-(i) in the arxiv version of~\cite{JDPziji} and~\eqref{1F14}, we arrive at~\eqref{F1lemma1result}.
\end{proof}

We are now in a position to prove Theorem~\ref{MT2}.

\textbf{Proof of Theorem 2.}

(i) When $f_{i}$ is $\mu$-strongly convex.

By using Assumption~\ref{A1}-(i) and~\eqref{F1lemma1result}, we obtain
\begin{equation}
\begin{aligned}
&\textstyle(\mathbb{E}[|R_{i}(f_{i}(\theta_{i,T+1}^{\prime}))-R_{i}(f_{i}(\theta_{i,T+1}))|])^2\\
&\textstyle\leq L_{R,i}^2\mathbb{E}[\| \theta_{i,T+1}^{\prime}-\theta_{i,T+1}\|^2]\leq L_{R,i}^2C_{4}\lambda_{T}.\nonumber
\end{aligned}
\end{equation}

By using Lyapunov's inequality for moments, we have
\begin{equation}
\textstyle \mathbb{E}[|R_{i}(f_{i}(\theta_{i,T+1}^{\prime}))-R_{i}(f_{i}(\theta_{i,T+1}))|]\leq L_{R,i}\sqrt{D_{1}\lambda_{T}}.\label{1F22}
\end{equation}
According to~\eqref{C1L104}, we have
\begin{equation}
\begin{aligned}
&\textstyle\mathbb{E}\left[P_{i,t}-P_{i,t}^{\prime}\right]\leq -\frac{C_{t}\deg(i)\lambda_{t}^2}{2}\mathbb{E}[\|(a_{i,t}-1)g_{i}(\theta_{i,t})\|^2]\\
&\textstyle\quad-\frac{C_{t}\deg(i)\lambda_{t}^2}{2}\mathbb{E}[\|b_{i,t}\xi_{i,t}\|^2]+c_{8}\deg(i)C_{t}\lambda_{t}^2\lambda_{t-1}^2.\label{1F23}
\end{aligned}
\end{equation}
Summing both sides of~\eqref{1F23} from $t=1$ to $t=T$ and using~\eqref{ML1result} yield
\begin{equation}
\begin{aligned}
&\textstyle\sum_{t=1}^{T}\mathbb{E}[P_{i,t}-P_{i,t}^{\prime}]\leq-\deg(i)\sum_{t=1}^{T}C_{t}\lambda_{t}^2\kappa_{t}^2\delta^2\\
&\textstyle-\deg(i)\sum_{t=1}^{T}c_{8}C_{t}\lambda_{t}^2\lambda_{t-1}^2\leq -2\deg(i)\sum_{t=1}^{T}C_{t}\lambda_{t}^2\kappa_{t}^{2}\delta^2,\label{1F231}
\end{aligned}
\end{equation}
where in the derivation we have used the fact that the decaying rate of $\lambda_{t}^2\lambda_{t-1}^2$ is faster than $\lambda_{t}^{2}\kappa_{t}^2$ and $c_{8}\geq 1$.

Combining~\eqref{1F22} and \eqref{1F231}, and using the definition $C_{t}=\frac{4L_{R,i}\sqrt{6d_{t+1\text{\tiny$\to$}T+1}}}{\deg(i)\lambda_{t}\kappa_{t}\delta}$, we have
\begin{flalign}
&\textstyle\mathbb{E}[R_{i}(f_{i}(\theta_{i,T+1}^{\prime}))\!-\!\sum_{t=1}^{T}P_{i,t}^{\prime}]\!-\!\mathbb{E}[R_{i}(f_{i}(\theta_{i,T+1}))\!-\!\sum_{t=1}^{T}P_{i,t}]\nonumber\\
&\textstyle\leq \frac{L_{R,i}\sqrt{D_{1}\lambda_{0}}}{(T+1)^{\frac{v}{2}}}+2\deg(i)\sum_{t=1}^{T}\frac{4L_{R,i}\lambda_{0}\delta\sqrt{6d_{t+1\text{\tiny$\to$}T+1}}}{\deg(i)(t+1)^{v+r}}\nonumber\\
&\textstyle\quad\leq \frac{L_{R,i}\sqrt{D_{1}\lambda_{0}}}{(T+1)^{\frac{v}{2}}}+\frac{8(v+r)L_{R,i}\lambda_{0}\delta\sqrt{6d_{t+1\text{\tiny$\to$}T+1}}}{v+r-1},\label{1F25}
\end{flalign}
where we have used~\eqref{3E17} in the last inequality. Eq.~\eqref{1F25} implies Theorem~\ref{MT2}. 
\vspace{0.5em}

(ii) When $f_{i}(\theta)$ is general convex.

From Theorem~\ref{MT1}-(ii), we have $\frac{1}{T+1}\sum_{t=0}^{T}\mathbb{E}[\|\nabla F(\bar{\theta}_{t}^{\prime})\|^2]\leq \mathcal{O}\left(\frac{1}{(T+1)^{1-v}}\right)$. According to the Stolz–Cesàro theorem, we obtain 
\begin{equation}
\begin{aligned}
&\textstyle\textstyle\lim_{t\rightarrow\infty}\frac{1}{t}\sum_{k=0}^{t-1}\mathbb{E}\left[\|\nabla F(\bar{\theta}_{k}^{\prime})\|^2\right]=0\\
&\textstyle\Longrightarrow \lim_{t\rightarrow\infty}\mathbb{E}[\|\nabla F(\bar{\theta}_{t}^{\prime})\|^2]=0.\label{2G3}
\end{aligned}
\end{equation}

We denote $\Theta^{*}=\{\theta\in\mathbb{R}^{n}\mid\nabla F(\theta)=0\}$ as the optimal-solution set. By using the continuity of $\nabla F(\theta)$ and~\eqref{2G3}, we have $\lim_{t\rightarrow\infty}\mathrm{d}(\bar{\theta}_{t}^{\prime},\Theta^{*})=0$, where $\mathrm{d}(\theta,\Theta^*)=\inf_{y\in \Theta^*}\|\theta-y\|$ denotes the distance from $\theta$ to $\Theta^*$. Furthermore, by using the continuity of $f_{i}(\theta)$, we have $\lim_{t\rightarrow\infty} f_{i}(\bar{\theta}_{t}^{\prime})=f_{i}(\theta^{\prime *})$ for some $\theta^{\prime *}\in \Theta^{*}$.

Given that $m_{i,t}=\nabla f_{i}(\theta_{i,t})$ is a special case of $m_{i,t}=a_{i,t}\nabla f_{i}(\theta_{i,t}^{\prime})+b_{i,t}\xi_{i,t}$ with $a_{i,t}\geq 1$ and $b_{i,t}\in\mathbb{R}$, we have $\lim_{t\rightarrow\infty} f_{i}(\bar{\theta}_{t})=f_{i}(\theta^{*})$ for some $\theta^{*}\in \Theta^{*}$. Note that when $F(\theta)$ is convex, the optimal solutions $\theta^{*}$ and $\theta^{\prime*}$ may be different elements in $\Theta^{*}$.

We proceed to prove that for any two points $\theta_{1}^{*},\theta_{2}^{*}\in X^{*}$, the relationship $f_{i}(\theta_{1}^{*})=f_{i}(\theta_{2}^{*})$ always holds. Since $F(\theta)$ is convex, its optimal-solution set $\Theta^{*}$ is convex, closed, and connected. We choose some point $\theta_{0}^{*}\in \Theta^{*}$ and consider a direction $b\in\mathbb{R}^{n}$ such that the segment $\theta_{0}^{*}+\varsigma b$ lies within $\Theta^{*}$ for an arbitrarily small $\varsigma>0$. Since $F(\theta)$ is constant over $\Theta^{*}$, its first and second directional derivatives at $\theta_{0}^{*}$ are zero, i.e., $\langle b,\nabla F(\theta_{0}^{*})\rangle=0$ and $b^{\top}\nabla^{2}F(\theta_{0}^{*})b=0$. Recalling $F(\theta)=\frac{1}{N}\sum_{i=1}^{N}f_{i}(\theta)$, we have $b^{\top}\frac{1}{N}\sum_{i=1}^{N}\nabla^{2}f_{i}(\theta_{0}^{*})b=0$. Furthermore, since each $f_{i}$ is convex and twice differentiable, we have $\nabla^2 f_{i}(\theta_{0}^*) \succeq 0$ and $b^{\top}\nabla^2 f_{i}(\theta_{0}^*)b\succeq 0$, which, combined with $b^{\top}\nabla^{2}F(\theta_{0}^{*})b=0$, leads to $b^{\top}\nabla^2 f_{i}(\theta_{0}^*)b=0$. Hence, each $f_{i}(\theta)$ has zero second directional derivative along any direction in $\Theta^{*}$, meaning that $f_{i}(\theta)$ is constant on $\Theta^{*}$. Hence, for any $\theta_{1}^{*},\theta_{2}^{*}\in \Theta^{*}$, we have $f_{i}(\theta_{1}^{*})=f_{i}(\theta_{2}^{*})$, which naturally leads to $f_{i}(\theta^{\prime *})=f_{i}(\theta^{*})$.

By using the relation $f_{i}(\theta^{\prime *})=f_{i}(\theta^{*})$ and~\eqref{3E25}, we have
\begin{equation}
\begin{aligned}
&\textstyle\lim_{T\rightarrow\infty}\mathbb{E}[R_{i}(f_{i}(\theta_{i,T+1}^{\prime}))-R_{i}(f_{i}(\theta_{i,T+1}))]\\
&\textstyle=\mathbb{E}[R_{i}(f_{i}(\theta^{\prime *}))-R_{i}(f_{i}(\theta^{*}))]=0.\label{2G4}
\end{aligned}
\end{equation}

Based on the definition $C_{t}=\frac{4L_{R,i}\sqrt{6d_{t+1\text{\tiny$\to$}T+1}}}{\deg(i)\lambda_{t}\kappa_{t}\delta}$ and~\eqref{1F231}, we obtain
\begin{equation}
\begin{aligned}
\sum_{t=1}^{T}\mathbb{E}[P_{i,t}-P_{i,t}^{\prime}]&\textstyle\leq -2\sum_{t=1}^{T}\frac{4L_{R,i}\lambda_{0}\delta\sqrt{6d_{t+1\text{\tiny$\to$}T+1}}}{(t+1)^{v+r}}\\
&\textstyle\leq \frac{8(v+r)L_{R,i}\lambda_{0}\delta\sqrt{6d_{t+1\text{\tiny$\to$}T+1}}}{v+r-1}.\label{2G6}
\end{aligned}
\end{equation}

Combining~\eqref{2G4} and~\eqref{2G6}, we arrive at
\begin{equation}
\begin{aligned}
&\textstyle\lim_{T\rightarrow\infty}\mathbb{E}[R_{i}(f_{i}(\theta_{i,T+1}^{\prime}))-P_{i,t}^{\prime}]\\
&\textstyle\quad-\lim_{T\rightarrow\infty}\sum_{t=1}^{T}\mathbb{E}[R_{i}(f_{i}(\theta_{i,T+1}))-\sum_{t=1}^{T}P_{i,t}]\\
&\textstyle\leq \frac{8(v+r)L_{R,i}\lambda_{0}\delta\sqrt{6d_{t+1\text{\tiny$\to$}T+1}}}{v+r-1}.\label{2G7}
\end{aligned}
\end{equation}
Furthermore, according to~\eqref{2G7}, for any finite $T$, there always exists a $C>0$ such that the following inequality holds:
\begin{flalign}
&\textstyle\mathbb{E}[R_{i}(f_{i}(\theta_{i,T+1}^{\prime}))\!-\!\sum_{t=1}^{T}P_{i,t}^{\prime}]-\mathbb{E}[R_{i}(f_{i}(\theta_{i,T+1}))\!-\!\sum_{t=1}^{T}P_{i,t}]\nonumber\\
&\textstyle\leq C\frac{8(v+r)L_{R,i}\lambda_{0}\delta\sqrt{6d_{t+1\text{\tiny$\to$}T+1}}}{v+r-1},\label{2G8}
\end{flalign}
which, combined with~\eqref{2G7}, proves Theorem~\ref{MT2}.

{\color{blue}\subsection{Proof of Lemma~\ref{ICNashlemma}}
Theorem~\ref{MT2} proves that Algorithm~\ref{algorithm} with Mechanism~\ref{mechanism} is $\epsilon$-incentive compatible. Furthermore, 
according to the definition of $\epsilon$-incentive compatibility in Definition~\ref{D3}, the following inequality holds for any agent $i\in[N]$:
\begin{equation}
\mathbb{E}[U_{i,0\text{\tiny$\to$}T}^{\mathcal{M}_{p}}(\boldsymbol{h}_{i},\boldsymbol{h}_{-i})]\geq\mathbb{E}[U_{i,0\text{\tiny$\to$}T}^{\mathcal{M}_{p}}(\boldsymbol{\alpha}_{i},\boldsymbol{h}_{-i})]-\epsilon,\label{ICN1}
\end{equation}
where $\boldsymbol{h}_{i}$ denotes the truthful action trajectory of agent $i$ and $\boldsymbol{h}_{-i}=\{\boldsymbol{h}_{1},\cdots,\boldsymbol{h}_{i-1},\boldsymbol{h}_{i+1},\cdots,\boldsymbol{h}_{N}\}$ denotes the truthful action trajectories of all agents except agent $i$.

By setting $\boldsymbol{h}=(\boldsymbol{h}_{1},\cdots,\boldsymbol{h}_{i},\cdots,\boldsymbol{h}_{N})=\boldsymbol{\alpha}^{*}$, \eqref{ICN1}
can be rewritten as
\begin{equation}
\begin{aligned}
&\textstyle\mathbb{E}\left[U_{i,0\text{\tiny$\to$}T}^{\mathcal{M}_{p}}(\boldsymbol{\alpha}_{1}^{*},\cdots,\boldsymbol{\alpha}_{i}^{*},\cdots,\boldsymbol{\alpha}_{N}^{*})\right]\\
&\textstyle\geq \mathbb{E}\left[U_{i,0\text{\tiny$\to$}T}^{\mathcal{M}_{p}}(\boldsymbol{\alpha}_{1}^{*},\cdots,\boldsymbol{\alpha}_{i},\cdots,\boldsymbol{\alpha}_{N}^{*})\right]-\epsilon,\label{ICN2}
\end{aligned}
\end{equation}
which is exactly the condition of an $\epsilon$-Nash equilibrium in Definition~\ref{DNash}.
Since $i$ is arbitrary, \eqref{ICN2} holds for every agent $i\in[N]$, and hence, $\boldsymbol{h}=(\boldsymbol{h}_{1},\ldots,\boldsymbol{h}_{N})$ is an $\epsilon$-Nash equilibrium.}

\subsection{Additional results}\label{AppendixE}
We present Corollary~\ref{C1} to prove that any manipulation of model parameters shared among agents in Algorithm~\ref{algorithm} corresponds to some form of alteration in the gradient estimates.
\begin{corollary}\label{C1}
For any agent $i \in [N]$, any manipulation of the model parameters that it shares with its neighbors in Algorithm~\ref{algorithm} corresponds to some form of alteration in the gradient estimates.
\end{corollary}
\begin{proof}
We first consider the conventional distributed SGD algorithm as follows:
\begin{equation}
\textstyle\theta_{i,t+1}= \sum_{j\in \mathcal{N}_i\cup\{i\} }w_{ij}\theta_{j,t}- \lambda_{t}g_{i}(\theta_{i,t}).
\label{eq:update}
\end{equation}
We assume that agent $i$ does not share its true model parameter $\theta_{i,t}$, but instead shares a manipulated model parameter $\tilde{\theta}_{i,t}=\hat{\alpha}_{i,t}(\theta_{i,t})$ with its neighbors, where $\hat{\alpha}_{i,t}$ represents an arbitrary action chosen by agent $i$ at iteration $t$. Then,
for any neighbor $j\in\mathcal{N}_{i}$ of agent $i$, its update rule from~\eqref{eq:update} becomes
\begin{equation}
\begin{aligned}
&\textstyle\theta_{j,t+1}= \sum_{l\in\mathcal{N}_j\setminus\{i\}} w_{jl}\theta_{l,t} + w_{ji}\tilde{\theta}_{i,t} - \lambda_{t} g_j(\theta_{j,t})\\
&\textstyle=\sum_{l\in\mathcal{N}_j} w_{jl}\theta_{l,t}-\lambda_{t}g_j(\theta_{j,t})+ w_{ji}(\hat{\alpha}_{i,t}(\theta_{i,t})-\theta_{i,t}),\label{xx1}
\end{aligned}
\end{equation}
which implies that an additional term $w_{ji}(\hat{\alpha}_{i,t}(\theta_{i,t})-\theta_{i,t})$ arises in the distributed SGD update rule implemented by the neighbor $j \in \mathcal{N}_i$.

Substituting~\eqref{xx1} into~\eqref{eq:update}, we obtain
\begin{equation}
\begin{aligned}
&\textstyle\theta_{i,t+1}=\sum_{j\in \mathcal{N}_i\cup\{i\}} w_{ij}\theta_{j,t} - \lambda_t g_i(\theta_{i,t})\\
&\quad\textstyle+\sum_{j\in \mathcal{N}_i\cup\{i\}}\left(w_{ji}(\hat{\alpha}_{i,t-1}(\theta_{i,t-1})-\theta_{i,t-1})\right)\\
&\textstyle=\sum_{j\in \mathcal{N}_i\cup\{i\}} w_{ij}\theta_{j,t} - \lambda_t\alpha_{i,t}\left(g_{i}(\theta_{i,t})\right),\label{XXXX}
\end{aligned}
\end{equation}
where $\alpha_{i,t}(g_{i}(\theta_{i,t}))$ is given by $\alpha_{i,t}(g_{i}(\theta_{i,t}))=g_i(\theta_{i,t})-\frac{\sum_{j\in \mathcal{N}_i\cup\{i\}}(w_{ji}(\hat{\alpha}_{i,t-1}(\theta_{i,t-1})-\theta_{i,t-1}))}{\lambda_{t-1}}$.~\eqref{XXXX} proves Corollary~\ref{C1}. 
\end{proof}

We present Corollary~\ref{plimit} to prove that our payment mechanism can ensure $\lim_{t\rightarrow\infty}[P_{i,t}]=0$. This guarantees that no payment is required from agent $i$ when it behaves truthfully in an infinite time horizon.
\begin{corollary}\label{plimit}
Under the conditions in Lemma~\ref{ML1}, we have $\lim_{t\rightarrow\infty}\mathbb{E}[P_{i,t}]=0$.
\end{corollary}
\begin{proof}
According to our decentralized payment mechanism in Mechanism~\ref{mechanism}, we have
\begin{equation}
\begin{aligned}
P_{i,t}^{\prime}\textstyle&\textstyle=C_{t}\sum_{j \in \mathcal{N}_{i}}(\|\theta_{i,t+1}^{\prime}-2\theta_{i,t}^{\prime}+\theta_{i,t-1}^{\prime}\|^2\\
&\textstyle\quad-\|\theta_{j,t+1}^{\prime}-2\theta_{j,t}^{\prime}+\theta_{j,t-1}^{\prime}\|^2).\label{P1}
\end{aligned}
\end{equation}
The first term on the right hand side of~\eqref{P1} satisfies
\begin{equation}
\begin{aligned}
&\textstyle C_{t}\sum_{j \in \mathcal{N}_{i}}\mathbb{E}[\|\theta_{i,t+1}^{\prime}-2\theta_{i,t}^{\prime}+\theta_{i,t-1}^{\prime}\|^2]\leq 2\deg(i)C_{t}\\
&\textstyle\times\left(\mathbb{E}[\|\theta_{i,t+1}^{\prime}-\theta_{i,t}^{\prime}\|^2]+\mathbb{E}[\|\theta_{i,t}^{\prime}-\theta_{i,t-1}^{\prime}\|^2]\right).\label{P2}
\end{aligned}
\end{equation}
Algorithm~\ref{algorithm} implies that the first term on the right hand side of~\eqref{P2} satisfies
\begin{flalign}
&\textstyle\mathbb{E}[\|\theta_{i,t+1}^{\prime}-\theta_{i,t}^{\prime}\|^2]\nonumber\\
&\textstyle=\mathbb{E}[\|\sum_{j\in\mathcal{N}_{i}\cup\{i\}}w_{ij}(\theta_{j,t}^{\prime}-\bar{\theta}_{t}^{\prime}-(\theta_{i,t}^{\prime}-\bar{\theta}_{t}^{\prime}))-\lambda_{t}g_{i}(\theta_{i,t}^{\prime})\|^2]\nonumber\\
&\leq 2\mathbb{E}[\|\boldsymbol{\theta}_{t}^{\prime}-\boldsymbol{1}_{N}\otimes\bar{\theta}_{t}^{\prime}\|^2]+2\lambda_{t}^2\mathbb{E}
[\|g_{i}(\theta_{i,t}^{\prime})\|^2].\label{P3}
\end{flalign}

When $f_{i}(\theta)$ is general convex, the $L_{f}$-Lipschitz continuity of $f_{i}(\theta)$ implies
\begin{equation}
\mathbb{E}[\|\theta_{i,t+1}^{\prime}-\theta_{i,t}^{\prime}\|^2]
\leq 2(c_{11}+L_{f}^2+\sigma^2)\lambda_{t}^2,\label{P4}
\end{equation}
where in the derivation we have used~\eqref{cons5}. 

When $f_{i}(\theta)$ is $\mu$-strongly convex, by using~\eqref{D2L10}, we have
\begin{flalign}
&\textstyle\mathbb{E}[\|\theta_{i,t+1}^{\prime}-\theta_{i,t}^{\prime}\|^2]\leq 2\mathbb{E}[\|\boldsymbol{\theta}_{t}^{\prime}-\boldsymbol{1}_{N}\otimes\bar{\theta}_{t}^{\prime}\|^2]+4HC_{1}\lambda_{t}^3\nonumber\\
&\textstyle\quad+(2\sigma^2+4H\mathbb{E}[\|\theta^{*}-\theta_{i}^{*}\|^2])\lambda_{t}^2.\label{P5}
\end{flalign}
By substituting~\eqref{D2L10} into~\eqref{D2L81}, we obtain
\begin{flalign}
&\mathbb{E}[\|\boldsymbol{\theta}_{t+1}^{\prime}-\boldsymbol{1}_{N}\otimes \bar{\theta}_{t+1}^{\prime}\|^2]\textstyle\leq \rho\mathbb{E}[\|\boldsymbol{\theta}_{t}^{\prime}-\boldsymbol{1}_{N}\otimes \bar{\theta}_{t}^{\prime}\|^2]+c_{15}\lambda_{t}^2,\nonumber
\end{flalign}
with $c_{15}=\frac{4H^2(2-\rho)(C_{1}\lambda_{0}+\mathbb{E}[\|\boldsymbol{\theta}^{*}-\boldsymbol{1}_{N}\otimes\theta^{*}\|^2])}{1-\rho}+\frac{N(2-\rho)(3\delta^2+\sigma^2)}{(1-\rho)\lambda_{0}}$.

By combining Lemma 11 in~\cite{zijiGT} and the preceding inequality, we arrive at 
\begin{equation}
\mathbb{E}[\|\boldsymbol{\theta}_{t}^{\prime}-\boldsymbol{1}_{N}\otimes \bar{\theta}_{t}^{\prime}\|^2]\leq c_{16}\lambda_{t}^2,\label{P7}
\end{equation}
with $c_{16}=(\frac{8v}{e\ln(\frac{2}{1+\rho})})^{2v}(\frac{\mathbb{E}[\|\boldsymbol{\theta}_{0}^{\prime}-\boldsymbol{1}_{N}\otimes \bar{\theta}_{0}^{\prime}\|^2]\rho}{c_{15}\lambda_{0}^2}+\frac{2}{1-\rho})$.

Substituting~\eqref{P7} into~\eqref{P5}, we arrive at
\begin{flalign}
\textstyle\mathbb{E}[\|\theta_{i,t+1}^{\prime}-\theta_{i,t}^{\prime}\|^2]\leq c_{17}\lambda_{t}^2,\label{P8}
\end{flalign}
with $c_{17}=2c_{16}+4HC_{1}\lambda_{0}+2\sigma^2+4H\mathbb{E}[\|\theta^{*}-\theta_{i}^{*}\|^2]$.

By substituting~\eqref{P8} into~\eqref{P2} and using the definition $C_{t}= \frac{4L_{R}\sqrt{6d_{t+1\text{\tiny$\to$}T+1}}}{\min\{\deg(i)\}\lambda_{t}\kappa_{t}\delta}$, we arrive at
\begin{equation}
\begin{aligned}
&\textstyle C_{t}\sum_{j \in \mathcal{N}_{i}}\mathbb{E}[\|\theta_{i,t+1}^{\prime}-2\theta_{i,t}^{\prime}+\theta_{i,t-1}^{\prime}\|^2]\\
&\textstyle\leq 2\deg(i)\frac{4L_{R}\sqrt{6d_{t+1\text{\tiny$\to$}T+1}}}{\min\{\deg(i)\}\kappa_{t}\delta}\left(1+2^{v}\right)c_{17}\lambda_{t}.\label{P9}
\end{aligned}
\end{equation}
Since the decaying rate of $\lambda_{t}$ is higher than that of $\kappa_{t}$, we have $\lim_{t\rightarrow\infty}C_{t}\sum_{j \in \mathcal{N}_{i}}\mathbb{E}[\|\theta_{i,t+1}^{\prime}-2\theta_{i,t}^{\prime}+\theta_{i,t-1}^{\prime}\|^2]=0$, which, combined with~\eqref{P1}, leads to  $\lim_{t\rightarrow\infty}\mathbb{E}[P_{i,t}]=0$.
\end{proof}

\bibliographystyle{ieeetr}  
\bibliography{arxivbib_V2}

\end{document}